\documentclass{article}

\PassOptionsToPackage{sort&compress}{natbib}

\usepackage[final]{neurips_2022}

\usepackage[utf8]{inputenc} 
\usepackage[T1]{fontenc}    
\usepackage{hyperref}       
\usepackage{url}            
\usepackage{booktabs}       
\usepackage{amsfonts}       
\usepackage{dsfont}
\usepackage{nicefrac}       
\usepackage{microtype}      
\usepackage{xcolor}         
\usepackage{natbib}
\usepackage{lipsum}
\usepackage{enumitem}

\usepackage{graphicx, wrapfig, subcaption} 

\usepackage{color, colortbl}
\definecolor{darkblue}{rgb}{0.0,0.0,0.65}
\definecolor{darkred}{rgb}{0.65,0.0,0.0}
\hypersetup{
  colorlinks = true,
  citecolor  = darkblue,
  linkcolor  = darkred,
  filecolor  = darkblue,
  urlcolor   = darkblue,
}

\usepackage{amsmath, amssymb, amsthm}
\usepackage{verbatim}

\usepackage{booktabs,makecell,multirow}       
\usepackage{xspace}
\usepackage{footnote, tablefootnote}
\usepackage{float} \floatstyle{plaintop}
\restylefloat{table}

\newcommand{\high}[1]{\textcolor{darkblue}{\bf #1}}

\newcommand{\rebuttal}[1]{#1}

\usepackage{listings}
\definecolor{codegreen}{rgb}{0,0.6,0}
\definecolor{codegray}{rgb}{0.5,0.5,0.5}
\definecolor{codepurple}{rgb}{0.58,0,0.82}
\definecolor{backcolour}{rgb}{0.95,0.95,0.92}

\newtheorem{theorem}{Theorem}
\newtheorem{prop}[theorem]{Proposition}
\newtheorem{lemma}[theorem]{Lemma}
\newtheorem{corollary}[theorem]{Corollary}
\newtheorem{definition}[theorem]{Definition}

\theoremstyle{definition}
\newtheorem{remark}[theorem]{Remark}

\newcommand{\RR}{\mathbb{R}}
\newcommand{\norm}[1]{\left\lVert #1 \right\rVert}
\DeclareMathOperator*{\argmin}{argmin}
\newcommand{\sgn}{\mathrm{sign}}

\newcommand{\1}{{\rm 1}\kern-0.24em{\rm I}}

\DeclareMathOperator*{\argmax}{argmax}
\newcommand{\breg}[3]{D_{#1}\left(#2,#3\right)}
\newcommand{\brg}[2]{D_{\psi}\left(#1,#2\right)}
 \newcommand{\brgl}[2]{D_{L}\left(#1,#2\right)}

\newcommand{\st}{\mathrm{s.t.}}

\newcommand{\R}{\mathbb{R}}

\newcommand{\inp}[2]{\left\langle#1,~ #2 \right\rangle}

\newcommand{\mir}{\psi}
\newcommand{\reg}[1]{u^{\sf \tiny r}_{#1}}
\newcommand{\mmd}[1]{u^{\sf \tiny m}_{#1}}
\newcommand{\mar}[1]{\hat{\gamma}_{#1}}

\newcommand{\algname}{$p$-{\small \sf GD}\xspace}

\newcommand{\pmval}[2]{#1 \textcolor{gray}{$\pm$  #2}}
\newcommand{\bpmval}[2]{\textbf{#1} \textcolor{gray}{$\pm$  #2}}

\title{Mirror Descent Maximizes Generalized Margin \\ and Can Be Implemented Efficiently}

\author{
    Haoyuan Sun$^1$, Kwangjun Ahn$^1$, Christos Thrampoulidis$^2${\normalfont , and} Navid Azizan$^1$\\
    \texttt{haoyuans@mit.edu, kjahn@mit.edu, cthrampo@ece.ubc.ca, azizan@mit.edu}\\
    $^1$~Massachusetts Institute of Technology \\
    $^2$~University of British Columbia
}

\begin{document}

\maketitle

\begin{abstract}

    Driven by the empirical success and wide use of deep neural networks, understanding the generalization performance of overparameterized models has become an increasingly popular question. To this end, there has been substantial effort to characterize the implicit bias of the optimization algorithms used, such as gradient descent (GD), and the structural properties of their preferred solutions. This paper answers an open question in this literature: For the classification setting, what solution does mirror descent (MD) converge to? 
    Specifically, motivated by its efficient implementation, we consider the family of mirror descent algorithms with  potential function chosen as the $p$-th power of the $\ell_p$-norm, which is an important generalization of GD. We call this algorithm \algname{}. For this family, we characterize the solutions it obtains and show that it converges in direction to a \emph{generalized maximum-margin} solution with respect to the $\ell_p$-norm for linearly separable classification. 
    While the MD update rule is in general expensive to compute and perhaps not suitable for deep learning, \algname{} is fully parallelizable in the same manner as SGD and can be used to train deep neural networks with virtually \emph{no additional computational overhead}.
    Using comprehensive experiments with both linear and deep neural network models, we demonstrate that \algname{} can noticeably affect the structure and the generalization performance of the learned models.

\end{abstract}

\section{Introduction}

Overparameterized deep neural networks have enjoyed a tremendous amount of success in a wide range of machine learning applications~\citep{schrittwieser2020mastering,ramesh2021zero, brown2020language, dosovitskiy2020image}.
However, as these highly expressive models have the capacity to have multiple solutions that interpolate training data, and not all these solutions perform well on test data, it is important to characterize which of these interpolating solutions the optimization algorithms converge to.
Such characterization is important as it helps understand the generalization performance of these models, which is one of the most fundamental questions in machine learning.

Notably, it has been observed that even in the absence of any explicit regularization, the interpolating solutions obtained by the standard gradient-based optimization algorithms, such as (stochastic) gradient descent, tend to generalize well. Recent research has highlighted that such algorithms favor particular types of solutions, i.e., they \emph{implicitly regularize} the learned models.
\rebuttal{Importantly, such implicit biases are shown to play a crucial role in determining generalization performance, 
e.g., \citep{neyshabur2014search, zhang2021understanding, wilson2017marginal, donhauser2022fast}.}

In the literature, the implicit bias of first-order methods is first studied in linear settings since the analysis is more tractable, and also, there have been several theoretical and empirical evidence that certain insights from linear models translate to deep learning, e.g. \citep{jacot2018neural,allen2019convergence, belkin2019reconciling,lyu2019gradient,bartlett2017spectrally,nakkiran2021deep}. In the linear setting, it is easier to establish implicit bias for regression tasks, where square loss is typically used and it attains its minimum at a finite value. For example, the implicit bias of gradient descent (GD) for square loss goes back to \cite{engl1996regularization}. Beyond GD, analysis of other popular algorithms such as the family of mirror descent (MD), which is an important generalization of GD
, is more involved and was established only recently by \citep{gunasekar2018characterizing,azizan2018stochastic}. Specifically, those works showed that mirror descent converges to the interpolating solution that is closest to the initialization in terms of a Bregman divergence. 
Thus, the implicit bias in linear regression is relatively well-understood by now.

On the other hand, {\bf in the classification setting, the implicit bias analysis becomes significantly more challenging, and several questions remain open} despite significant progress in the past few years. A key differentiating factor in the classification setting is that the loss function does not attain its minimum at a finite value, and the weights have to grow to infinity.
It has been shown that for the logistics loss, the gradient descent iterates converge to the $\ell_2$-maximum margin SVM solution in direction~\citep{soudry2018implicit, ji2019implicit}.
However, such characterizations for mirror descent are missing in the literature.
Because it is possible for optimization algorithms to exhibit implicit bias in regression but not in classification (and vice versa) \citep{gunasekar2018characterizing}, resolving this gap of knowledge warrants careful analysis.
See Table~\ref{table:main} for a summary.

In this paper, we advance the understanding of the implicit regularization of mirror descent in the classification setting.
In particular, inspired by their practicality, we focus on mirror descents with potential function $\mir(\cdot) = \frac{1}{p}\norm{\cdot}_p^p$ for $p > 1$.  
More specifically, such choice of potential results in an update rule that \textit{can be applied coordinate-wise}, in the sense that updating the value at one coordinate does not depend on the values at other coordinates.
Thanks to this property, this subclass of mirror descent can be implemented with \textit{no additional computational overhead}, making it much more practical than other algorithms in the literature; see Remark~\ref{rmk:sep} for more details.
\begin{table}[t]
\centering
\renewcommand{\arraystretch}{1.5}
\setlength\tabcolsep{10pt}
\begin{tabular}{ |c |c|c| }
\hline  & Regression & Classification  \\
\hline\hline  
\multirow{4}{*}{\begin{tabular}{c}Gradient Descent\\($\mir(\cdot) = \frac{1}{2}\norm{\cdot}_2^2$)\end{tabular}}  & $\argmin_w \norm{w-w_0}_2$ & $\argmin_w \norm{w}_2$  \\
 & $\st~~w \text{ fits all data} $ & $\st~~w \text{ classifies all data} $    \\
 & 
\multirow{2}{*}{\citep[Thm 6.1]{engl1996regularization}} & \cite{soudry2018implicit}   \\
 & & \cite{ji2019implicit} \\
   \hline
 \multirow{4}{*}{\begin{tabular}{c}Mirror Descent\\(e.g. $\mir(\cdot) = \frac{1}{p}\norm{\cdot}_p^p$)\end{tabular}}  & $\argmin_w \norm{w-w_0}_p$ & $\argmin_w \norm{w}_p$  \\
 & $\st~~w \text{ fits all data} $ & $\st~~w \text{ classifies all data} $    \\
 & 
\cite{gunasekar2018characterizing} & \multirow{2}{*}{ \large \high{This work}}  \\
& \cite{azizan2018stochastic} & \\
  \hline 
\end{tabular}
\caption{{\bf Conceptual summary of our results.} In the case of linear regression, the implicit regularization results are complete; it is shown that mirror descent converges to the interpolating solution that is closest to the initialization. 
However, such characterization in the classification setting is missing in the literature and this is precisely the goal of this work. 
In particular, motivated by its practical application, we consider the 
potential function $\mir(\cdot) = \frac{1}{p}\norm{\cdot}_p^p$ and extend the result of the gradient descent to such mirror descents. }
\label{table:main}
\end{table}


\paragraph{Our contributions.}
In this paper, we make the following contributions:
\begin{list}{{\tiny $\blacksquare$}}{\leftmargin=1.5em}
  \setlength{\itemsep}{-0.75pt}
    \item We study mirror descent with potential $\frac{1}{p}\norm{\cdot}_p^p$ for $p > 1$, which will call \textit{$p$-norm GD}, and abbreviated as \algname, as a practical and efficient generalization of the popular gradient descent.
    \item We show that for separable linear classification with logistics loss, \algname exhibits implicit regularization by converging in direction to a ``generalized'' maximum-margin solution with respect to the $\ell_p$ norm.
    More generally, we show that, for monotonically decreasing loss functions, \algname follows the so-called regularization path, which is defined in Section~\ref{sec:priliminaries}.
    \item We investigate the implications of our theoretical findings with two sets of experiments:
    Our experiments involving linear models corroborate our theoretical results, and real-world experiments with deep neural networks and popular datasets suggest that our findings carry over to such nonlinear settings.  \rebuttal{Our deep learning experiments further show that \algname with different $p$ lead to significantly different generalization performance.}
\end{list}


\paragraph{Additional related work.}
We remark that recent works also attempt to accelerate the convergence of gradient descent to its implicit regularization, either by using an aggressive step size schedule \citep{nacson2019convergence, ji2021characterizing} or with momentum \citep{ji2021fast}. 
Further, there have been several results for other optimization methods, including steepest descent, AdaBoost, and various adaptive methods such as RMSProp and Adam~\citep{telgarsky2013margins,gunasekar2018characterizing,rosset2004boosting,wang2021implicit,min2022one}. 
A mirror-descent-based algorithm for explicit regularization was recently proposed by \cite{azizan2021beyond}. 
\rebuttal{
Comparatively, there has been very little progress on mirror descent in the classification setting. \cite{li2021implicit} consider a mirror descent, but their assumptions are not applicable beyond the $\ell_2$ geometry.\footnote{To be precise, they assume that the Bregman divergence is lower and upper bounded by a constant factor of the squared Euclidean distance, e.g., as in the case of a squared Mahalanobis distance.} To the best of our knowledge, there is no result for more general mirror descent algorithms in the classification setting.
}

\section{Background and Problem Setting}
\label{sec:priliminaries}

 We are interested in the well-known classification setting. Consider a collection of input-label pairs $\{(x_i, y_i)\}_{i=1}^n \subset \RR^d \times \{\pm 1\}$ and  a classifier $f_w(x)$, where $w \in \mathcal{W}$.
For some \textit{loss function} $\ell : \RR \times \{\pm 1\} \to \RR$, our goal is to minimize the empirical loss:
\[ L(w) = \frac{1}{n}\sum_{i=1}^n \ell(y_i \cdot f_w(x_i)).\]
Throughout the paper, we assume that the classification loss function $\ell$ is decreasing, convex and does not attain its minimum, as in most common loss functions in practice (e.g., logistics loss and exponential loss). Without loss of generality, we assume that $\inf \ell(\cdot) = 0$. 


For our theoretical analysis, we consider a linear model, where the models can be expressed by $f_w(x) = w^\top x$ and $w \in \RR^d$.
We also make the following assumptions about the data. 
\rebuttal{
First, since we are mainly interested in the over-parameterized setting where $d > n$, we assume that the data is linearly separable, i.e., 
there exists  $w^* \in \RR^d$ s.t. $\sgn(\inp{w^*}{x_i}) = y_i$ for all $i\in[n]$.
}
We also assume that the inputs $x_i$'s are bounded. More specifically, for our later purpose, we assume that for $p>0$, there exists some constant $C$ so that $\max_i \norm{x_i}_q < C$, where $1/q + 1/p = 1$. 


\paragraph{Preliminaries on mirror descent.}
The key component of mirror descent is a \textit{potential function}. 
In this work, we will focus on differentiable and strictly convex potentials defined on the entire domain $\R^n$.\footnote{In general, the mirror map is a convex function of Legendre type~(see, e.g., \citep[Section 26]{rockafellar1970convex}).} 
We call $\nabla \psi$ the corresponding \textit{mirror map}.
Given a potential, the natural notion of ``distance'' associated with the potential $\psi$ is given by the Bregman divergence.

\begin{definition}[Bregman divergence~\citep{bregman1967relaxation}]
For a mirror map  $\mir$, the Bregman divergence $\breg{\mir}{\cdot}{\cdot}$ associated to $\mir$ is defined as
\begin{align*}
    \breg{\mir}{x}{y}:= \mir(x)-\mir(y) -\inp{\nabla \mir(y)}{x-y},\qquad \forall x,y\in \R^n\,.
\end{align*}
\end{definition} 


An important case is the potential $\psi = \frac{\norm \cdot^2}{2}$, where $\norm \cdot$ denotes the Euclidean norm.
Then, the Bregman divergence becomes $D_\psi(x,y) = \frac{1}{2}\norm{x-y}^2$.
For more background on Bregman divergence and its properties, see, e.g., \citep[Section 2.2]{bauschke2017descent} and \citep{azizan2019stochastic}.
 
Mirror descent (MD) with respect to the mirror map $\psi$ is a generalization of gradient descent where we use Bregman divergence as a measure of distance:
\begin{align} \tag{\sf MD}\label{equ:md}
    w_{t+1} = \argmin_w \left\{\frac{1}{\eta}D_\psi(w, w_t) + \inp{\nabla L(w_t)}{w}\right\}
\end{align}


Equivalently, \ref{equ:md} can be written as  $\nabla\psi(w_{t+1}) = \nabla\psi(w_t) - \eta \nabla L(w_t)$. We refer readers to \cite[Figure 4.1]{bubeck2015convex} for a nice illustration of mirror descent.
Also, see \citep[Section 5.7]{juditsky2011first} for various examples of potentials depending on applications.


One property we will repeatedly use is the following~\citep{azizan2018stochastic}:

\begin{lemma}[\ref{equ:md} identity]
\label{thm:key-iden}
For any $w \in \RR^n$, the following identities hold for \ref{equ:md}:
\begin{subequations}
\begin{align}
    &D_\psi(w, w_t) = D_\psi(w, w_{t+1}) + D_{\psi - \eta L}(w_{t+1}, w_t) + \eta D_{L}(w, w_t) - \eta L(w) + \eta L(w_{t+1})\,, \label{equ:key-iden-1}\\
     &\quad = D_\psi(w, w_{t+1}) + D_{\psi - \eta L}(w_{t+1}, w_{t}) - \eta \inp{\nabla L(w_t)}{w - w_t} - \eta L(w_{t}) + \eta L(w_{t+1})\, \label{equ:key-iden-2}.
\end{align}
\end{subequations}
\end{lemma}


Using Lemma~\ref{thm:key-iden}, we make several new observations and prove the following useful statements.

\begin{lemma}
\label{thm:decreasing-lose}
For sufficiently small step size $\eta$ such that $\psi - \eta L$ is convex, the loss is monotonically decreasing after each iteration of \ref{equ:md}, i.e., $L(w_{t+1}) \le L(w_{t})$.
\end{lemma}

\begin{lemma}
\label{thm:to-infinity}
In a separable linear classification problem, if  $\eta$ is chosen sufficiently small s.t. $\psi - \eta L$ is convex, then we have $L(w_t) \to 0$ as $t \to \infty$. Hence, $\lim_{t\to \infty}\norm{w_t} = \infty$ for any norm $\norm{\cdot}$.
\end{lemma}

The formal proofs of these lemmas can be found in Appendix~\ref{sec:proof-basic-lemmas}.

\begin{remark}
\rebuttal{
We can relax the condition in Lemma \ref{thm:decreasing-lose} and \ref{thm:to-infinity} such that for a sufficiently small step size $\eta$, $\psi - \eta L$ is only locally convex at the iterates $w_t$.
The relaxed condition allows us to analyze losses such as the exponential loss (see, e.g. Footnote 2 of \cite{soudry2018implicit}).
}
This condition can be considered as the mirror descent counterpart to the standard smoothness assumption in the analysis of gradient descent (see \cite{lu2018relatively}).
\end{remark}

{\bf Preliminaries on the convergence of linear classifier.}
As we discussed above,  the weights vector $w_t$ diverges for mirror descent.
Here the main theoretical question is:


\begin{center}
What direction does \ref{equ:md} diverge to? In other words, can we characterize $w_t / \norm{w_t}$ as $t\to \infty$?
\end{center}


To answer this question, we define two special directions whose importance will be illustrated later.

\begin{definition}
The \textbf{regularization path} with respect to the $\ell_p$-norm is defined as
\begin{equation}
    \bar{w}_p(B) = \argmin_{\norm{w}_p \le B} L(w)
\end{equation}
And if the limit $\lim_{B\to\infty} \bar{w}_p(B) / B$ exists, we call it the \textbf{regularized direction} and denote it by $\reg{p}$.
\end{definition}

\begin{definition}
The \textbf{margin} $\gamma$ of the a linear classifier $w$ is defined as
$\gamma(w) = \min_{i=1, \dots, n} y_i \inp{x_i}{w}. $
The {\bf max-margin direction} with respect to the $\ell_p$-norm is defined as:
\begin{equation}
    \mmd{p} := \argmax_{\norm{w}_p \le 1} \left\{ \min_{i=1, \dots, n} y_i \inp{x_i}{w} \right\}
\end{equation}
And let $\mar{p}$ be the optimal value to the equation above.
\end{definition}

Note that $\mmd{p}$ is parallel to the hard-margin SVM solution w.r.t. $\ell_p$-norm: $\argmin_w \{\norm{w}_p : \gamma(w) \ge 1\}$.
\rebuttal{Also note that the superscripts in $\reg{p}$ and $\mmd{p}$ are not variables and we only use this notation to differentiate the two definitions.}
Prior results had shown that, in linear classification, gradient descent converges in direction.

\begin{theorem}[\cite{soudry2018implicit}]
\label{thm:gd-maxmargin}
For separable linear classification with logistics loss, the gradient descent iterates with sufficiently small step size converge in direction to $\mmd{2}$, i.e., $\lim_{t\to\infty} \frac{w_t}{\norm{w_t}_2} = \mmd{2}$. 
\end{theorem}

\begin{theorem}[\cite{ji2020gradient}]
\label{thm:gd-regdir}
If the regularized direction $\reg{p}$ with respect to the $\ell_2$-norm exists, then the gradient descent iterates with sufficiently small step size converge to the regularized direction $\reg{2}$, i.e., $\lim_{t\to\infty} \frac{w_t}{\norm{w_t}_2} = \reg{2}$. 
\end{theorem}

\section{Mirror Descent with the $p$-th Power of $\ell_p$-norm}
\label{sec:main-result}

In this section, we investigate theoretical properties of the main algorithm of interest, namely mirror descent with $\mir(\cdot) = \frac{1}{p} \norm{\cdot}_p^p$ and for $p > 1$.\footnote{\rebuttal{Because the potential function must be \textit{strictly} convex for Bregman divergence to be well-defined, the value of $p$ cannot be exactly 1.}}
We shall call this algorithm \textit{$p$-norm GD} because it naturally generalizes gradient descent to $\ell_p$ geometry, and for conciseness, we will refer to this algorithm by the shorthand \algname.
As noticed by \cite{azizan2021stochastic}, this choice of mirror potential is particularly of practical interest because the mirror map $\nabla \psi$ updates becomes \textit{separable} in coordinates and thus can be implemented \textit{coordinate-wise} independent of other coordinates:
\begin{align} \tag{\algname}\label{mdpp}
\forall j \in [d],\quad \begin{cases}
    \rebuttal{w_{t+1}[j] \leftarrow \left| w_t^+[j] \right|^{\frac{1}{p-1}} \cdot \sgn\left( w_t^+[j]\right)}\\
    w_t^+[j]:= |w_t[j]|^{p-1}\sgn(w_t[j]) - \eta \nabla L(w_t)[j]
    \end{cases}
\end{align}

Furthermore, we can extend upon the observation of \cite{azizan2021stochastic} and derive these identities that allow us to better manipulate \algname:
\begin{subequations}
\begin{align}
 \inp{\nabla \psi (w)}{w} &= \sgn(w_1)w_1 \cdot |w_1|^{p-1} + \cdots + \sgn(w_d)w_d \cdot |w_d|^{p-1} = \norm{w}^p\\
  \brg{c w}{c w'} &= |c|^p \brg{w}{w'} \quad \forall c\in \R. \label{equ:homo}
\end{align}
\end{subequations}
\begin{remark} \label{rmk:sep}
Note that the coordinate-wise separability property is not shared among other related algorithms in the literature.
For instance, the choice $\psi = \frac{1}{2} \norm{\cdot}_q^2$ for $1/p + 1/q = 1$, which is referred to as the $p$-norm algorithm~\citep{grove2001general, gentile2003robustness} is not fully coordinate-wise separable since it requires computing $\norm{w_t}_p$ at each step (see, e.g., \citep[eq. (1)]{gentile2003robustness}). 
Another related algorithm is steepest descent, where the Bregman divergence in \ref{equ:md} is replaced with $\norm{\cdot}^2$ for general norm $\norm{\cdot}$.\footnote{It is also worth noting that steepest descent is not an instance of mirror descent since $\norm{\cdot}^2$ is not a Bregman divergence for a general norm $\norm{\cdot}$.}
However, for similar reasons, the update rule is not fully separable.
\end{remark}

\subsection{Main theoretical results}

We extend Theorems~\ref{thm:gd-maxmargin}~and~\ref{thm:gd-regdir} to the setting of \algname.
We will resolve two major obstacles in the analysis of implicit regularization in linear classification:
\begin{list}{{\tiny $\blacksquare$}}{\leftmargin=1.5em}
  \setlength{\itemsep}{-1pt}
    \item Our analysis approaches the classification setting as a limit of the regression implicit bias. This argument gives stronger theoretical justification for utilizing the regularized direction (as employed by~\cite{ji2020gradient}) and partially addresses the concern from \cite{gunasekar2018characterizing} that the implicit bias of regression and classification problems are ``fundamentally different.''
    \item On a more technical note, analyzing the implicit bias requires handling the cross terms of the form $\inp{\nabla\mir(w)}{w'}$, which lack direct geometric interpretations.
    We demonstrate that for our potential functions of interest, these terms can be nicely written and can be handled in the analysis.
\end{list}

We begin with the motivation behind the regularized direction, and consider the regression setting in which there exists some weight vector $w$ such that $L(w) = 0$.
Then, we can apply Lemma~\ref{thm:key-iden} to get
\[D_\psi(w, w_{t}) = D_\psi(w, w_{t+1}) + D_{\psi - \eta L}(w_{t+1}, w_t) + \eta D_{L}(w, w_t) + \eta (L(w_{t+1}) - L(w))\]
Since we assumed $L(w) = 0$, the equation above implies that $D_\psi(w, w_{t}) \ge D_\psi(w, w_{t+1})$ for sufficiently small step-size $\eta$.
This can be interpreted as \ref{equ:md} having a decreasing ``potential'' of the from $\brg{w}{\cdot}$ during each step.
Using this property, \citet{azizan2018stochastic} establishes the implicit bias results of mirror descent in the regression setting.

However, such weight vector $w$ does not exist in the classification setting.
One natural workaround would   then be to choose a vector $w$ so that $L(w) \le L(w_t)$ for all $t \le T$.
The following result, which is a generalization of  \citep[Lemma 9]{ji2020gradient}, shows that one can in fact choose the reference vector $w$ as a scalar multiple of the regularized direction.

\begin{lemma}
\label{thm:approx-reg-dir-loss}
If the regularized direction $\reg{p}$ exists, then $\forall \alpha > 0$, there exists $r_\alpha$ such that for any $w$ with $\norm{w}_p > r_\alpha$, we have $L((1+\alpha)\norm{w}_p \reg{p}) \le L(w)$. 
\end{lemma}

However, this does not resolve the issue altogether. Recall from Lemma~\ref{thm:to-infinity} that the loss approaches 0, and therefore one cannot choose a fixed reference vector $w$ in the limit as $T\to \infty$.
But due to the homogeneity of Bregman divergence \eqref{equ:homo}, we can scale $\reg{p}$ by a constant factor during each iteration, and, by doing so, we choose the reference vector $w$ to be a ``moving target.''
In other words, the idea behind our analysis is that the classification problem is chasing after a regression one and would behave similar to it in the limit.
Let us formalize this idea. We begin with the following inequality:
\begin{align} \label{ineq:1}
\brg{c_t\reg{p}}{w_{t+1}} \le \brg{c_t\reg{p}}{w_t} - \eta L(w_{t+1}) + \eta L(w_t),
\end{align}
where $c_t$ is taken to be $\approx \norm{w_t}_p$.\footnote{To be more precise, we want $c_t = (1+\alpha) \norm{w_t}_p$; and reason behind this choice is self-evident after we present Corollary~\ref{thm:cross-term}.}

Now we modify \eqref{ineq:1} so that it can telescope over  different iterations.
One way is to add $\brg{c_{t+1}\reg{p}}{w_{t+1}}$ on both sides of \eqref{ineq:1} and move $ \brg{c_t\reg{p}}{w_{t+1}}$ to the right-hand side as follows:
\begin{align*}
    &\brg{c_{t+1}\reg{p}}{w_{t+1}} \\
    \le{}& \brg{c_t\reg{p}}{w_t} - \eta L(w_{t+1}) + \eta L(w_t) + \brg{c_{t+1}\reg{p}}{w_{t+1}} - \brg{c_t\reg{p}}{w_{t+1}} \\
    ={}& \brg{c_t\reg{p}}{w_t} - \eta L(w_{t+1}) + \eta L(w_t) + \psi(c_{t+1} \reg{p}) - \psi(c_t \reg{p})   - \inp{\nabla\psi(w_{t+1})}{(c_{t+1} - c_t) \reg{p}}
\end{align*}
Summing over $t = 0, \dots, T-1$ gives us
\begin{equation}
\label{equ:tele-sum}
\begin{aligned}
    \brg{c_T\reg{p}}{w_T}
    &\le \brg{c_0\reg{p}}{w_0} - \eta L(w_1) + \eta L(w_T) + \psi(c_{T} \reg{p}) - \psi(c_1 \reg{p}) \\
    &\hspace{6em}- \sum_{t=1}^{T-1}\inp{\nabla\psi(w_{t+1})}{(c_{t+1} - c_t) \reg{p}}
\end{aligned}
\end{equation}
The rest of the argument deals with simplifying  quantities that do not cancel under telescoping sum.
For instance, in order to deal with $\inp{\nabla\psi(w_{t+1})}{\reg{p}}$, we invoke the \ref{equ:md} update rule as follows
\begin{align*}
 \inp{\nabla\psi(w_{t+1}) - \nabla\psi(w_{t})}{\reg{p}} = \inp{-\eta\nabla L(w_t)}{\reg{p}} \gtrsim \inp{-\eta\nabla L(w_t)}{w_t}, 
\end{align*}
where the last inequality follows from the intuition that $\reg{p}$ is the direction along which the loss achieves the smallest value and hence $\nabla L(w_t)$ must point away from $\reg{p}$, i.e., it must be that $\inp{\nabla L(w_t)}{\reg{p}} \lesssim \inp{\nabla L(w_t)}{u}$ for any direction $u$.
The following result formalizes this intuition.
\begin{corollary}
\label{thm:cross-term}
For $w$ so that $\norm{w}_p > r_\alpha$, we have $\inp{\nabla L(w)}{w} \ge (1+\alpha)\norm{w}_p\inp{\nabla L(w)}{\reg{p}}$.  
\end{corollary}
\noindent \emph{Proof.}
This follows from the convexity of $L$ and Lemma~\ref{thm:approx-reg-dir-loss}:
$\inp{\nabla L(w)}{w - (1+\alpha)\norm{w}\reg{p})} \ge L(w) - L((1+\alpha)\norm{w}\reg{p}) \ge 0$. \qed

Now we are left with the terms $\inp{-\eta\nabla L(w_t)}{w_t}$.
For general potential $\mir$, the quantity $\inp{-\eta\nabla L(w_t)}{w_t} = \inp{\nabla\psi(w_{t+1}) - \nabla\psi(w_{t})}{w_t}$ cannot be simplified.
On the other hand, due to our choice of potential, one can invoke Lemma~\ref{thm:key-iden} to lower bound these quantities in terms of $\norm{w_{t+1}}_p$ and $\norm{w_t}_p$, and this step is detailed in Lemma~\ref{thm:cross-term-diff} in Appendix~\ref{sec:cross-term-diff}.
Once we have established a lower bound on $\inp{\nabla\psi(w_{t+1})}{\reg{p}}$, we can turn \eqref{equ:tele-sum} entirely into a telescoping sum and unwind the above process to show that $\brg{\reg{p}}{w_t / \norm{w_t}_p}$ must converge to zero in the limit as $t \to \infty$.
Putting this all together, we obtain the following result.

\begin{theorem}
\label{thm:primal-bias}
For a separable linear classification problem, if the regularized direction $\reg{p}$ exists, then with sufficiently small step size, the iterates of \algname converge to $\reg{p}$ in direction:
\begin{equation}
    \lim_{t\to\infty} \frac{w_t}{\norm{w_t}_p} = \reg{p}.
\end{equation}
\end{theorem}

A formal proof of this theorem can be found in Appendix~\ref{sec:proof-primal-bias}.
We note that our proof further simplifies  derivations using  the separability of the mirror map.
The final missing piece would be the existence of the regularized direction.
In general, finding the limit direction $\reg{p}$ would be difficult.
Fortunately, we can sometimes appeal to the max-margin direction that is much easier to compute.
The following result is a generalization of \citep[Proposition 10]{ji2020gradient} and shows that for common losses in classification, the regularized direction and the max-margin direction are the same, hence proving the existence of the former.

\begin{prop}
\label{thm:reg-max-dir}
    If we have a loss with exponential tail, e.g. $\lim_{z\to\infty} \ell(z) e^{az} = b$, then the regularized direction exists and it is equal to the max-margin direction $\mmd{p}$.
\end{prop}
The proof of this result can be found in Appendix~\ref{sec:proof-reg-max-dir}.
Note that many commonly used losses in classification, e.g., logistic loss, have exponential tail.

\subsection{Asymptotic convergence rate}
\label{sec:asymp-result}
In this section, we characterize  the rate of convergence in Theorem~\ref{thm:primal-bias}.
Following the proof of Theorem~\ref{thm:primal-bias}, one can show the following   result in the case of linearly separable data.
\begin{corollary}
    \label{thm:convg-rate}
    The following rate of convergence holds:
    \[\brg{\reg{p}}{\frac{w_t}{\norm{w_t}_p}} \in O\left(\norm{w_t}_p^{-(p-1)}\right).\]
\end{corollary}
In order to fully understand the convergence rate, we need to characterize  the asymptotic behavior of $\norm{w_t}_p$.
The next result precisely does that. 
Recall that we assumed the dataset is bounded so that $\max_i \norm{x_i}_q \le C$ for $1/p + 1/q = 1$, and the max-margin direction $\mmd{p}$ satisfies $\inp{x_i}{\mmd{p}} \ge \hat{\gamma}_p \, \forall i \in [n]$.
Then, we have the following bound on $\norm{w_t}_p$.

\begin{lemma}
    \label{thm:norm-rate}
    For exponential loss $\ell(z) = \exp(-z)$, the asymptotic growth of $\norm{w_t}_p$ is contained in $\Theta(\log t)$.
    In particular, we have
    \[\liminf_{t\to\infty} \norm{w_t}_p \ge \frac{1}{C} (\log t - p \log\log t) + O(1) \text{ and } \limsup_{t\to\infty} \norm{w_t}_p \le \hat{\gamma}_p^{-1} \frac{p}{p-1} \log t.\]
\end{lemma}
The proof of this lemma  can be found in Appendix~\ref{sec:proof-asymp-result}.
Similar conclusions can be reached for other losses with exponential tail.
Therefore, in such cases, \algname{} has poly-logarithmic rate of convergence.
\begin{corollary}
    \label{thm:final-convg-rate}
    For exponential loss, we have convergence rate 
    \[ \brg{\reg{p}}{\frac{w_t}{\norm{w_t}_p}} \in O\left(\frac{1}{\log^{p-1}(t)}\right).\]
\end{corollary}

\section{Experiments}
\label{sec:experiments}

In this section, we investigate the behavior and performance of \algname for various values of $p$.
We naturally pick $p = 2$ that corresponds to gradient descent, 
Because \algname does not directly support $p = 1$ and $\infty$, we choose $p = 1.1$ as a surrogate for $\ell_1$, and $p = 10$ as a surrogate for $\ell_\infty$.
We also consider $p = 1.5, 3, 6$ to interpolate these points.
This section will present a summary of our results; the complete experimental setup and full results can be found in Appendices~\ref{sec:experiment-detail} and \ref{sec:add-experiments}.

\subsection{Linear classification}
\label{sec:linear-classifier}
\begin{figure}
    \centering
    \begin{subfigure}[b]{0.32\textwidth}
        \centering
        \includegraphics[width=\textwidth]{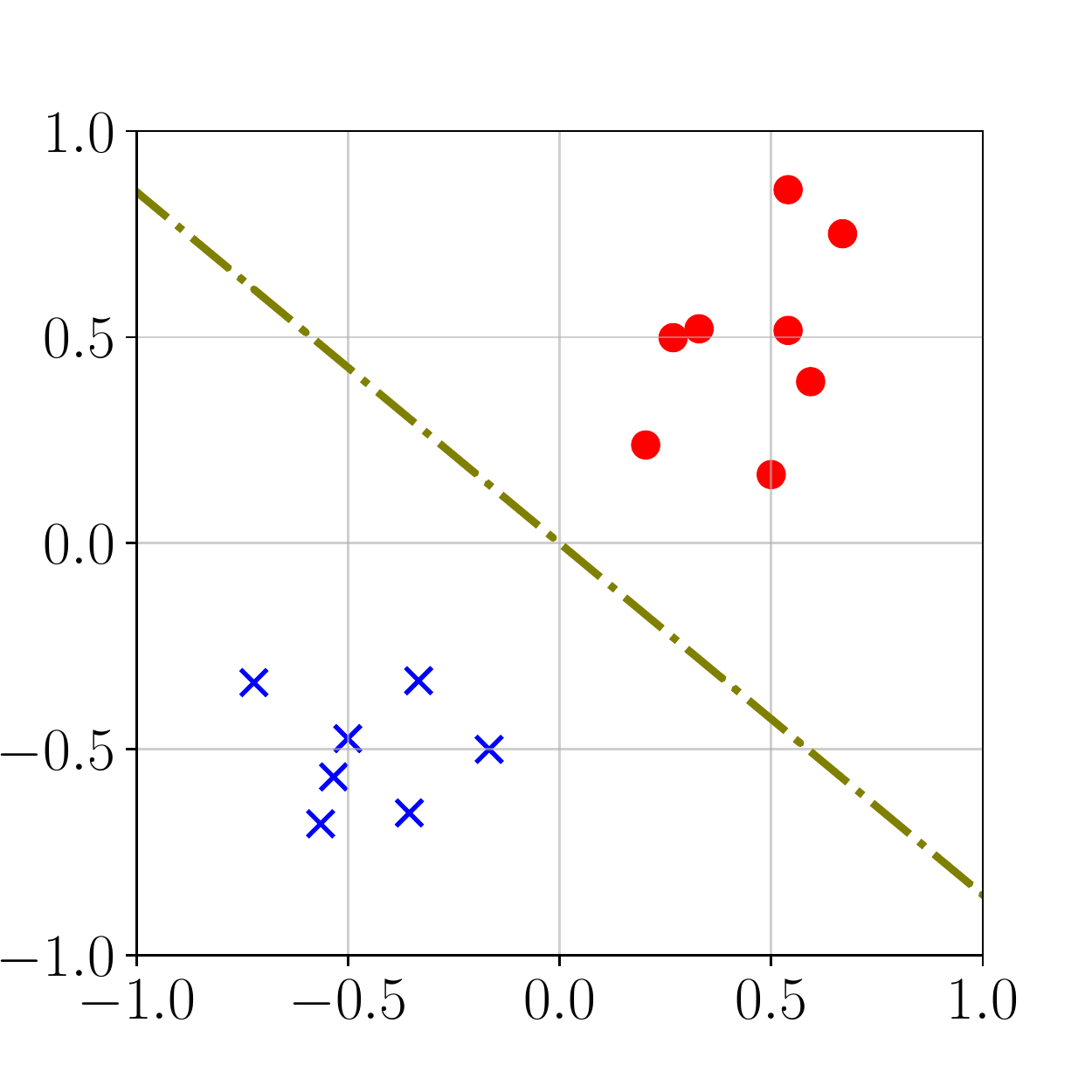}
    \end{subfigure}
    \begin{subfigure}[b]{0.32\textwidth}
        \includegraphics[width=\textwidth]{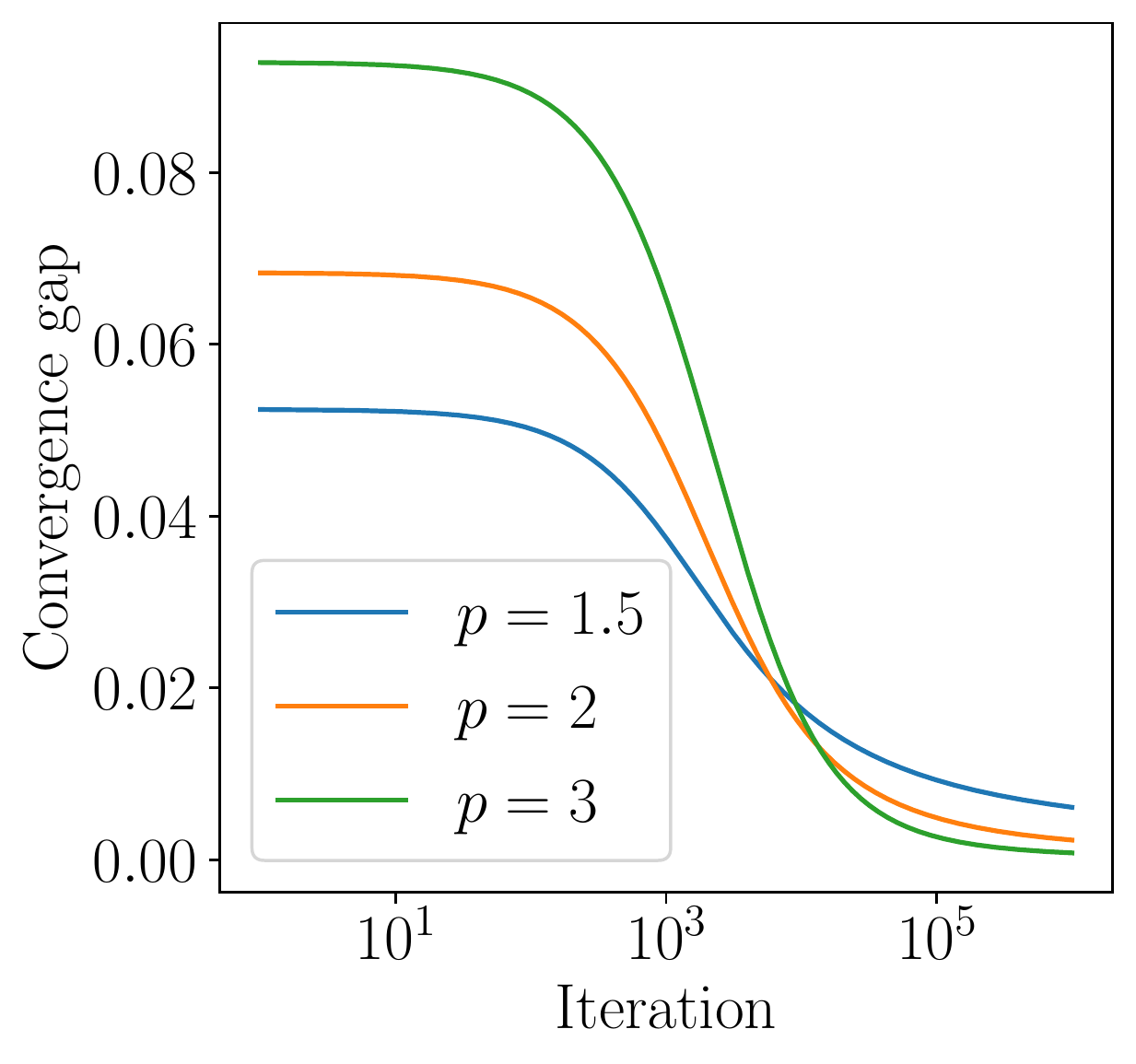}
    \end{subfigure}
    \begin{subfigure}[b]{0.32\textwidth}
        \includegraphics[width=\textwidth]{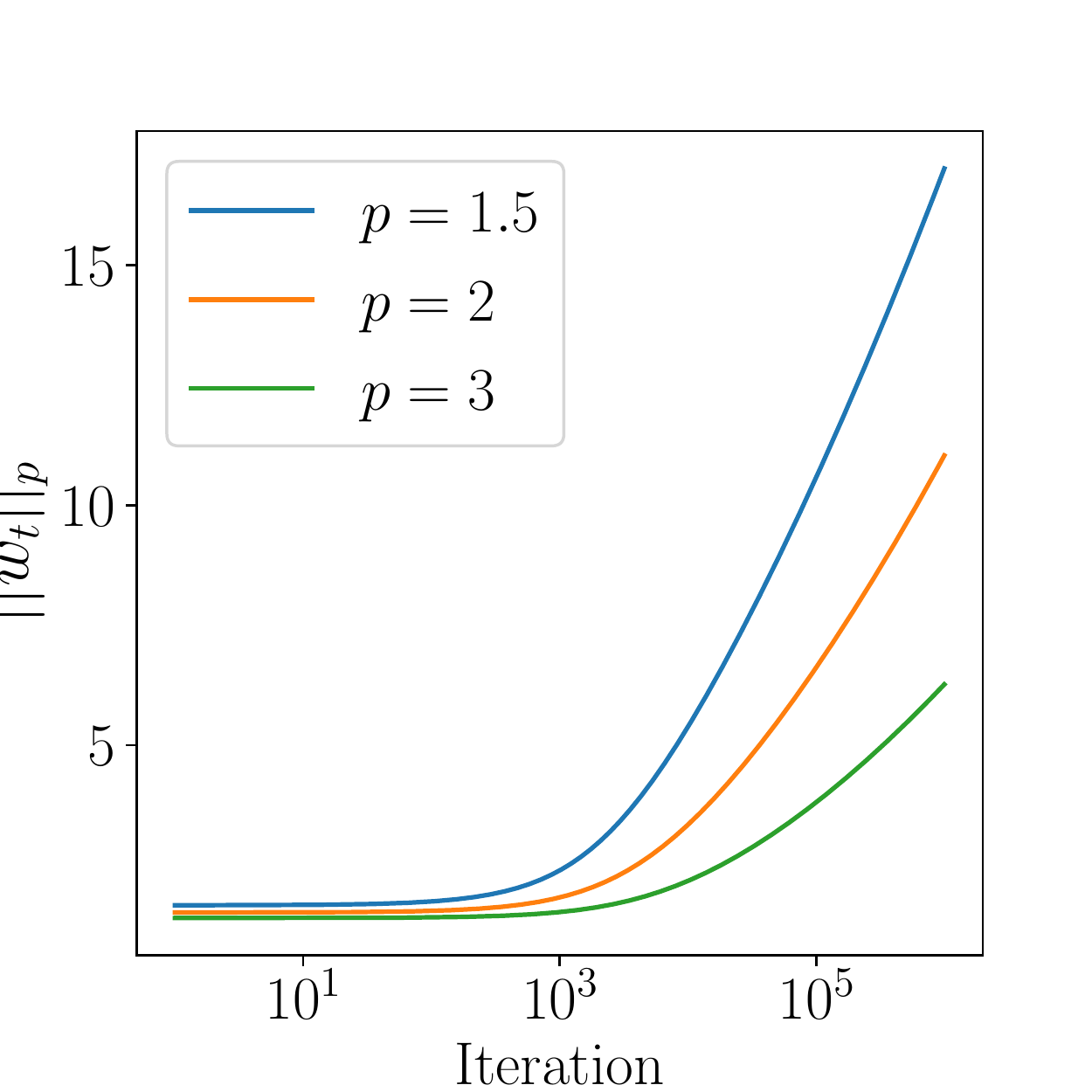}
    \end{subfigure}
    \caption{An example of \algname on randomly generated data with exponential loss and $p = 1.5, 2, 3$. 
    \textbf{(1)} The left plot is a scatter plot of the data: \textcolor{blue}{$\times$}'s and \textcolor{red}{$\bullet$}'s denote the two different labels ($y_i = \pm 1$). The dotted line is the $\ell_2$ max-margin classifier. For clarity, other $\ell_p$ max-margin classifiers are omitted from the plot.
    \textbf{(2)} The middle plot shows the rate which the quantity $\brg{\reg{p}}{w_t / \norm{w_t}_t}$ converges to 0.
    \textbf{(3)} The right plot shows how fast the $p$-norm of $w_t$ growths.
    We can observe that the asymptotic behaviors of these plots are consistent with Corollary~\ref{thm:final-convg-rate}.
    }
    \label{fig:synthetic-data}
\end{figure}

\paragraph{Visualization of the convergence of \algname.}
To visualize the results of Theorem~\ref{thm:primal-bias} and Corollary~\ref{thm:final-convg-rate}, we randomly generated a linearly separable set of 15 points in $\R^2$.
We then employed \algname on this dataset with exponential loss $\ell(z) = \exp(-z)$ and fixed step size $\eta = 10^{-4}$.
We ran this experiment for $p = 1.5, 2, 3$ and for $10^6$ iterations.

In the illustrations of Figure~\ref{fig:synthetic-data}, the mirror descent iterates $w_t$ have unbounded norm and converge in direction to $\mmd{p}$. 
These results are consistent with Lemma~\ref{thm:to-infinity} and with Theorem~\ref{thm:primal-bias}.
Moreover, as predicted by Corollary~\ref{thm:final-convg-rate}, the exact rate of convergence for $\brg{\mmd{p}}{w_t / \norm{w_t}_t}$ is poly-logarithmic with respect to the number of iterations.
Corollary~\ref{thm:final-convg-rate} also indicates that the convergence rate would be faster for larger $p$ due to the larger exponent, and this is consistent with our observation in the second plot of Figure~\ref{fig:synthetic-data}.
Finally, in the third plot of Figure~\ref{fig:synthetic-data}, the norm of the iterates $w_t$ grows at a logarithmic rate, which is the same as the prediction by Lemma~\ref{thm:norm-rate}.

\paragraph{Implicit bias of \algname in linear classification.}
We now verify the conclusions of Theorem~\ref{thm:primal-bias}.
To this end, we recall that $\mmd{p}$ is parallel to the SVM solution $\argmin_w \{\norm{w}_p : \gamma(w) \ge 1\}$.
Hence, we can exploit the linearity and rescale any classifier so that its margin is equal to $1$.
If the prediction of Theorem~\ref{thm:primal-bias} holds, then for each fixed value of $p$, the classifier generated by \algname should have the smallest $\ell_p$-norm after rescaling.

To ensure that $\mmd{p}$ are sufficiently different for different values of $p$, we simulate an over-parameterized setting by randomly select 15 points in $\RR^{100}$.
We used fixed step size of $10^{-4}$ and ran 250 thousand iterations for different $p$'s.

Table~\ref{tab:linear-bias} shows the results for $p = 1.1, 2, 3$ and 10; under each norm, we highlight the smallest classifier in \textbf{boldface}.
Among the four classifiers we presented, \algname with $p = 1.1$ has the smallest $\ell_{1.1}$-norm.
And similar conclusions hold for $p = 2, 3, 10$.
Although \algname converges to $\mmd{p}$ at a very slow rate, we are able to observe a very strong implicit bias of \algname classifiers toward their respective $\ell_p$ geometry in a highly over-parameterized setting.
This suggests we should be able to take advantage of the implicit regularization in practice and at a moderate computational cost.
Due to space constraints, we defer a more complete result with additional values of $p$ to Appendix~\ref{sec:add-experiment-synthetic}.
\begin{wraptable}[8]{r}{0.5\textwidth}  
\centering 
    \centering
{\small 
    \begin{tabular}{l| c|c|c|c}
    & $\ell_{1.1}$ & $\ell_{2}$ & $\ell_{3}$& $\ell_{10}$ \\
    \hline\hline
    $p=1.1$ & \textbf{5.670} & 1.659 & 1.100 & 0.698 \\
    $p=2$ & 6.447 & \textbf{1.273} & 0.710 & 0.393 \\
     $p=3$ & 7.618 & 1.345 & \textbf{0.691} & 0.318 \\
     $p=10$ & 9.086 & 1.520 & 0.742 & \textbf{0.281} \\
    \hline
    \end{tabular}
}
\caption{Size of the linear classifiers generated by \algname (after rescaling) in $\ell_{1.1}, \ell_2, \ell_3$ and $\ell_{10}$ norms.}
\label{tab:linear-bias} 
 \end{wraptable}
 
\subsection{Deep neural networks} 
\label{sec:cifar}
Going beyond linear models, we now investigate \algname{} in deep-learning settings in its impact on the structure of the learned model and potential implications on the generalization performance.
As we had discussed in Section \ref{sec:main-result}, {\bf the implementation of \algname{} is straightforward}; to illustrate simplicity of implementation, we provide code snippets in
Appendix~\ref{sec:practicality}. 
Thus, we are able to effectively experiment with the behaviors \algname in neural network training.
Specifically, we perform a set of experiments on the CIFAR-10 dataset~\citep{krizhevsky2009learning}.
We use the \textit{stochastic} version of \algname with different values of $p$.
We choose a variety of networks: \textsc{VGG}~\citep{simonyan2014very}, \textsc{ResNet}~\citep{he2016deep}, \textsc{MobileNet}~\citep{sandler2018mobilenetv2} and \textsc{RegNet}~\citep{radosavovic2020designing}.

\paragraph{Implicit bias of \algname in deep neural networks.}
Since the notion of margin is not well-defined in this highly nonlinear setting, we instead visualize the impacts of \algname's implicit regularization on the histogram of weights (in absolute value) in the trained model.

In Figure~\ref{fig:cifar10-hist}, we report the weight histograms of \textsc{ResNet-18} models trained under \algname with $p = 1.1, 2, 3$ and $10$.
Depending on $p$, we observe interesting differences between the histograms.
Note that the deep network is most sparse when $p=1.1$ as most weights clustered around $0$.
Moreover, comparing the maximum weights, one can see that the case of $p = 10$ achieves the smallest value. 
Another observation is that the network becomes denser as $p$ increases; for instance, there are more weights away from zero for the cases $p = 3, 10$. 
These overall tendencies are also observed for other deep neural networks; see Appendix~\ref{sec:add-experiment-cifar-bias}.

\begin{figure}
    \centering
    \begin{subfigure}[b]{0.48\textwidth}
        \centering
        \includegraphics[width=\textwidth]{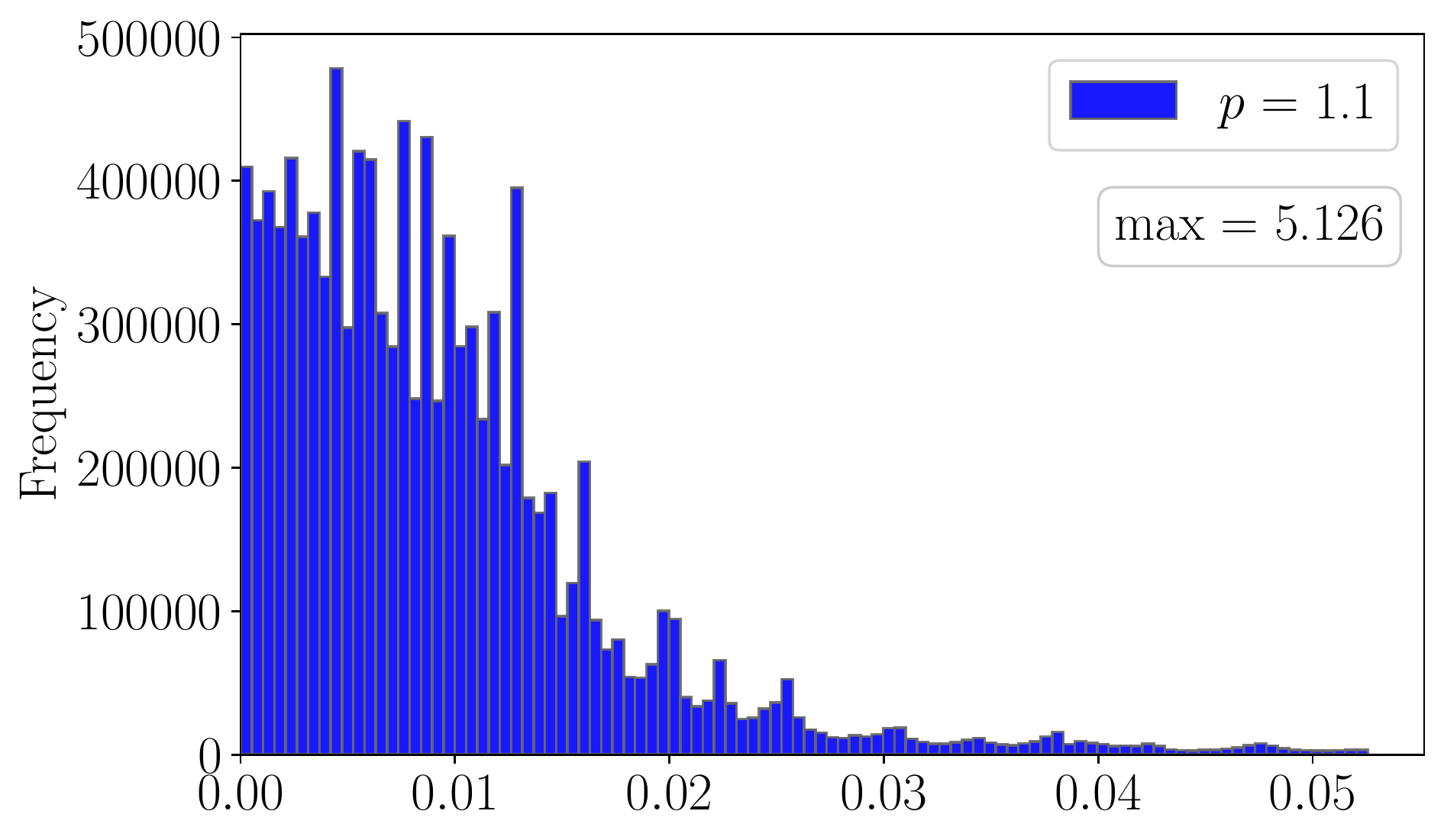}
    \end{subfigure}
    ~
    \begin{subfigure}[b]{0.48\textwidth}
        \includegraphics[width=\textwidth]{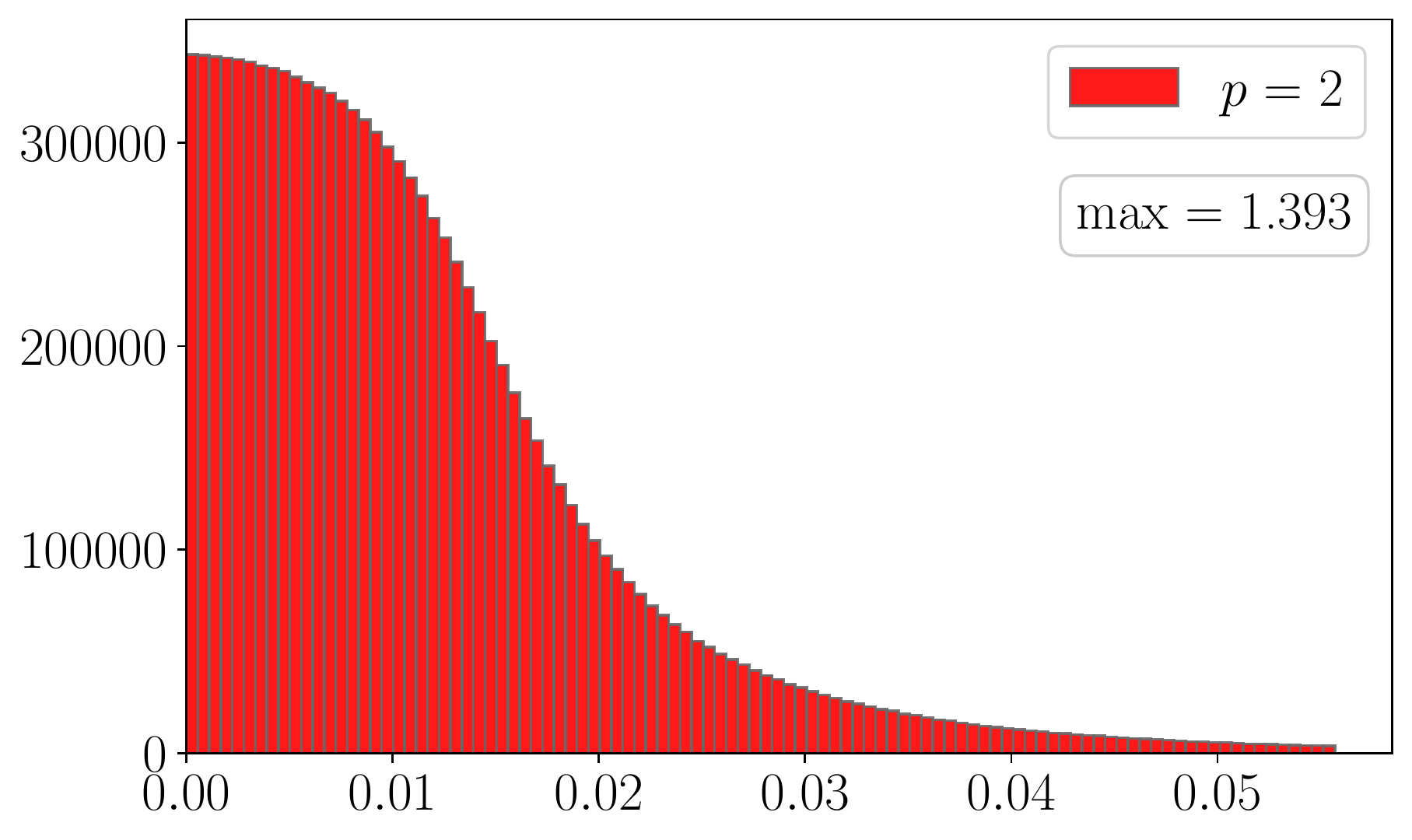}
    \end{subfigure}
    \begin{subfigure}[b]{0.48\textwidth}
        \includegraphics[width=\textwidth]{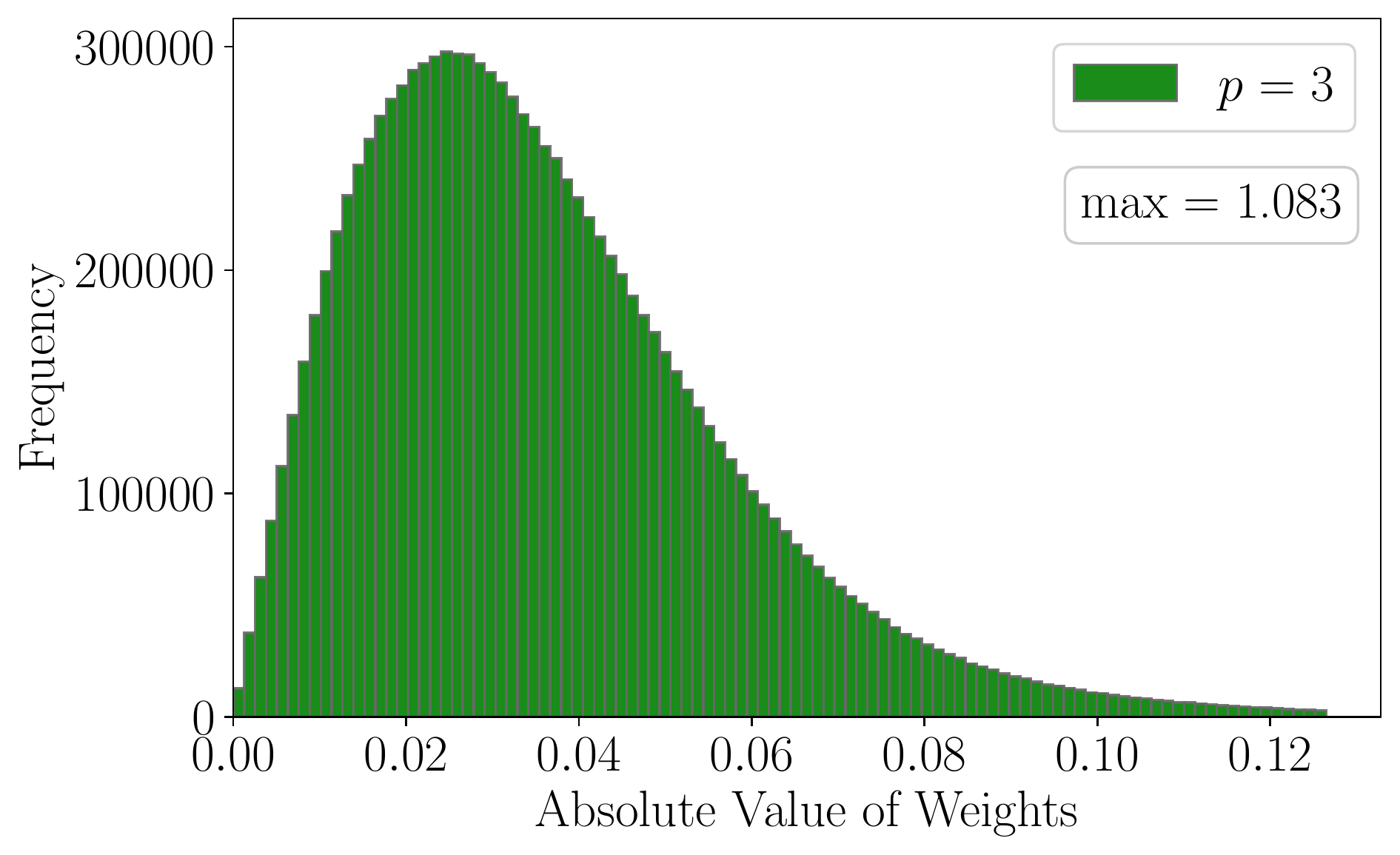}
    \end{subfigure}
    ~
    \begin{subfigure}[b]{0.48\textwidth}
        \includegraphics[width=\textwidth]{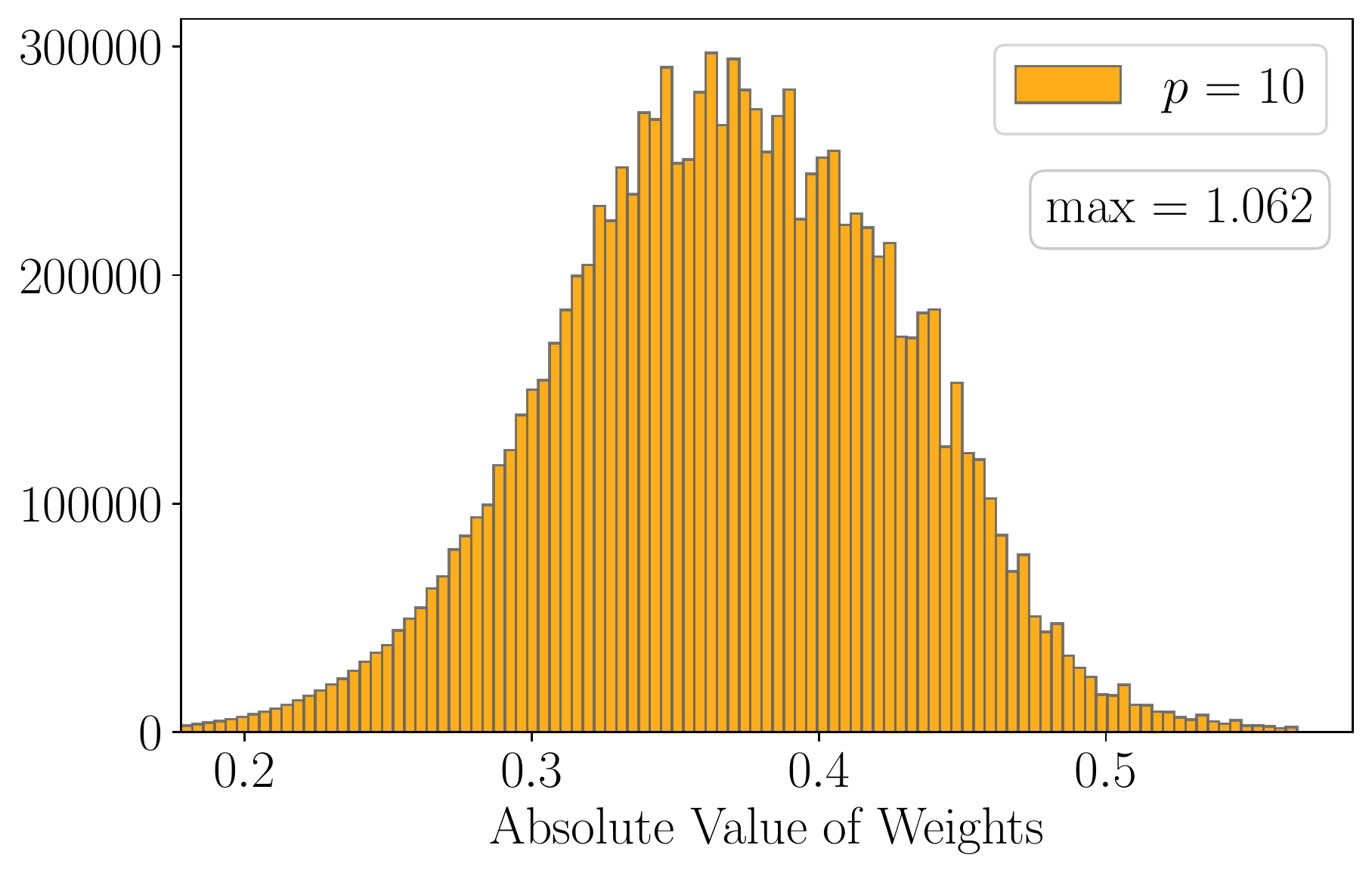}
    \end{subfigure}
    \caption{The histogram of weights in \textsc{ResNet-18} models trained with \algname for the CIFAR-10 dataset. 
    For clarity, we cropped out the tails and each plot has 100 bins after cropping.
    The trends of these histograms reflect the implicit biases of \algname: the distribution of $p = 1.1$ has the most number of weights around zero; and the maximum weight is smallest when $p = 10$.
    }
    \label{fig:cifar10-hist}
\end{figure}

\begin{table}[]
    \centering
    \setlength{\tabcolsep}{5.4pt}
    \begin{tabular}{l| c|c|c|c}
         \hline
         &  \hspace{1.5em} \textsc{VGG-11} \hspace{1.5em} & \hspace{0.75em} \textsc{ResNet-18} \hspace{0.75em} & \textsc{MobileNet-v2} & \textsc{RegNetX-200mf}  \\
         \hline \hline
         $p = 1.1$ & \pmval{88.19}{.17} & \pmval{92.63}{.12} & \pmval{91.16}{.09}& \pmval{91.21}{.18}  \\
         $p = 2$ (SGD) & \pmval{90.15}{.16} & \bpmval{93.90}{.14} & \pmval{91.97}{.10}& \pmval{92.75}{.13} \\
         $p = 3$ & \bpmval{90.85}{.15} & \bpmval{94.01}{.13} & \bpmval{93.23}{.26}& \bpmval{94.07}{.12} \\
         $p = 10$ & \pmval{88.78}{.37} & \pmval{93.55}{.21} & \pmval{92.60}{.22}& \pmval{92.97}{.16} \\
         \hline
    \end{tabular}
    \caption{CIFAR-10 test accuracy (\%) of \algname on various deep neural networks. For each deep network and value of $p$, the average $\pm$ \textcolor{gray}{std. dev.} over 5 trials are reported. And the best performing value(s) of $p$ for each individual deep network is highlighted in \textbf{boldface}.}
    \label{tab:generalization-cifar10}
\end{table}

\paragraph{Generalization performance.}
We next investigate the generalization performance of networks trained with different $p$'s.  
To this end, we adopt a fixed selection of hyper-parameters and then train four deep neural network models to 100\% training accuracy with  \algname with different $p$'s.
As Table~\ref{tab:generalization-cifar10} shows, interestingly the networks trained by \algname with $p = 3$ consistently outperform other choices of $p$'s; notably, for \textsc{MobileNet} and \textsc{RegNet}, the case of $p=3$ outperforms the others by more than 1\%.
Somewhat counter-intuitively, the sparser network trained by \algname with $p = 1.1$ does not exhibit better generalization performance, but rather shows worse generalization than other values of $p$.
\rebuttal{
Although these observations are not directly predicted by our theoretical results, we believe that they nevertheless establish an important step toward understanding generalization of overparameterized models.
}
Due to space limit, we defer other experimental results to Appendix~\ref{sec:add-experiment-cifar-generalization}.

\paragraph{\textsc{ImageNet} experiments.} We also perform a similar set of experiments on the {\sc ImageNet} dataset~\citep{russakovsky2015imagenet}, and these results can be found in Appendix \ref{sec:add-experiment-imagenet}.

\section{Conclusion and Future Work}
\label{sec:conclusion}
In this paper, we establish an important step towards better  understanding implicit bias in the classification setting,   by showing that  \algname converges in direction to the generalized  regularized/max-margin directions.
We also run several experiments to corroborate our main findings along with the practicality of \algname{}. 
The experiments are conducted in various settings: (i) linear models in both low and high dimensions, (ii) real-world data with highly over-parameterized nonlinear models.

We conclude this paper with several important future directions:
\vspace{-8pt}
\begin{list}{{\tiny $\blacksquare$}}{\leftmargin=1.5em}
  \setlength{\itemsep}{-1pt}
\item Our analysis holds for $\mir(\cdot) = \norm{\cdot}_p^p$, where we argued that this choice is key practical interest due to its efficient algorithmic implementations.
It is mathematically interesting to generalize our analysis to other potential functions regardless of practical interest.
\item As we discussed in Section \ref{sec:cifar}, different choices of $p$'s for our \algname algorithm result in different generalization performance. \rebuttal{It would be interesting to investigate this phenomenon and to develop theory that explains why certain values of $p$ lead to better generalization performance}. 
\item Another interesting question is to further investigate how practical techniques used in training neural networks (such as weight decay and adaptive learning rate) can affect the implicit bias and generalization properties of \algname. 
\end{list}

\section*{Acknowledgement}
The authors thank MIT UROP students Tiffany Huang and Haimoshri Das for contributing to the experiments in Section~\ref{sec:cifar}.
The authors acknowledge the MIT SuperCloud and Lincoln Laboratory Supercomputing Center for providing computing resources that have contributed to the results reported within this paper.
NA acknowledges support from the Edgerton Career Development Professorship. This work was supported in part by MIT and MathWorks.

\bibliographystyle{plainnat}
\bibliography{reference}

\clearpage

\section*{Checklist}

\begin{enumerate}

\item For all authors...
\begin{enumerate}
  \item Do the main claims made in the abstract and introduction accurately reflect the paper's contributions and scope?
    \answerYes{All of our claims accurately reflect the results in Sections~\ref{sec:main-result} and \ref{sec:experiments}. }
  \item Did you describe the limitations of your work?
    \answerYes{In Section~\ref{sec:conclusion}, we described several directions where we can improve this work. }
  \item Did you discuss any potential negative societal impacts of your work?
    \answerNA{Our paper investigates the foundational properties of mirror descent algorithms in learning; we do not believe there are any direct societal impacts.}
  \item Have you read the ethics review guidelines and ensured that your paper conforms to them?
    \answerYes{ }
\end{enumerate}

\item If you are including theoretical results...
\begin{enumerate}
  \item Did you state the full set of assumptions of all theoretical results?
    \answerYes{See Section~\ref{sec:priliminaries} for the assumptions we used and the motivation behind them. } 
        \item Did you include complete proofs of all theoretical results?
    \answerYes{All proofs are included in the Appendix, and we have referenced them as we present our claims. }
\end{enumerate}

\item If you ran experiments...
\begin{enumerate}
  \item Did you include the code, data, and instructions needed to reproduce the main experimental results (either in the supplemental material or as a URL)?
    \answerYes{They are included in the supplemental material.}
  \item Did you specify all the training details (e.g., data splits, hyperparameters, how they were chosen)?
    \answerYes{ We gave an overview of our training setup in Section~\ref{sec:experiments} and the full details are given in Appendix~\ref{sec:experiment-detail}. }
        \item Did you report error bars (e.g., with respect to the random seed after running experiments multiple times)?
    \answerYes{We reported the standard deviation whenever multiple trials were performed. }
        \item Did you include the total amount of compute and the type of resources used (e.g., type of GPUs, internal cluster, or cloud provider)?
    \answerYes{ They are reported in Appendix~\ref{sec:experiment-detail}. }
\end{enumerate}

\item If you are using existing assets (e.g., code, data, models) or curating/releasing new assets...
\begin{enumerate}
  \item If your work uses existing assets, did you cite the creators?
    \answerYes{}
  \item Did you mention the license of the assets?
    \answerNA{We only used public datasets.}
  \item Did you include any new assets either in the supplemental material or as a URL?
    \answerNA{We did not curate any new assets.}
  \item Did you discuss whether and how consent was obtained from people whose data you're using/curating?
    \answerNA{}
  \item Did you discuss whether the data you are using/curating contains personally identifiable information or offensive content?
    \answerNA{The datasets we used are well-known; so, we did not feel it was necessary to repeat that they do not contain sensitive information.}
\end{enumerate}

\item If you used crowdsourcing or conducted research with human subjects...
\begin{enumerate}
  \item Did you include the full text of instructions given to participants and screenshots, if applicable?
    \answerNA{There were no human subjects in our work.}
  \item Did you describe any potential participant risks, with links to Institutional Review Board (IRB) approvals, if applicable?
    \answerNA{}
  \item Did you include the estimated hourly wage paid to participants and the total amount spent on participant compensation?
    \answerNA{}
\end{enumerate}

\end{enumerate}


\clearpage

\appendix

\section{Proofs for Section~\ref{sec:priliminaries}}
\label{sec:proof-basic-lemmas}
\subsection{Proof of Lemma~\ref{thm:key-iden}}
The following proof is adopted from~\citep{azizan2021stochastic}.
We make several small modifications to better fit the classification setting in this paper.
In particular, in classification, there is no $w \in \RR^d$ that satisfies $L(w) = 0$.

\begin{proof}
We start with the definition of Bregman divergence:
\[\brg{w}{w_{t+1}}  = \psi(w) - \psi(w_{t+1}) - \inp{\nabla\psi(w_{t+1})}{w-w_{t+1}}.\]
Now, we plugin the \ref{equ:md} update rule $\nabla\psi(w_{t+1}) = \nabla\psi(w_t) - \eta\nabla L(w_t)$:
\[\brg{w}{w_{t+1}} = \psi(w) - \psi(w_{t+1}) - \inp{\nabla\psi(w_{t})}{w - w_{t+1}} + \eta \inp{\nabla L(w_t)}{w - w_{t+1}}. \]
We again invoke the definition of Bregman divergence so that:
\begin{align*}
    \brg{w}{w_{t+1}}  &= \psi(w) - \psi(w_{t+1}) - \inp{\nabla\psi(w_{t+1})}{w-w_{t+1}}, \\
    \brg{w_{t+1}}{w_t}  &= \psi(w_{t+1}) - \psi(w_t) - \inp{\nabla\psi(w_t)}{w_{t+1}-w_t}.
\end{align*}

It follows that
\begin{equation}
\label{equ:proof-key-iden-1}
    \begin{aligned}
        \brg{w}{w_{t+1}} 
        &= \psi(w) - \psi(w_{t}) - \inp{\nabla\psi(w_{t})}{w - w_{t}} \\
        &\hspace{7.25em}+ \inp{\nabla\psi(w_{t})}{w_{t+1} - w_t} - \psi(w_{t+1}) + \psi(w_t) \\
        &\hspace{7.25em}+ \eta \inp{\nabla L(w_t)}{w - w_{t+1}} \\
        &= \brg{w}{w_t} - \brg{w_{t+1}}{w_t} + \eta \inp{\nabla L(w_t)}{w - w_{t+1}}
    \end{aligned}
\end{equation}

Next, we consider the term $\inp{\nabla L(w_t)}{w - w_{t-1}}$:
\begin{equation}
    \label{equ:proof-key-iden-2}
    \begin{aligned}
        \inp{\nabla L(w_t)}{w - w_{t-1}}
        &= \inp{\nabla L(w_t)}{w - w_t} - \inp{\nabla L(w_t)}{w_{t+1} - w_t} \\
        &\hspace{8em}+ L(w_{t+1}) - L(w_t) - L(w_{t+1}) + L(w_t) \\
        &= \inp{\nabla L(w_t)}{w - w_t} + D_L(w_{t+1}, w_t) - L(w_{t+1}) + L(w_t),
    \end{aligned}
\end{equation}
where the last step holds because $L$ is convex.

Combining \eqref{equ:proof-key-iden-1} and \eqref{equ:proof-key-iden-2} yields:
\begin{align*}
    &\brg{w}{w_t} \\
    ={}& \brg{w}{w_{t+1}} + \brg{w_{t+1}}{w_t} -\eta\big( \inp{\nabla L(w_t)}{w - w_t} + D_L(w_{t+1}, w_t) - L(w_{t+1}) + L(w_t)\big)\\
    ={}& \brg{w}{w_{t+1}} + \breg{\psi-\eta L}{w_{t+1}}{w_t} - \eta  \inp{\nabla L(w_t)}{w - w_t} + \eta L(w_{t+1}) - \eta L(w_t),
\end{align*}
where in the last step, we note that Bregman divergence is additive in its potential.
This gives us \eqref{equ:key-iden-2}.
And for \eqref{equ:key-iden-1}, we use the definition of Bregman divergence again, i.e. $D_L(w, w_t) = L(w) - L(w_t) - \inp{\nabla L (w_t)}{w - w_t}$:
\begin{align*}
    \brg{w}{w_t}
    &= \brg{w}{w_{t+1}} + \breg{\psi-\eta L}{w_{t+1}}{w_t} - \eta  \inp{\nabla L(w_t)}{w - w_t} \\
    &\hspace{8em}  + \eta L(w) - \eta L(w_t)  + \eta L(w_{t+1}) - \eta L(w) \\
    &= \brg{w}{w_{t+1}} + \breg{\psi-\eta L}{w_{t+1}}{w_t} + \eta  D_L(w, w_t) - \eta L(w)  + \eta L(w_{t+1}) \\
\end{align*}

\end{proof}

\subsection{Proof of Lemma~\ref{thm:decreasing-lose}}
\begin{proof}
This is an application of Lemma~\ref{thm:key-iden} with $w = w_t$:
\begin{align*}
    0 &= D_\psi(w_t, w_{t+1}) + D_{\psi - \eta L}(w_{t+1}, w_t) - \eta L(w_t) + \eta L(w_{t+1}) \\
    \implies \eta L(w_t) &= D_\psi(w_t, w_{t+1}) + D_{\psi - \eta L}(w_{t+1}, w_t) + \eta L(w_{t+1}) \ge \eta L(w_{t+1})
\end{align*}
where we used the fact that Bregman divergence with a convex potential function is non-negative.
\end{proof}

\subsection{Proof of Lemma~\ref{thm:to-infinity}}
\begin{proof}
By Lemma~\ref{thm:decreasing-lose}, $L(w_t)$ is decreasing with respect to $t$, therefore the limit exists.
Suppose the contrary that $\lim_{t\to\infty} L(w_t) = \varepsilon > 0$.
Since the data is separable, we can pick $w$ so that $L(w) \le \varepsilon / 2$.
Applying Lemma~\ref{thm:key-iden}, the following holds for all $t$:
\begin{align*}
    D_\psi(w, w_{t+1}) 
    &= D_\psi(w, w_{t}) - D_{\psi - \eta L}(w_{t+1}, w_t) - \eta D_{L}(w, w_t) + \eta L(w) - \eta L(w_{t+1}) \\
    &\le D_\psi(w, w_{t}) + \eta\varepsilon/2 - \eta\varepsilon = D_\psi(w, w_{t}) - \eta\varepsilon/2
\end{align*}
Hence, $D_\psi(w, w_{t}) \le D_\psi(w, w_0) - t\eta\varepsilon / 2$.
This implies that $\limsup_{t\to\infty} D_\psi(w, w_{t}) = -\infty$, contradiction.
\end{proof}

\section{Proofs for Section~\ref{sec:main-result}}
\subsection{Proof of Lemma~\ref{thm:approx-reg-dir-loss}}
\begin{proof}
Let $\bar{\gamma}$ be the margin of $\reg{p}$.
Under separability, we know $\bar{\gamma} > 0$.
Recall the definition of the regularization path.
There exists sufficiently large $r_\alpha$ so that 
\[ \norm{\frac{\bar{w}_p(\norm{w}_p)}{\norm{w}_p} - \reg{p}}_p \le \frac{\alpha \bar{\gamma}}{C} \]
whenever $\norm{w}_p \ge r_\alpha$.
Recall the definition that $C = \max_{i = 1, \dots, n} \norm{x_i}_q, 1/p + 1/q = 1$.
Then, for all $i \in [n]$, we have
\begin{align*}
    y_i \inp{\bar{w}(\norm{w}_p)}{x_i}
    &= y_i \inp{\bar{w}(\norm{w}_p) - \norm{w}_p\reg{p}}{x_i} + y_i \inp{\norm{w}_p\reg{p}}{x_i} \\
    &\le \alpha \bar{\gamma} \norm{w}_p \norm{x_i}_q / C + y_i \inp{\norm{w}_p\reg{p}}{x_i} \\
    &\le \alpha \bar{\gamma} \norm{w}_p + y_i \inp{\norm{w}_p\reg{p}}{x_i} \\
    &\le y_i \inp{(1+\alpha) \norm{w}_p \reg{p}}{x_i}
\end{align*}
Since the loss $L$ is decreasing, we have
\[L((1+\alpha)\norm{w}_p\reg{p}) \le L(\bar{w}(\norm{w}_p)) \le L(w). \]
\end{proof}

\subsection{Lower bounding the mirror descent updates}
\label{sec:cross-term-diff}
\begin{lemma}
\label{thm:cross-term-diff}
For $\psi(\cdot)  = \frac{1}{p} \norm{\cdot}_p^p$ with $p > 1$, the mirror descent update satisfies the following inequality:
\begin{equation}
    \frac{p-1}{p} \norm{w_{t+1}}_p^p - \frac{p-1}{p} \norm{w_{t}}_p^p + \eta L (w_{t+1}) - \eta L (w_{t}) \le \inp{-\eta\nabla L (w_{t})}{w_t}
\end{equation}
\end{lemma}
\begin{proof}
This result follows from Lemma~\ref{thm:key-iden} with $w = 0$:
\begin{align*}
    \brg{0}{w_t}
    &= \brg{0}{w_{t+1}} + \breg{\psi-\eta L}{w_{t+1}}{w_t} + \eta \brgl{0}{w_t}+ \eta L(w_{t+1}) - \eta L(0) \\
    &\ge \brg{0}{w_{t+1}} + \eta \brgl{0}{w_t}+ \eta L(w_{t+1}) - \eta L(0) \\
    &= \brg{0}{w_{t+1}} + \eta(L(0) - L(w_t) - \inp{\nabla L(w_t)}{-w_t}) + \eta L(w_{t+1}) - \eta L(0) \\
    &= \brg{0}{w_{t+1}} + \eta \inp{\nabla L(w_t)}{w_t} + \eta L(w_{t+1}) - \eta L(w_t)
\end{align*}
Rearranging the terms yields
\[ \brg{0}{w_{t+1}} - \brg{0}{w_{t}} + \eta L (w_{t+1}) - \eta L (w_{t}) \le \inp{-\eta\nabla L (w_{t})}{w_t} \]
We conclude the proof by noting that for any $w \in \RR^d$, 
\[\brg{0}{w}  = \psi(0) - \psi(w) - \inp{\nabla\psi(w)}{-w} = \inp{\nabla\psi(w)}{w} - \psi(w) = \frac{p-1}{p} \norm{w}_p^p \]
\end{proof}

\subsection{Proof of Theorem~\ref{thm:primal-bias}}
\label{sec:proof-primal-bias}
\begin{proof}
Consider arbitrary $\alpha \in (0, 1)$ and define $r_\alpha$ according to Lemma~\ref{thm:approx-reg-dir-loss}.
Since $\lim_{t\to\infty}\norm{w_t}_p = \infty$, we can find $t_0$ so that $\norm{w_t}_p > \max(1, r_\alpha)$ for all $t \ge t_0$.
Let $c_t = (1+\alpha)\norm{w_t}_p$.

We list some properties about $\psi(\cdot) = \frac{1}{p}\norm{\cdot}_p^p$ that will be useful in our analysis: 
\begin{align*}
 \nabla \psi (w) &= (\sgn(w_1)|w_1|^{p-1},\cdots, \sgn(w_d)|w_d|^{p-1})\\
 \inp{\nabla \psi (w)}{w} &= \sgn(w_1)w_1|w_1|^{p-1} + \cdots + \sgn(w_d)w_d|w_d|^{p-1} = \norm{w}_p^p\\
     \norm{\nabla\psi(w)}_q  &= \norm{w}_p^{p-1}, \text{ for } 1/p+1/q = 1 \\
  \brg{c w}{c w'} &= |c|^p \brg{w}{w'} \quad \forall c\in \R.
\end{align*}

Substitute $w = c_t \reg{p}$ into Lemma~\ref{thm:key-iden}, we get
\[\brg{c_t\reg{p}}{w_{t+1}} \le \brg{c_t\reg{p}}{w_t} + \eta\inp{\nabla L(w_t)}{c_t \reg{p} - w_t} - \eta L(w_{t+1}) + \eta L(w_t).\]
By Corollary~\ref{thm:cross-term}, we have $\inp{\nabla L(w_t)}{c_t \reg{p} - w_t} \le 0$.
Therefore,
\[\brg{c_t\reg{p}}{w_{t+1}} \le \brg{c_t\reg{p}}{w_t} - \eta L(w_{t+1}) + \eta L(w_t).\]

It follows that
\begin{align*}
    &\brg{c_{t+1}\reg{p}}{w_{t+1}} \\
    \le{}& \brg{c_t\reg{p}}{w_t} - \eta L(w_{t+1}) + \eta L(w_t) + \brg{c_{t+1}\reg{p}}{w_{t+1}} - \brg{c_t\reg{p}}{w_{t+1}} \\
    ={}& \brg{c_t\reg{p}}{w_t} - \eta L(w_{t+1}) + \eta L(w_t) + \psi(c_{t+1} \reg{p}) - \psi(c_t \reg{p}) - \inp{\nabla\psi(w_{t+1})}{(c_{t+1} - c_t) \reg{p}}
\end{align*}
Summing over $t = t_0, \dots, T-1$ gives us
\begin{align}
    \brg{c_T\reg{p}}{w_T}
    &\le \brg{c_{t_0}\reg{p}}{w_{t_0}} - \eta L(w_{t_0}) + \eta L(w_T) + \psi(c_{T} \reg{p}) - \psi(c_{t_0} \reg{p}) \nonumber\\
    &\quad\quad- \sum_{t=t_0}^{T-1}\inp{\nabla\psi(w_{t+1})}{(c_{t+1} - c_t) \reg{p}} \label{equ:succ-cross-term}
\end{align}

Now we want to establish a lower bound on the last term of \eqref{equ:succ-cross-term}.
To do so, we inspect the change in $\nabla\psi(w_t)$ from each successive mirror descent update:
\begin{subequations}
\begin{align}
    &\inp{\nabla\psi(w_{t+1}) - \nabla\psi(w_t)}{\reg{p}} \\
    ={}& \inp{-\eta\nabla L(w_{t})}{\reg{p}} \\
    \ge{}& \frac{1}{(1+\alpha)\norm{w_t}_p} \inp{-\eta\nabla L(w_t)}{w_t} \label{equ:cross-term-expansion-L1}\\
    \ge{}& \frac{1}{(1+\alpha)\norm{w_t}_p} \left(\frac{p-1}{p} \norm{w_{t+1}}_p^p - \frac{p-1}{p} \norm{w_{t}}_p^p + \eta L(w_{t+1}) - \eta L(w_{t})\right) \label{equ:cross-term-expansion-L2} \\
    \ge{}& \frac{1}{(1+\alpha)\norm{w_t}_p} \left(\frac{p-1}{p} \norm{w_{t+1}}_p^p - \frac{p-1}{p} \norm{w_{t}}_p^p\right) + \eta L(w_{t+1}) - \eta L(w_{t}) \label{equ:cross-term-expansion-L3}
\end{align}
\end{subequations}
where we applied Corollary~\ref{thm:cross-term} on \eqref{equ:cross-term-expansion-L1} and Lemma~\ref{thm:cross-term-diff} on \eqref{equ:cross-term-expansion-L2}.

Now we bound \eqref{equ:cross-term-expansion-L3}.
We claim the following identity and defer its derivation to Section~\ref{sec:main-thm-aux-pow}.
\begin{equation}
    \label{equ:p-norm-diff}
    \frac{p-1}{p} (\norm{w_{t+1}}_p^p - \norm{w_t}_p^p) \ge (\norm{w_{t+1}}_p^{p-1} - \norm{w_{t}}_p^{p-1}) \norm{w_t}_p.
\end{equation}

We are left with
\begin{equation*} \inp{\nabla\psi(w_{t+1}) - \nabla\psi(w_t)}{\reg{p}} \ge \frac{\norm{w_{t+1}}_p^{p-1} - \norm{w_{t}}_p^{p-1}}{1+\alpha} + \eta L(w_{t+1}) - \eta L(w_{t}).
\end{equation*}

Summing over $t = t_0, \dots, T-1$ gives us
\begin{equation}\label{equ:mirror-cross-term}
\inp{\nabla\psi(w_{T}) - \nabla\psi(w_{t_0})}{\reg{p}} \ge \frac{\norm{w_T}_p^{p-1} - \norm{w_{t_0}}_p^{p-1}}{1+\alpha} + \eta L(w_{T}) - \eta L(w_{t_0}).
\end{equation}

With \eqref{equ:mirror-cross-term}, we can bound the last term of \eqref{equ:succ-cross-term} as follows:
\begin{align}
    \sum_{t=t_0}^{T-1}\inp{\nabla\psi(w_{t+1})}{(c_{t+1} - c_t) \reg{p}} \nonumber
    &\ge \sum_{t=t_0+1}^{T} \frac{\norm{w_t}_p^{p-1} + O(1)}{1+\alpha}(c_t - c_{t-1}) \nonumber\\
    &= \sum_{t=t_0+1}^{T} (\norm{w_t}_p^{p-1} + O(1))(\norm{w_t}_p - \norm{w_{t-1}}_p) \nonumber\\
    &\ge \sum_{t=t_0+1}^{T} \frac{1}{p}(\norm{w_t}_p^{p} - \norm{w_{t-1}}_p^p) + O(1) \cdot (\norm{w_T}_p - \norm{w_{t_0}}_p)\nonumber\\
    &= \frac{1}{p} \norm{w_T}_p^p + O(\norm{w_T}_p) \label{equ:cross-telescoping}
\end{align}
where we defer the computation on the last inequality to Section~\ref{sec:main-thm-aux-pow}.

We now apply the inequality in \eqref{equ:cross-telescoping} to \eqref{equ:succ-cross-term}.
Note that $\psi(c_T \reg{p}) = \frac{1}{p}(1+\alpha)^p\norm{w_T}_p^p$.
We now have the following:
\[ \brg{(1+\alpha)\norm{w_T}_p \reg{p}}{w_T} \le \frac{1}{p} \norm{w_T}_p^p ((1+\alpha)^p - 1) + O(\norm{w_T}_p).\]
After applying homogeneity of Bregman divergence, and recalling that $\alpha = \frac{\varepsilon}{1-\varepsilon}$, we have
\[ \brg{\reg{p}}{(1-\varepsilon)\frac{w_T}{\norm{w_T}_p}} \le \frac{\frac{1}{p} \norm{w_T}_p^p (1 - (1-\varepsilon)^{p})}{\norm{w_T}_p^p} + o(1).\]
Let $\tilde{w}_T = \frac{w_T}{\norm{w_T}_p}$.
We note that Bregman divergence in fact satisfies the Law of Cosine:
\begin{lemma}[Law of Cosine]
\label{thm:breg-loc}
\begin{equation*}
    \brg{w}{w'} = \brg{w}{w''} + \brg{w''}{w'} - \inp{\nabla\psi(w') - \nabla\psi(w'')}{w - w''}
\end{equation*}
\end{lemma}

Therefore,

\begin{equation}
    \label{equ:final-limit}
    \begin{aligned}
    \brg{\reg{p}}{\tilde{w}_T} 
    &\le \frac{\frac{1}{p} \norm{w_T}_p^p (1 - (1-\varepsilon)^{p})}{\norm{w_T}_p^p} + \brg{(1-\varepsilon)\tilde{w}_T}{\tilde{w}_T} \\
    &\hspace{6em} - \inp{\nabla\psi(\tilde{w}_T) - \nabla\psi((1-\varepsilon)\tilde{w}_T)}{\reg{p} - (1-\varepsilon)\tilde{w}_T} + o(1) \\
    &\le \frac{1}{p} (1 - (1-\varepsilon)^{p}) + \frac{1}{p}((1-\varepsilon)^p - 1) + \varepsilon + 2d^{1/p}(1 - (1-\varepsilon)^p) + o(1)
    \end{aligned}
\end{equation}
And we defer the computation for the last inequality to Section~\ref{sec:main-thm-aux-pow}.
Taking the limit as $T \to \infty$ and $\varepsilon \to 0$, we have that
\begin{equation}
    \begin{aligned}
    \limsup_{T\to\infty} \brg{\reg{p}}{\frac{w_T}{\norm{w_T}_p}}
    &\le \varepsilon + 2d^{1/p}(1 - (1-\varepsilon)^p)
    \end{aligned}
\end{equation}
Note that the RHS vanishes in the limit as $\varepsilon \to 0$.
Since the choice of $\varepsilon$ is arbitrary, we have $w_T/\norm{w_T}_p \to \reg{p}$ as $T \to\infty$.

\end{proof}

\subsection{Auxiliary Computation for Section~\ref{sec:proof-primal-bias}}
\label{sec:main-thm-aux-pow}

To show \eqref{equ:p-norm-diff}, we claim that for $\delta \ge -1$ and $p > 1$, we have 
\[ \frac{p-1}{p} ((1+\delta)^p - 1) \ge (1+\delta)^{p-1} - 1. \]
Note that we equality when $\delta=0$, and now we consider the first derivative:
\[\frac{d}{d\delta} \left\{\frac{p-1}{p} ((1+\delta)^p - 1) - (1+\delta)^{p-1} + 1 \right\} = (p-1)\delta(1+\delta)^{p-2},\]
which is negative when $\delta \in [-1, 0)$ and positive when $\delta > 0$, so this identity holds.
Now, \eqref{equ:p-norm-diff} follows from setting $\delta = (\norm{w_{t+1}}_p - \norm{w_t}_p)/\norm{w_{t}}_p$ and then multiplying by $\norm{w_t}_p^p$ on both sides.

To finish showing \eqref{equ:cross-telescoping}, we claim that for $\delta \ge -1$ and $p > 1$, we have 
\[ \frac{1}{p}((1+\delta)^p - 1) \le \delta (1+\delta)^{p-1}. \]
Note that we equality when $\delta=0$, and now we consider the first derivative:
\[\frac{d}{d\delta} \left\{\frac{1}{p}((1+\delta)^p - 1) - \delta (1+\delta)^{p-1}\right\} = -(p-1)\delta(1+\delta)^{p-2},\]
which is positive when $\delta \in [-1, 0)$ and negative when $\delta > 0$, so this identity holds.
Now, the last inequality of \eqref{equ:cross-telescoping} follows by setting $\delta = (\norm{w_{t}}_p - \norm{w_{t-1}}_p)/\norm{w_{t-1}}_p$ and then multiply by $\norm{w_t}_p^p$ on both sides.

Finally, we simplify the RHS of \eqref{equ:final-limit} by taking advantage of the fact that $\tilde{w}_T$ is normalized:
\begin{align*}
    \brg{(1-\varepsilon)\tilde{w}_T}{\tilde{w}_T} 
    &= (1-\varepsilon)^p \psi(\tilde{w}_T) - \psi(\tilde{w}_T) + \inp{\nabla\psi(\tilde{w}_T)}{\varepsilon \tilde{w}_T} \\
    &= \frac{1}{p}((1-\varepsilon)^p - 1) + \varepsilon
\end{align*}
\begin{align*}
    &\left|\inp{\nabla\psi(\tilde{w}_T) - \nabla\psi((1-\varepsilon)\tilde{w}_T)}{\reg{p} - (1-\varepsilon)\tilde{w}_T}\right|\\
    ={}& \left|\inp{(1-(1-\varepsilon)^p)\nabla\psi(\tilde{w}_T)}{\reg{p} - (1-\varepsilon)\tilde{w}_T}\right|\\
    \le{}& (1 - (1-\varepsilon)^p) \norm{\nabla\psi(\tilde{w_T})}_q \cdot \norm{\reg{p} - (1-\varepsilon)\tilde{w}_T}_p \\
    ={}& (1 - (1-\varepsilon)^p) \norm{\tilde{w_T}}_p^{p-1} \cdot \norm{\reg{p} - (1-\varepsilon)\tilde{w}_T}_p \\
    \le{}& 2d^{1/p}(1 - (1-\varepsilon)^p)
\end{align*}

\subsection{Proof of Theorem~\ref{thm:reg-max-dir}}
\label{sec:proof-reg-max-dir}
\begin{proof}
We first show that $\mmd{p}$ is unique.
Suppose the contrary that there are two distinct unit $p$-norm vectors $u_1 \neq u_2$ both achieving the maximum-margin $\mar{p}$.
Then $u_3 = (u_1 + u_2)/2$ satisfies
\[ \forall i, y_i \inp{u_3}{x_i} = \frac{1}{2} y_i \inp{u_1}{x_i} + \frac{1}{2} y_i \inp{u_2}{x_i} \ge \mar{p} \]
Therefore, $u_3$ has margin of at least $\mar{p}$.
Since $\norm{\cdot}_p$ is strictly convex, we must have $\norm{u_3}_p < 1$.
Therefore, the margin of $u_3 / \norm{u_3}_p$ is strictly greater than $\mar{p}$, contradiction.

Define $\beta > 0$ so that $\ell(z) e^{az} \in [b/2, 2b]$ for $z = B\mar{p}/2$ and whenever $B > \beta$.
Note that
\[L(B\mmd{p}) = \sum_{i=1}^n \ell(y_i\inp{B\mmd{p}}{x_i}) \le n \cdot \ell(B\mar{p}) \le 2 b n \cdot \exp(-aB\mar{p})\]

Suppose the contrary that the regularized direction does not converge to $\mmd{p}$, then there must exist $\mar{p}/2 > \varepsilon > 0$ so that there are arbitrarily large values of $B$ satisfying
\[\min_{i=1, \dots, n} y_i \inp{\frac{\bar{w}(B)}{B}}{x_i} \le \mar{p} - \varepsilon. \]
And this implies
\[ L(\bar{w}(B)) \ge \ell(B(\mar{p} - \varepsilon)) \ge \frac{b}{2} \exp(-a B\mar{p}) \exp(a B\varepsilon)\]

Then, for sufficiently large $B > \beta$, we have $\exp(aB\varepsilon) > 4 n \Rightarrow L(\bar{w}(B)) > L(B\mmd{p})$, contradiction.
Therefore, the regularized direction exists and $\reg{p} = \mmd{p}$.
\end{proof}

\subsection{Simpler proof of Theorem~\ref{thm:primal-bias}}
For potential function $\psi(\cdot) = \frac{1}{p} \norm{\cdot}_p^p$, we can avoid most calculations involving \eqref{equ:succ-cross-term} by directly computing for Bregman divergence.
However, we want to note that this approach is less general, and does not highlight the role of $\reg{p}$ as clearly.

\begin{proof}
Consider arbitrary $\alpha \in (0, 1)$.
Since $\lim_{t\to\infty}\norm{w_t}_p = \infty$, we can find $t_0$ so that $\norm{w_t} > \max(1, r_\alpha)$ for all $t \ge t_0$.
For $T > t_0$, define $\tilde{w}_T = \frac{w_T}{\norm{w_T}_p}$.

We can perform the following manipulation on Bregman divergence:
\begin{equation}
\label{equ:breg-expa}
    \begin{aligned}
        \brg{\reg{p}}{\tilde{w}_T}
        &= \psi(\reg{p}) - \psi\left(\tilde{w}_T\right) - \inp{\nabla\psi\left(\tilde{w}_T\right)}{\reg{p} - \tilde{w}_T} \\
        &= \psi(\reg{p}) - \psi\left(\tilde{w}_T\right) + \inp{\nabla\psi(\tilde{w}_T)}{\tilde{w}_T} - \inp{\nabla\psi(\tilde{w}_T)}{\reg{p}} \\
        &= \frac{1}{p}\norm{\reg{p}}_p^p - \frac{1}{p} \norm{\tilde{w}_T}_p^p + \norm{\tilde{w}_T}_p^p - \inp{\nabla\psi(\tilde{w}_T)}{\reg{p}} \\
        &= 1 - \inp{\nabla\psi(\tilde{w}_T)}{\reg{p}}
    \end{aligned}
\end{equation}

We divide both sides of \eqref{equ:mirror-cross-term} by $\norm{w_T}$ and then taking the limit as $T \to \infty$ yields
\begin{equation}\label{equ:cross-term-limit}
\liminf_{T\to\infty} \frac{1}{\norm{w_T}_p^{p-1}}\inp{\nabla\psi(w_{T})}{\reg{p}} \ge \frac{1}{1+\alpha}.
\end{equation}

Now, substituting \eqref{equ:cross-term-limit} into \eqref{equ:breg-expa} results in
\begin{align*}
    \limsup_{T\to\infty} \brg{\reg{p}}{\frac{w_T}{\norm{w_T}_p}} 
    &= 1 - \liminf_{T\to\infty} \inp{\nabla\psi\left(\frac{w_T}{\norm{w_T}_p}\right)}{\reg{p}} \\
    &= 1 - \liminf_{T\to\infty} \frac{1}{\norm{w_T}_p^{p-1}} \inp{\nabla\psi(w_{T})}{\reg{p}} \\
    &\le 1 - \frac{1}{1+\alpha} < \alpha
\end{align*}

Since the value of $\alpha$ is arbitrary, we can conclude that
\[\lim_{T\to\infty} \brg{\reg{p}}{\frac{w_T}{\norm{w_T}_p}} = 0.\]
\end{proof}

\section{Proofs for Section~\ref{sec:asymp-result}}
\label{sec:proof-asymp-result}
\subsection{Proof of Corollary~\ref{thm:convg-rate}}
\begin{proof}
    This is an immediate consequence of \eqref{equ:breg-expa} and \eqref{equ:cross-term-limit}.
\end{proof}

\subsection{Proof of Lemma~\ref{thm:norm-rate}}
For the following proof, we assume without loss of generality that $y_i = 1$ by replacing every instance of $(x_i, -1)$ with $(-x_i, 1)$.

\begin{proof}
    For the upper bound, we consider a reference vector $w^\star = \mar{p}^{-1} \mmd{p}$.
    By the definition of the max-margin direction, the margin of $w^\star$ is 1 and $\norm{w^\star}_p = \mar{p}^{-1}$.
    From Lemma~\ref{thm:key-iden}, we have
    \begin{align*} D_\psi(w^\star \log T, w_t) = D_\psi(w^\star \log T, w_{t+1}) + D_{\psi - \eta L}(w_{t+1}, w_{t}) &- \inp{\nabla L(w_t)}{w^\star \log T - w_t} \\
    &- \eta L(w_{t}) + \eta L(w_{t+1}). 
    \end{align*}
    We first bound the quantity $\inp{\nabla L(w_t)}{w^\star \log T - w_t}$ by expanding the definition of exponential loss:
    \begin{align*}
        &\inp{\nabla L(w_t)}{w^\star \log T - w_t} \\
        ={}& \sum_{i=1}^n \inp{\nabla \exp(-\inp{w_t}{x_i})}{w^\star \log T - w_t} \\
        ={}& \sum_{i=1}^n \inp{\exp(-\inp{w_t}{x_i})x_i}{w_t - w^\star \log T} \\
        ={}& \sum_{i=1}^n \exp(-\inp{w^\star \log T}{x_i}) \exp(-\inp{w_t - w^\star \log T}{x_i})  \inp{x_i}{w_t - w^\star \log T} \\
        \le{}& \sum_{i=1}^n \frac{1}{T} \cdot \frac{1}{e} = \frac{n}{eT}
    \end{align*}
    where the last line follows from the definition of $w^\star$ and the fact that for any $x \in \RR$, we have $e^{-x} x \le 1/e$.
    It follows that
    \[ D_\psi(w^\star \log T, w_t) \ge D_\psi(w^\star \log T, w_{t+1}) - \frac{n}{eT} - \eta L(w_{t}) + \eta L(w_{t+1}). \]
    
    Summing over $t = 0, \dots, T-1$ gives us
    \[ D_\psi(w^\star \log T, w_0) \ge D_\psi(w^\star \log T, w_T) - \frac{n}{e} - \eta L(w_0) + \eta L(w_T). \]
    Since Bregman divergence with respect to the $p$th power of $\ell_p$-norm is homogeneous, we can divide by a factor of $\log^p T$ on both sides:
    \begin{equation}
    \label{equ:lim-reference-scale}
        D_\psi\left(w^\star, \frac{w_0}{\log T}\right) \ge D_\psi\left(w^\star, \frac{w_T}{\log T}\right) - o(1).
    \end{equation}
    As $T\to\infty$, the left-hand side converges to $\brg{w^\star}{0}  = \psi(w^\star) = \frac{1}{p} \mar{p}^{-p}$.
    Let $\tilde{w} = w_T / \log T$, we expand the right-hand side as
    \begin{align*}
        \brg{w^\star}{\tilde{w}}
        &= \psi(w^\star) - \psi(\tilde{w}) - \inp{\nabla\psi(\tilde{w})}{w^\star - \tilde{w}} \\
        &= \frac{1}{p} \mar{p}^{-p} + \frac{p-1}{p} \norm{\tilde{w}}_p^p - \inp{\nabla\psi(\tilde{w})}{w^\star} \\
        &\ge \frac{1}{p} \mar{p}^{-p} + \frac{p-1}{p} \norm{\tilde{w}}_p^p - \mar{p}^{-1} \norm{\nabla\psi(\tilde{w})}_q
    \end{align*}
    for $1/p + 1/q = 1$.
    Recall that $\psi = \frac{1}{p}\norm{\cdot}_p^p$ has the following nice properties:
    \begin{align*}
     \nabla \psi (w) &= (\sgn(w_1)|w_1|^{p-1},\cdots, \sgn(w_d)|w_d|^{p-1})\\
     \inp{\nabla \psi (w)}{w} &= \sgn(w_1)w_1|w_1|^{p-1} + \cdots + \sgn(w_d)w_d|w_d|^{p-1} = \norm{w}_p^p
    \end{align*}
    So, we can further simplify $\norm{\nabla\psi(\tilde{w})}_q$:
    \begin{align*}
        \norm{\nabla\psi(\tilde{w})}_q
        &= \left(\sum_{i=1}^d |\tilde{w}_i|^{(p-1)q} \right)^{1/q} \\
        &= \left(\sum_{i=1}^d |\tilde{w}_i|^{p} \right)^{1/q} \\
        &= \norm{\tilde{w}}_p^{p/q} = \norm{\tilde{w}}_p^{p-1},
    \end{align*}
    where we note that because $1/p + 1/q = 1$, we also have $pq = p + q$ and $1 + p/q = p$.
    
    Now, we have
    \[\brg{w^\star}{\tilde{w}} \ge \frac{1}{p} \mar{p}^{-p} + \frac{p-1}{p} \norm{\tilde{w}}_p^p - \mar{p}^{-1} \norm{\tilde{w}}_p^{p-1}\]
    If $\norm{w_T / \log T}_p > \mar{p}^{-1} \cdot \frac{p}{p-1}$ for arbitrarily large $T$, then $\brg{w^\star}{w_T / \log T} > \frac{1}{p} \mar{p}^{-p}$ for those $T$.
    This in turn contradicts inequality \eqref{equ:lim-reference-scale}.
    Therefore, we must have 
    \[ \limsup_{T\to\infty} \norm{w_T}_p \le \mar{p}^{-1} \frac{p}{p-1} \log T. \]
    
    Now we can turn our attention to the lower bound.
    Let $m_t = \gamma(w_t)$ be the margin of the mirror descent iterates.
    Then, 
    \[ L(w_t) = \frac{1}{n} \sum_{i=1}^n \exp(-\inp{w_t}{x_i}) \ge \frac{1}{n} \exp(-m_t).\]
    Due to Lemma~\ref{thm:to-infinity}, we also know that $m_t \xrightarrow{t\to\infty} \infty$.
    
    By the definition of the max-margin direction, we know that $\gamma(\norm{w_t}_p \mmd{p}) \ge m_t$.
    Then by linearity of margin, there exists $w^\star$ so that $\gamma(w^\star) \ge (1+\frac{2n}{m_t}) m_t$ and $\norm{w^\star}_p \le (1+\frac{2n}{m_t}) \norm{w_t}_p$.
    It follows that
    \[ L(w^\star) = \frac{1}{n} \sum_{i=1}^n \exp(-\inp{w_t}{x_i}) \le \exp(-\gamma(w^\star)) = \frac{1}{2n} \exp(-m_t).\]
    
    Under the assumption that the step size $\eta$ is sufficiently small so that $\psi - \eta L$ is convex on the iterates, we can apply the convergence rate of mirror descent ~\citep[Theorem 3.1]{lu2018relatively}:
    \[ L(w_t) - L(w^\star) \le \frac{1}{\eta t} \brg{w^\star}{w_0} \]
    From our choice of $w^\star$, we have
    \begin{align*}
        \frac{1}{2n} \exp(-m_t)
        &\le \frac{1}{\eta t} \brg{w^\star}{w_0} \\
        &= \frac{1}{\eta t} (\psi(w^\star) - \psi(w_0) - \inp{\nabla\psi(w_0)}{w^\star - w_0})
    \end{align*}
    After dropping the lower order terms and recall the upper bounds on $\norm{w^\star}_p$ and $\norm{w_t}_p$, we have
    \[ \frac{1}{2n} \exp(-m_t) \le O(1) \cdot \frac{1}{\eta t} \cdot \frac{1}{p} \left(1 + \frac{\log (2n)}{m_t}\right)^p \left(\mar{p}^{-1} \frac{p}{p-1} \log t\right)^p\]
    Since $m_t$ is unbounded, the quantity $1 + \frac{\log (2n)}{m_t}$ is upper bounded by a constant.
    Taking the logarithm on both sides yields
    \[ m_t \ge \log t - p \log\log t + O(1)\]
    
    Finally, we use the definition of margin to conclude that $ m_t \le \inp{w_t}{x_i} \le C \cdot \norm{w_t}_p$.
    Therefore, 
    \[ \norm{w_t}_p \ge \frac{1}{C} (\log t - p \log\log t) + O(1).\]
\end{proof}

\clearpage

\section{Practicality of \algname}
\label{sec:practicality}
To illustrate that \algname can be easily implemented, we show a proof-of-concept implementation in PyTorch.
This implementation can directly replace existing optimizers and thus require only minor changes to any existing training code. 

We also note that while the \algname update step requires more arithmetic operations than a standard gradient descent update, this does not significantly impact the total runtime because differentiation is the most computationally intense step.
We observed from our experiments that training with \algname is approximate 10\% slower than with PyTorch's \texttt{optim.SGD} (in the same number of epochs),\footnote{This measurement may not be very accurate because we were using shared computing resources.} and we believe that this gap can be closed with a more optimized code.

\begin{lstlisting}[caption={Sample PyTorch implementation of \algname}]
import torch
from torch.optim import Optimizer

class pnormSGD(Optimizer):
    def __init__(self, params, lr=0.01, pnorm=2.0):
        if not 0.0 <= lr:
            raise ValueError("Invalid learning rate: {}".format(lr))
        # p-norm must be strictly greater than 1
        if not 1.01 <= pnorm:
            raise ValueError("Invalid p-norm value: {}".format(pnorm))

        defaults = dict(lr=lr, pnorm=pnorm)
        super(pnormSGD, self).__init__(params, defaults)

    def __setstate__(self, state):
        super(pnormSGD, self).__setstate__(state)

    def step(self, closure=None):
        loss = None
        if closure is not None:
            with torch.enable_grad():
                loss = closure()

        for group in self.param_groups:
            lr = group["lr"]
            pnorm = group["pnorm"]

            for param in group["params"]:
                if param.grad is None:
                    continue

                x = param.data
                dx = param.grad.data

                # \ell_p^p potential function
                update = torch.pow(torch.abs(x), pnorm-1) * \
                                torch.sign(x) -  lr * dx           
                param.data =  torch.sign(update) * \
                            torch.pow(torch.abs(update), 1/(pnorm-1))

        return loss
\end{lstlisting}

\clearpage

\section{Experimental details}
\label{sec:experiment-detail}
\subsection{Linear classification}
Here, we describe the details behind our experiments from Section~\ref{sec:linear-classifier}.
First, we note that we can absorb the labels $y_i$ by replacing $(x_i, y_i)$ with $(y_ix_i, 1)$.
This way, we can choose points with the same $+1$ label.

For the $\RR^2$ experiment, we first select three points $(\frac{1}{6}, \frac{1}{2}), (\frac{1}{2}, \frac{1}{6})$ and $(\frac{1}{3}, \frac{1}{3})$ so that the maximum margin direction is approximately $\frac{1}{\sqrt{2}}(1, 1)$.
Then we sample 12 additional points from $\mathcal{N}((\frac{1}{2}, \frac{1}{2}), 0.15 I_2)$.
The initial weight $w_0$ is selected from $\mathcal{N}(0, I_2)$.
We ran \algname with step size $10^{-4}$ for 1 million steps.
As for the scatter plot of the data, we randomly re-assign a label and plot out $(x_i, 1)$ or $(-x_i, -1)$ uniformly at random.

For the $\RR^{100}$ experiment, we select 15 sparse vectors that each has up to 10 nonzero entries.
Each nonzero entry is independently sampled from $\mathcal{U}(-2, 4)$.
Because we are in the over-parameterized case, these vectors are linearly separable with high probability.
The initial weight $w_0$ is selected from $\mathcal{N}(0, 0.1 I_{100})$.
We ran \algname with step size $10^{-4}$ for 1 million steps.

These experiments were performed on an Intel Skylake CPU.

\subsection{CIFAR-10 experiments}
For the experiments with the CIFAR-10 dataset, we adopted the example implementation from the \texttt{FFCV} library.\footnote{\url{https://github.com/libffcv/ffcv/tree/main/examples/cifar}}
For consistency, we ran \algname with the same hyper-parameters for all neural networks and values of $p$.
We used a cyclic learning rate schedule with maximum learning rate of 0.1 and ran for 400 epochs so the training loss is approximately 0.\footnote{This differs from the setup from~\cite{azizan2021stochastic}, where they used a fixed small learning rate and much larger number of epochs.}

This experiment was performed on a single Nvidia V100 GPU.

\subsection{ImageNet experiments}
For the experiments with the ImageNet dataset, we used the example implementation from the \texttt{FFCV} library.\footnote{\url{https://github.com/libffcv/ffcv-imagenet/}}
For consistency, we ran \algname with the same hyper-parameters for all neural networks and values of $p$.
We used a cyclic learning rate schedule with maximum learning rate of 0.5 and ran for 120 epochs.
Note that, to more accurately measure the effect of \algname on generalization, we turned off any parameters that may affect regularization, e.g. with momentum set to 0, weight decay set to 0, and label smoothing set to 0, etc.

This experiment was performed on a single Nvidia V100 GPU.

\section{Additional experimental results}
\label{sec:add-experiments}
\subsection{Linear classification}
\label{sec:add-experiment-synthetic}
We present a more complete result for the setting of  Section \ref{sec:linear-classifier} with more values of $p$.
Note that Table~\ref{tab:linear-bias} is a subset of Table~\ref{tab:linear-bias-full-1}, as shown below.

Except for $p = 1.1$, \algname produces the smallest linear classifier under the corresponding $\ell_p$-norm and thus consistent with the prediction of Theorem~\ref{thm:primal-bias}.
When $p = 1.1$, Corollary~\ref{thm:final-convg-rate} predicts a much slower convergence rate.
So, for the number of iterations we have, \algname with $p = 1.1$ in fact cannot compete against \algname with $p = 1.5$, which has much faster convergence rate but similar implicit bias.
The second trial shows a rare case where \algname with $p = 1.1$ could not even match \algname with $p = 2$ under the $\ell_{1.1}$-norm.
Therefore, before we come up with techniques to speed up the convergence of \algname, it is not advisable to pick $p$ that is too close to 1.

\begin{table}
    \centering
    \setlength{\tabcolsep}{4.5pt}
    \begin{tabular}{l| c|c|c|c|c|c|c|c}
    \hline
    & $\ell_1$ norm & $\ell_{1.1}$ norm & $\ell_{1.5}$ norm & $\ell_{2}$ norm & $\ell_{3}$ norm & $\ell_{6}$ norm & $\ell_{10}$ norm & $\ell_{\infty}$ norm \\
    \hline\hline
    $p=1.1$ & \textbf{7.692} & 5.670 & 2.650 & 1.659 & 1.100 & 0.782 & 0.698 & 0.634 \\
    $p=1.5$ & 7.924 & \textbf{5.607} & \textbf{2.333} & 1.346 & 0.830 & 0.573 & 0.526 & 0.515 \\
    $p=2$ & 9.417 & 6.447 & 2.413 & \textbf{1.273} & 0.710 & 0.444 & 0.393 & 0.374 \\
    $p=3$ & 11.307 & 7.618 & 2.696 & 1.345 & \textbf{0.691} & 0.381 & 0.318 & 0.285 \\
    $p=6$ & 13.115 & 8.787 & 3.044 & 1.481 & 0.729 & \textbf{0.369} & 0.288 & 0.233 \\
    $p=10$ & 13.572 & 9.086 & 3.137 & 1.520 & 0.742 & 0.367 & \textbf{0.281} & \textbf{0.213} \\
    \hline
    \end{tabular}
    \caption{Size of the linear classifiers generated by \algname (after rescaling) in $\ell_1, \ell_{1.1}, \ell_{1.5}, \ell_2, \ell_3, \ell_6$ and $\ell_{10}$ norms.
    For each norm, we highlight the value of $p$ for which \algname generates the smallest classifier under that norm. (Trial 1)}
    \label{tab:linear-bias-full-1}
\end{table}

\begin{table}
    \centering
    \setlength{\tabcolsep}{4.5pt}
    \begin{tabular}{l| c|c|c|c|c|c|c|c}
    \hline
    & $\ell_1$ norm & $\ell_{1.1}$ norm & $\ell_{1.5}$ norm & $\ell_{2}$ norm & $\ell_{3}$ norm & $\ell_{6}$ norm & $\ell_{10}$ norm & $\ell_{\infty}$ norm \\
    \hline\hline
    $p=1.1$ & 10.688 & 8.013 & 3.883 & 2.465 & 1.644 & 1.187 & 1.082 & 1.009 \\
    $p=1.5$ & \textbf{9.308} & \textbf{6.546} & \textbf{2.674} & 1.518 & 0.913 & 0.602 & 0.535 & 0.488 \\
    $p=2$ & 10.735 & 7.340 & 2.735 & \textbf{1.435} & 0.790 & 0.479 & 0.418 & 0.397 \\
    $p=3$ & 12.298 & 8.327 & 2.991 & 1.508 & \textbf{0.782} & 0.432 & 0.359 & 0.324 \\
    $p=6$ & 13.817 & 9.322 & 3.297 & 1.631 & 0.816 & \textbf{0.418} & 0.328 & 0.265 \\
    $p=10$ & 14.545 & 9.798 & 3.447 & 1.695 & 0.841 & 0.423 & \textbf{0.325} & \textbf{0.247} \\
    \hline
    \end{tabular}
    \caption{Size of the linear classifiers generated by \algname (after rescaling) in $\ell_1, \ell_{1.1}, \ell_{1.5}, \ell_2, \ell_3, \ell_6$ and $\ell_{10}$ norms.
    For each norm, we highlight the value of $p$ for which \algname generates the smallest classifier under that norm. (Trial 2)}
    \label{tab:linear-bias-full-2}
\end{table}

\clearpage

\subsection{CIFAR-10 experiments: implicit bias}
We present more complete illustrations of the implicit bias trends of trained models in CIFAR-10.
Compared to Figure~\ref{fig:cifar10-hist}, the plots below include data from additional values for additional values of $p$ and more deep neural network architectures.

We see that the trends we observed in Section~\ref{sec:cifar} continue to hold under architectures other than \textsc{ResNet}.
In particular, for smaller $p$'s, the weight distributions of models trained with \algname have higher peak around zero, and higher $p$'s result in smaller maximum weights. 

\label{sec:add-experiment-cifar-bias}
\begin{figure}[!h]
    \centering
    \begin{subfigure}[b]{0.45\textwidth}
        \centering
        \includegraphics[width=\textwidth]{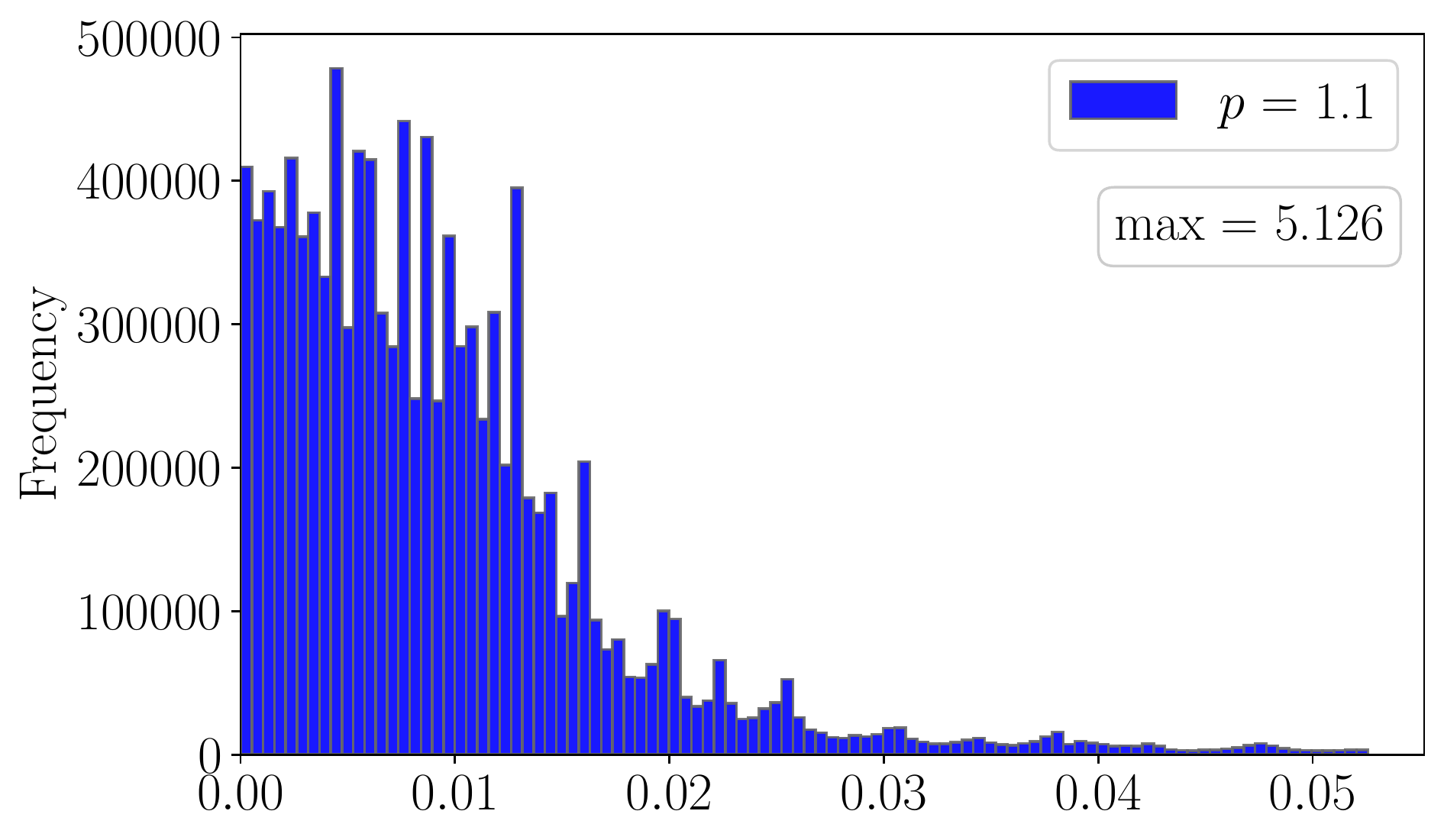}
    \end{subfigure}
    ~
    \begin{subfigure}[b]{0.45\textwidth}
        \includegraphics[width=\textwidth]{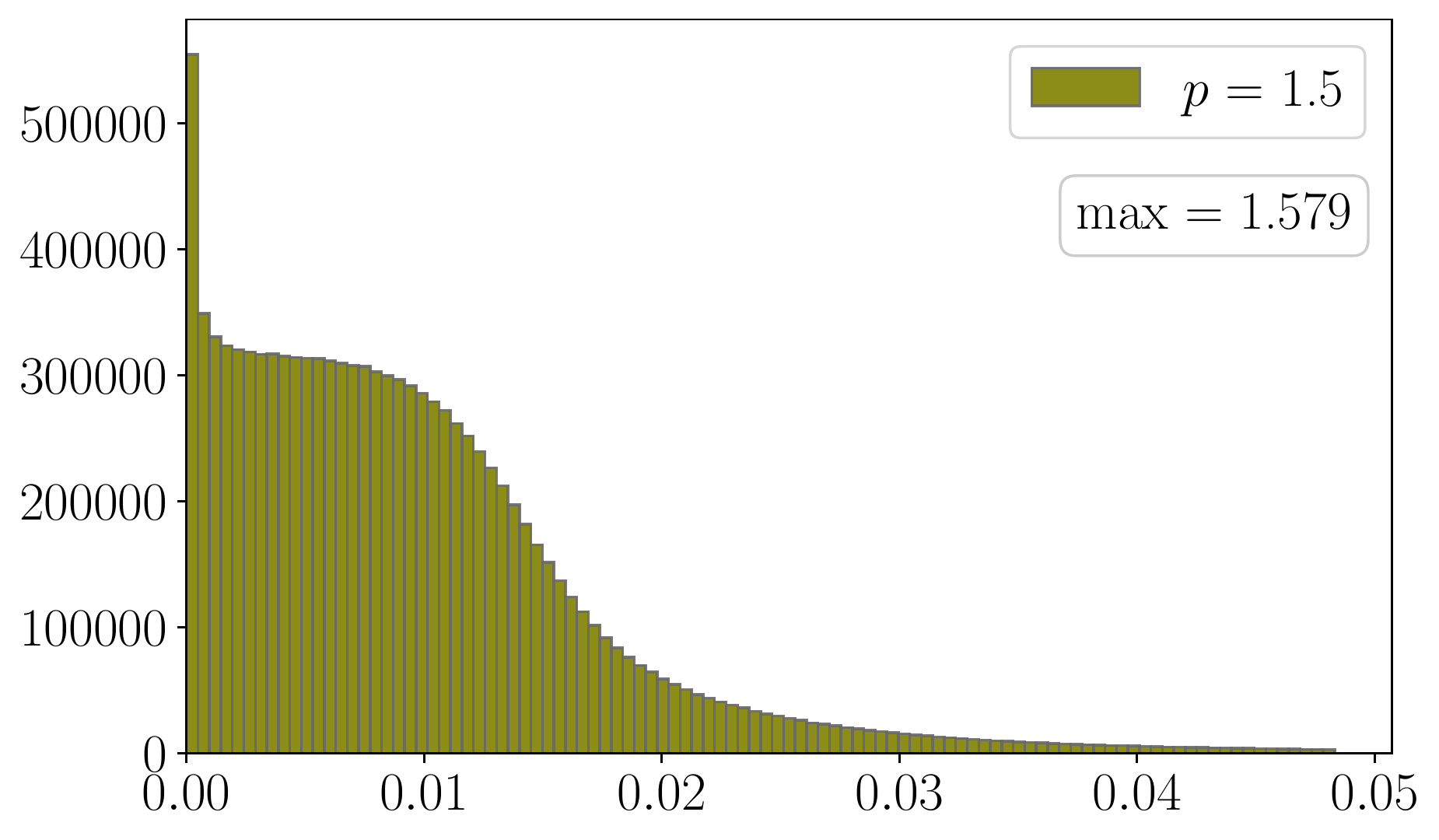}
    \end{subfigure}
    \begin{subfigure}[b]{0.45\textwidth}
        \includegraphics[width=\textwidth]{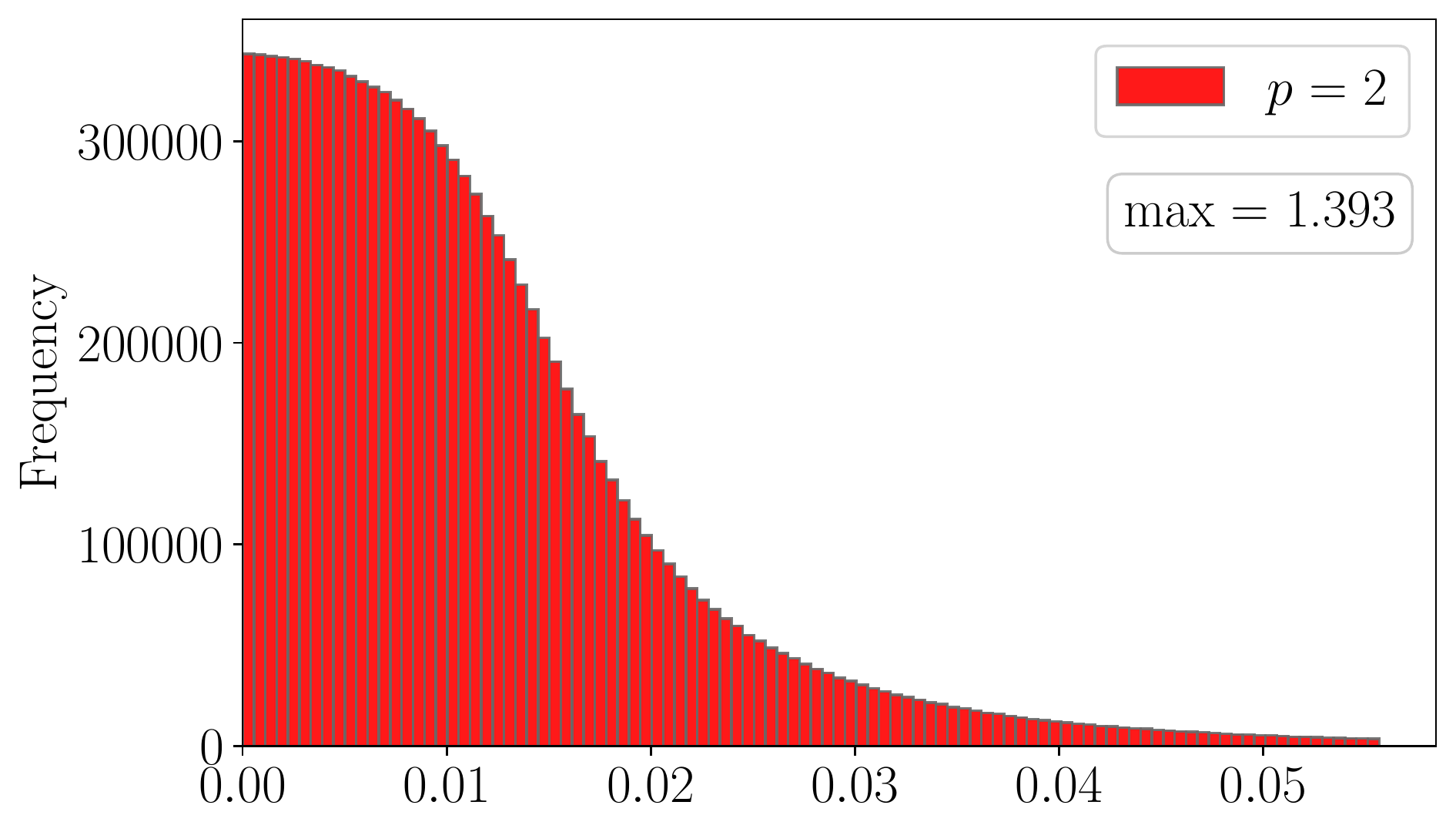}
    \end{subfigure}
    ~
    \begin{subfigure}[b]{0.45\textwidth}
        \includegraphics[width=\textwidth]{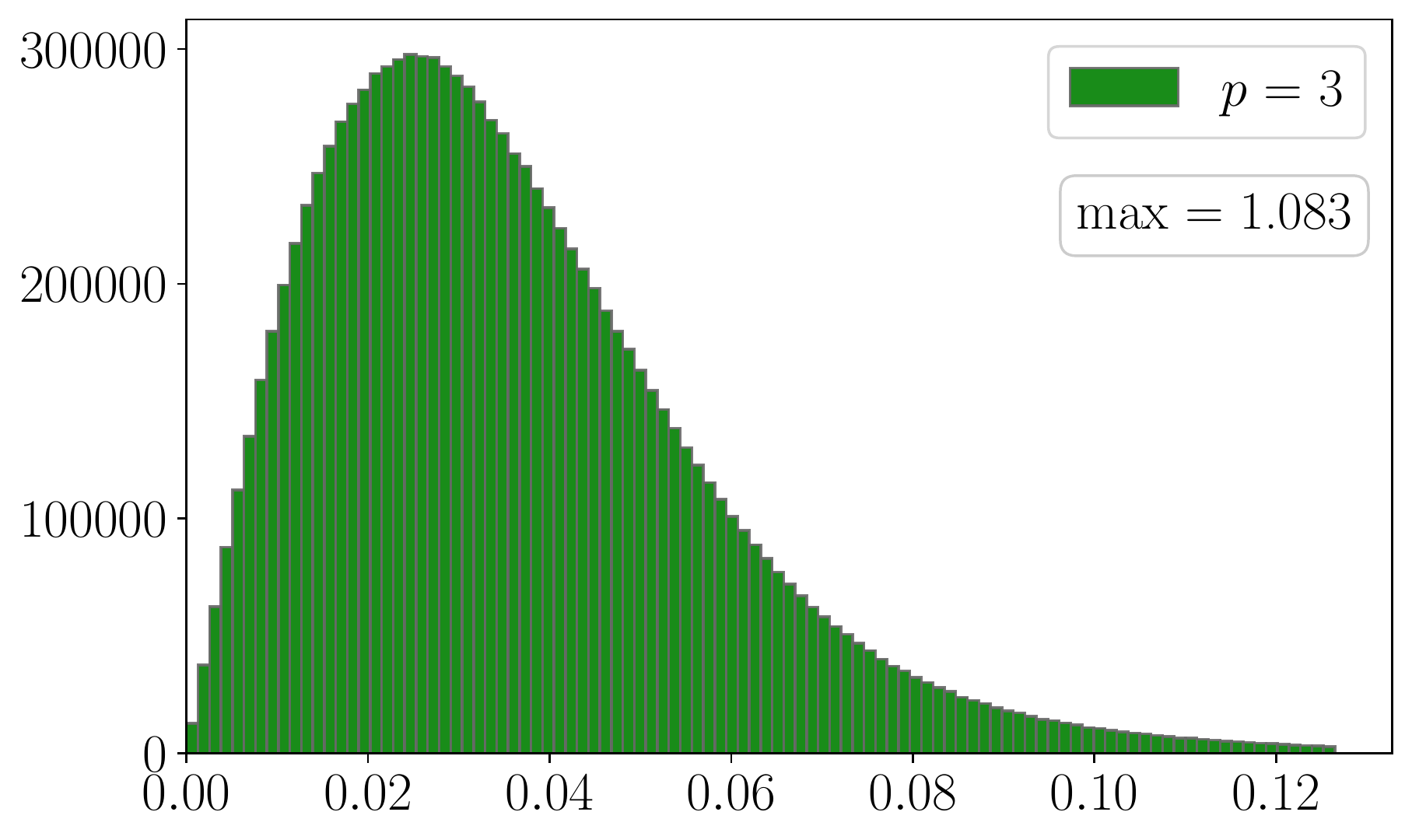}
    \end{subfigure}
    \begin{subfigure}[b]{0.45\textwidth}
        \includegraphics[width=\textwidth]{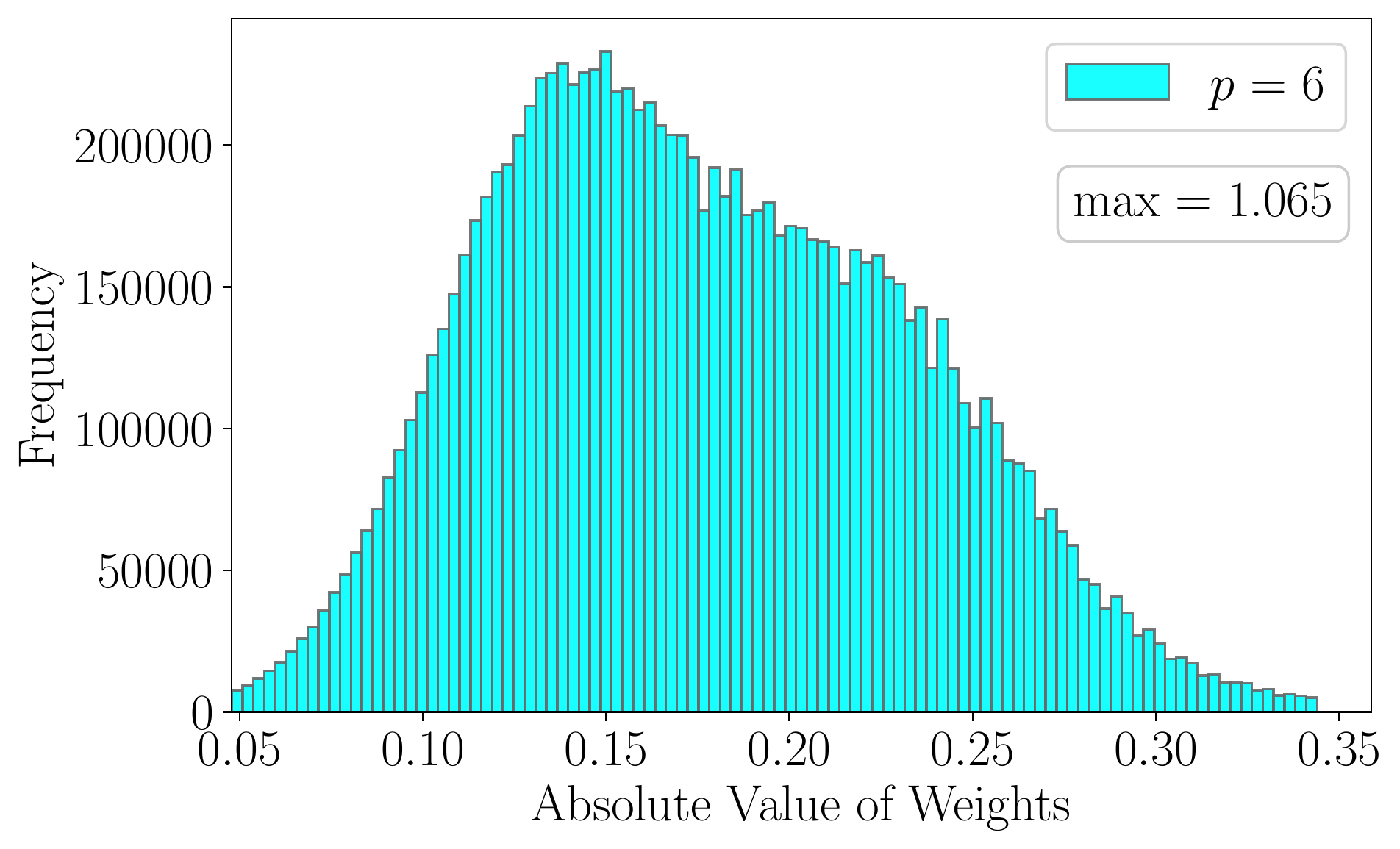}
    \end{subfigure}
    ~
    \begin{subfigure}[b]{0.45\textwidth}
        \includegraphics[width=\textwidth]{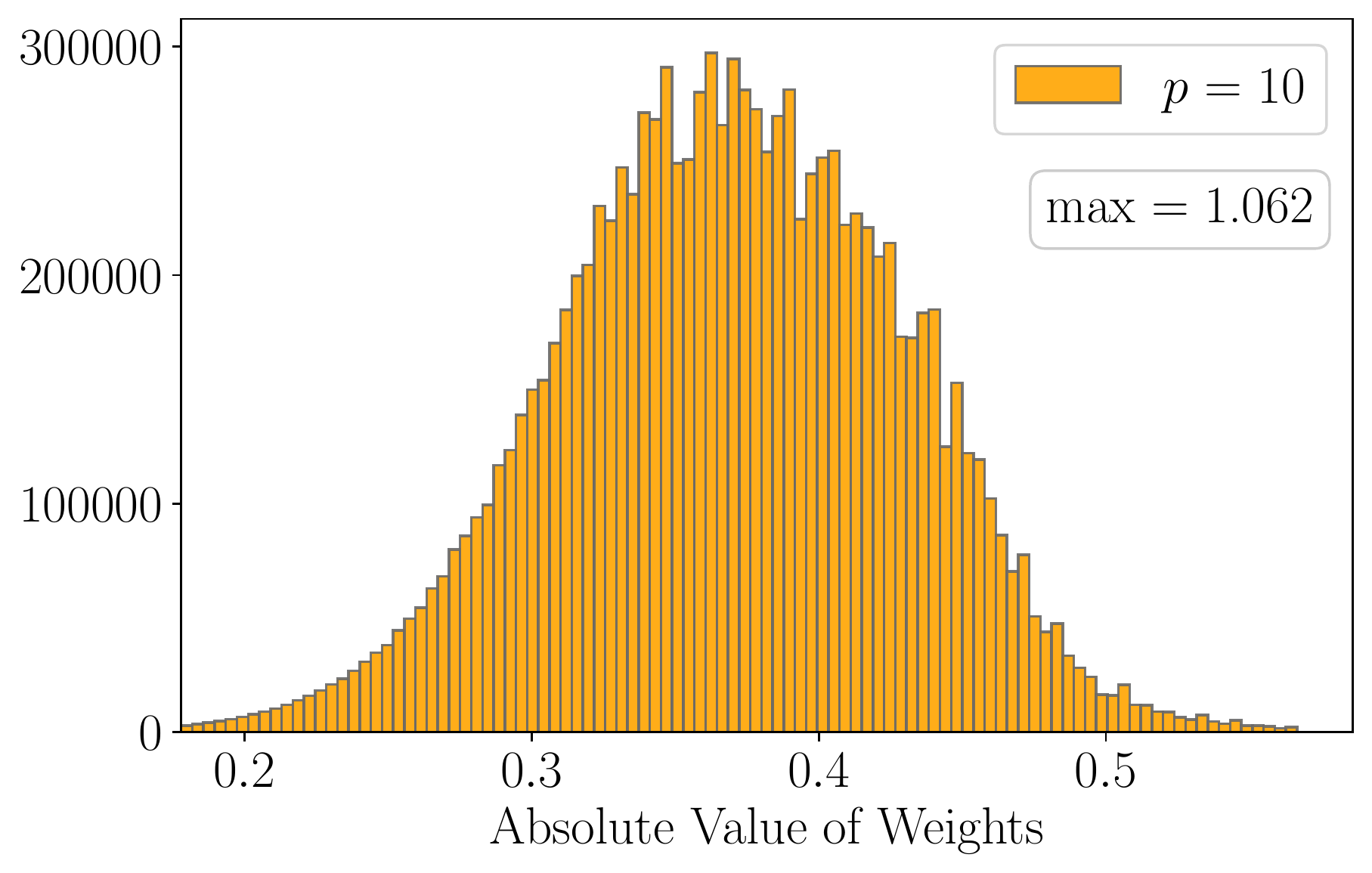}
    \end{subfigure}
    \caption{The histogram of weights in \textsc{ResNet-18} models trained with \algname for the CIFAR-10 dataset. 
    For clarity, we cropped out the tails and each plot has 100 bins after cropping.
    Note that the scale on the $y$-axis differs per graph.
    }
    \label{fig:cifar10-hist-resnet-full}
\end{figure}

\begin{figure}
    \centering
    \begin{subfigure}[b]{0.45\textwidth}
        \centering
        \includegraphics[width=\textwidth]{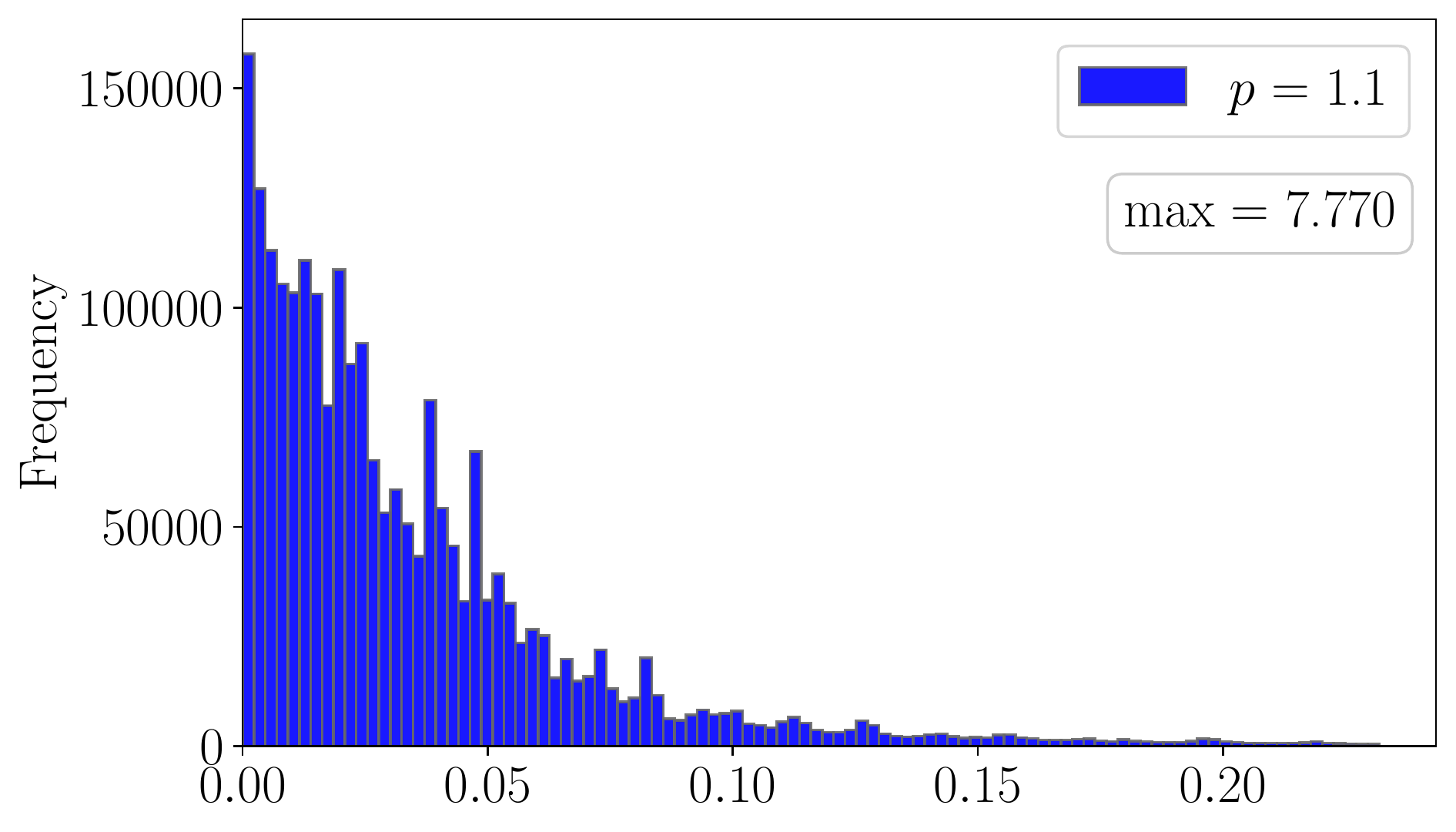}
    \end{subfigure}
    ~
    \begin{subfigure}[b]{0.45\textwidth}
        \includegraphics[width=\textwidth]{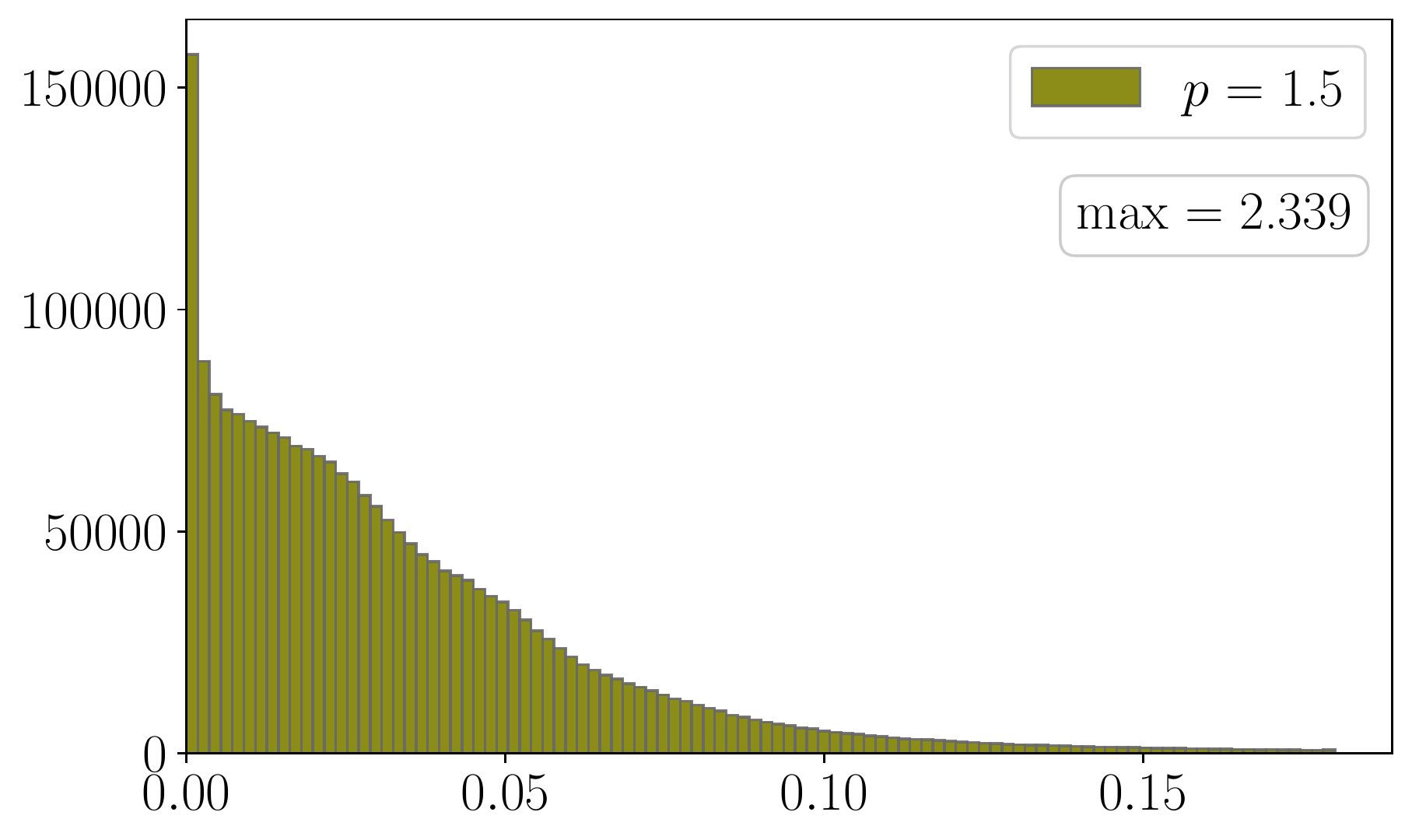}
    \end{subfigure}
    \begin{subfigure}[b]{0.45\textwidth}
        \includegraphics[width=\textwidth]{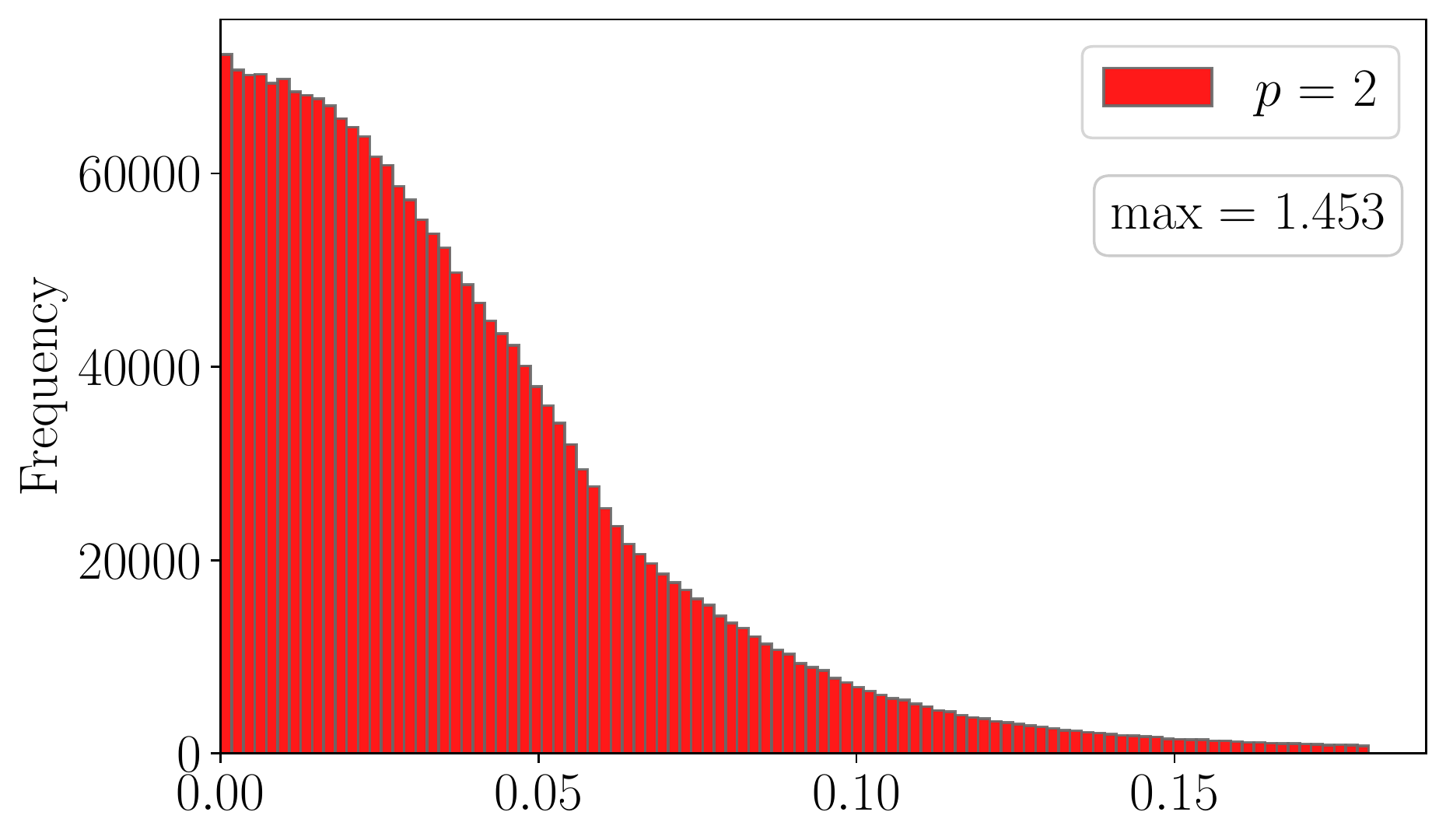}
    \end{subfigure}
    ~
    \begin{subfigure}[b]{0.45\textwidth}
        \includegraphics[width=\textwidth]{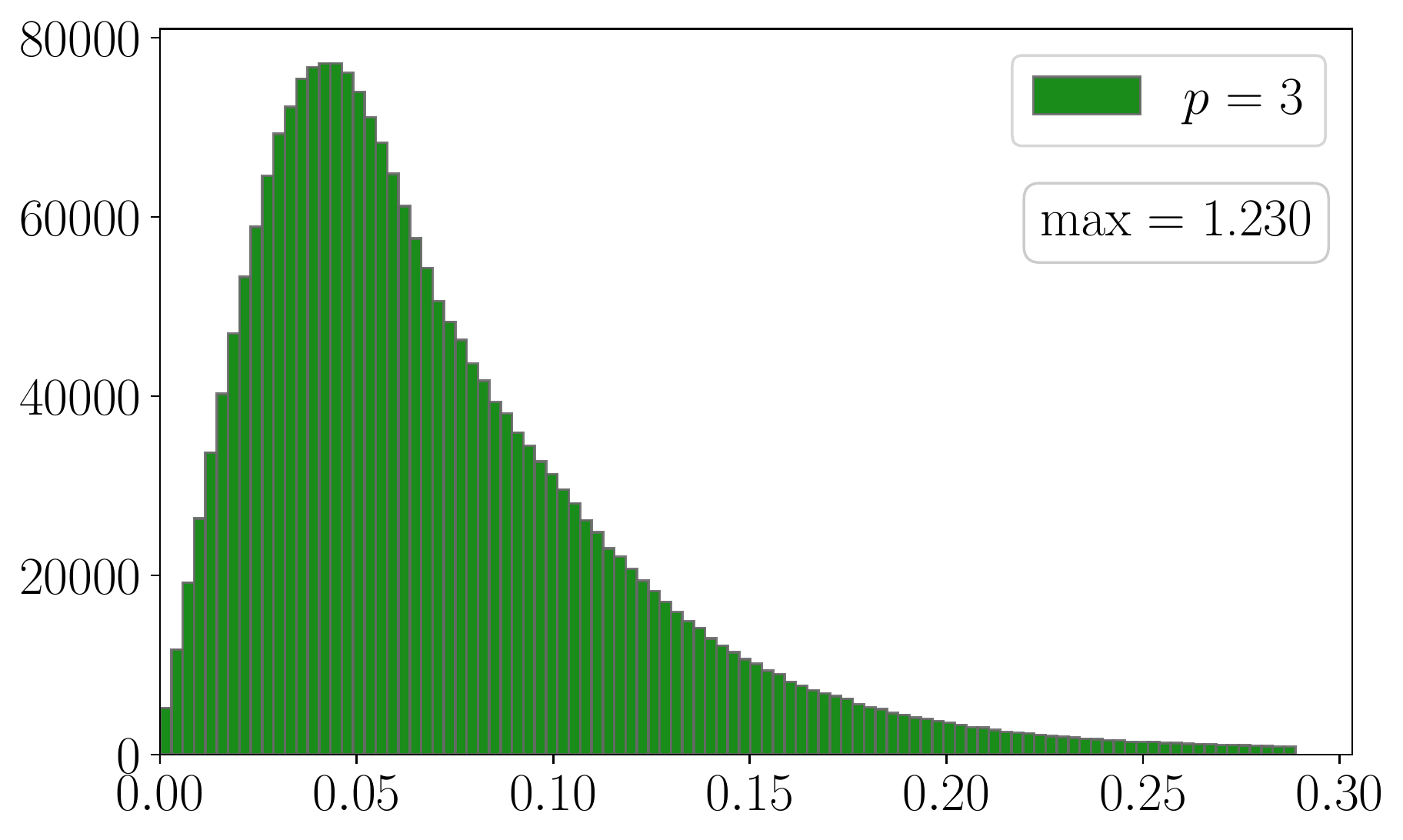}
    \end{subfigure}
    \begin{subfigure}[b]{0.45\textwidth}
        \includegraphics[width=\textwidth]{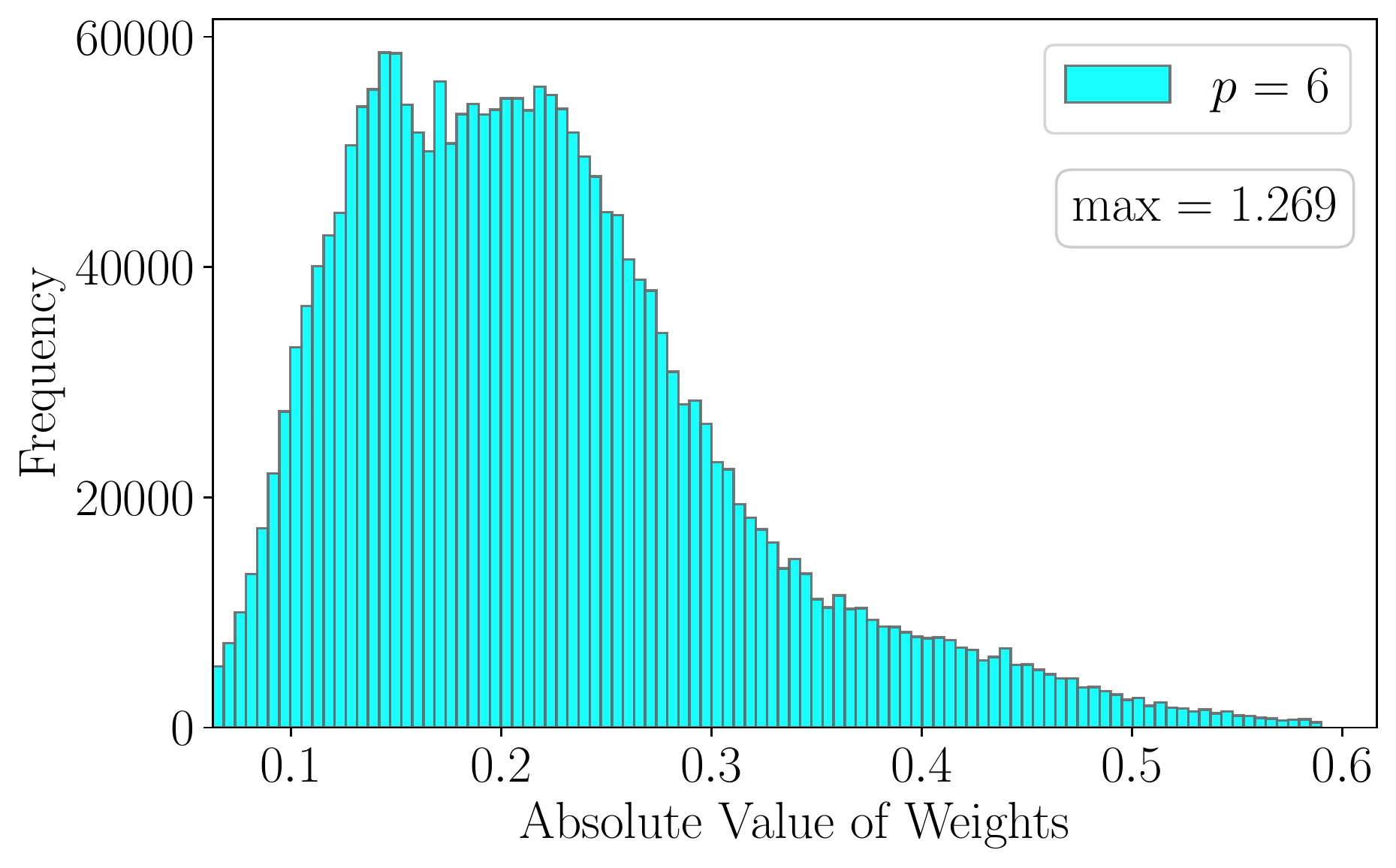}
    \end{subfigure}
    ~
    \begin{subfigure}[b]{0.45\textwidth}
        \includegraphics[width=\textwidth]{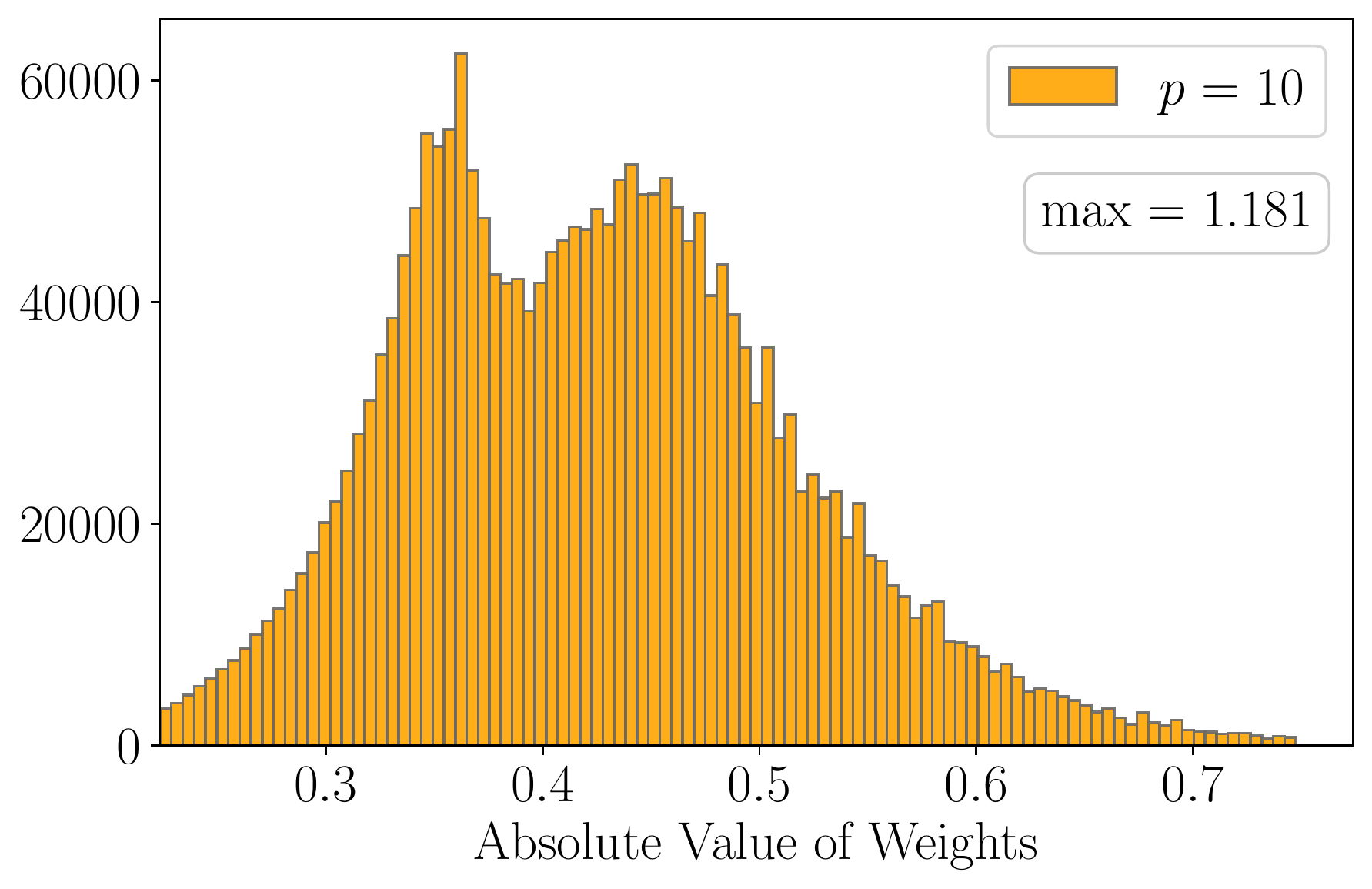}
    \end{subfigure}
    \caption{The histogram of weights in \textsc{MobileNet-v2} models trained with \algname for the CIFAR-10 dataset. 
    For clarity, we cropped out the tails and each plot has 100 bins after cropping.
    }
    \label{fig:cifar10-hist-mobilenet-full}
\end{figure}

\begin{figure}
    \centering
    \begin{subfigure}[b]{0.45\textwidth}
        \centering
        \includegraphics[width=\textwidth]{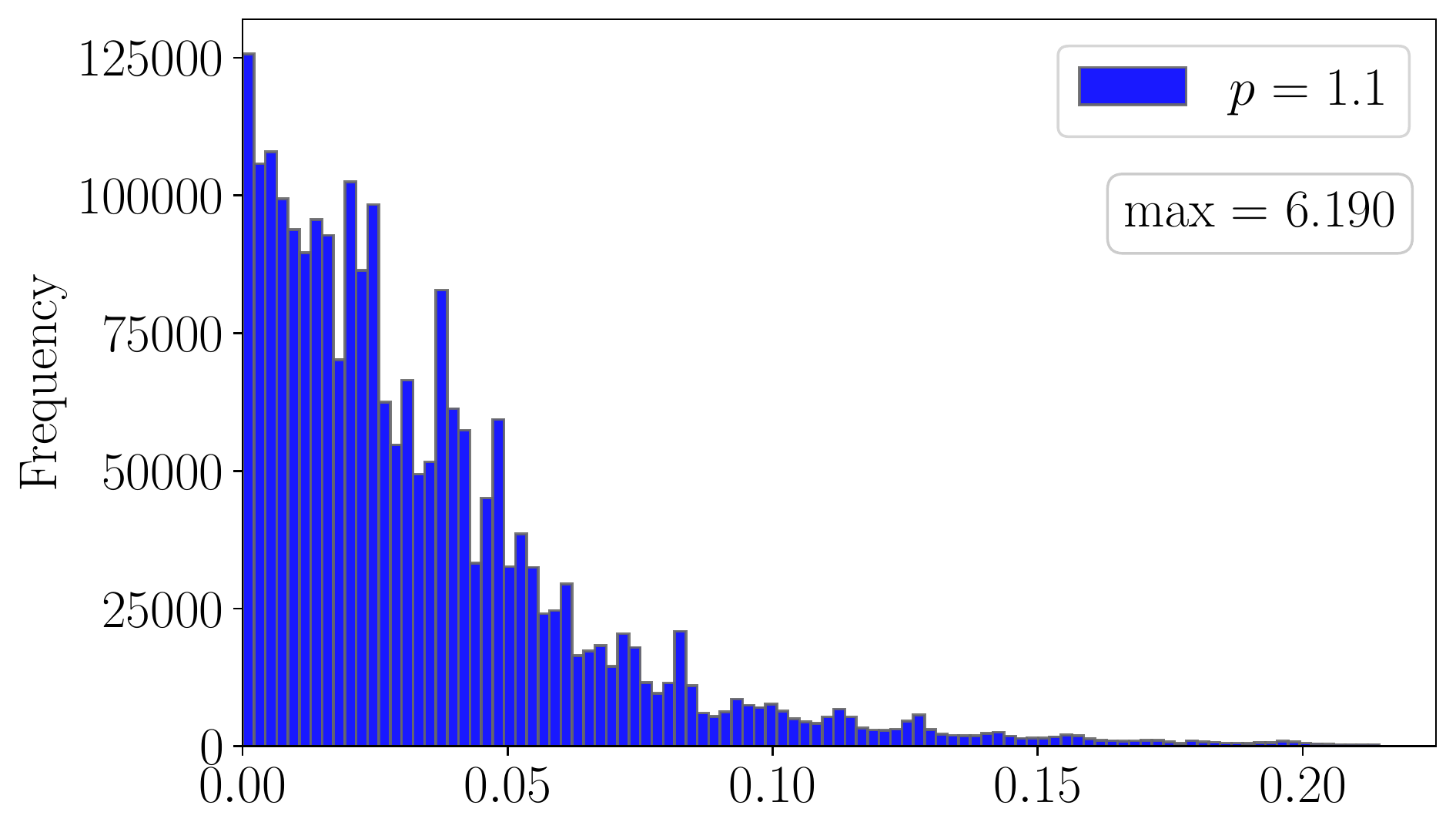}
    \end{subfigure}
    ~
    \begin{subfigure}[b]{0.45\textwidth}
        \includegraphics[width=\textwidth]{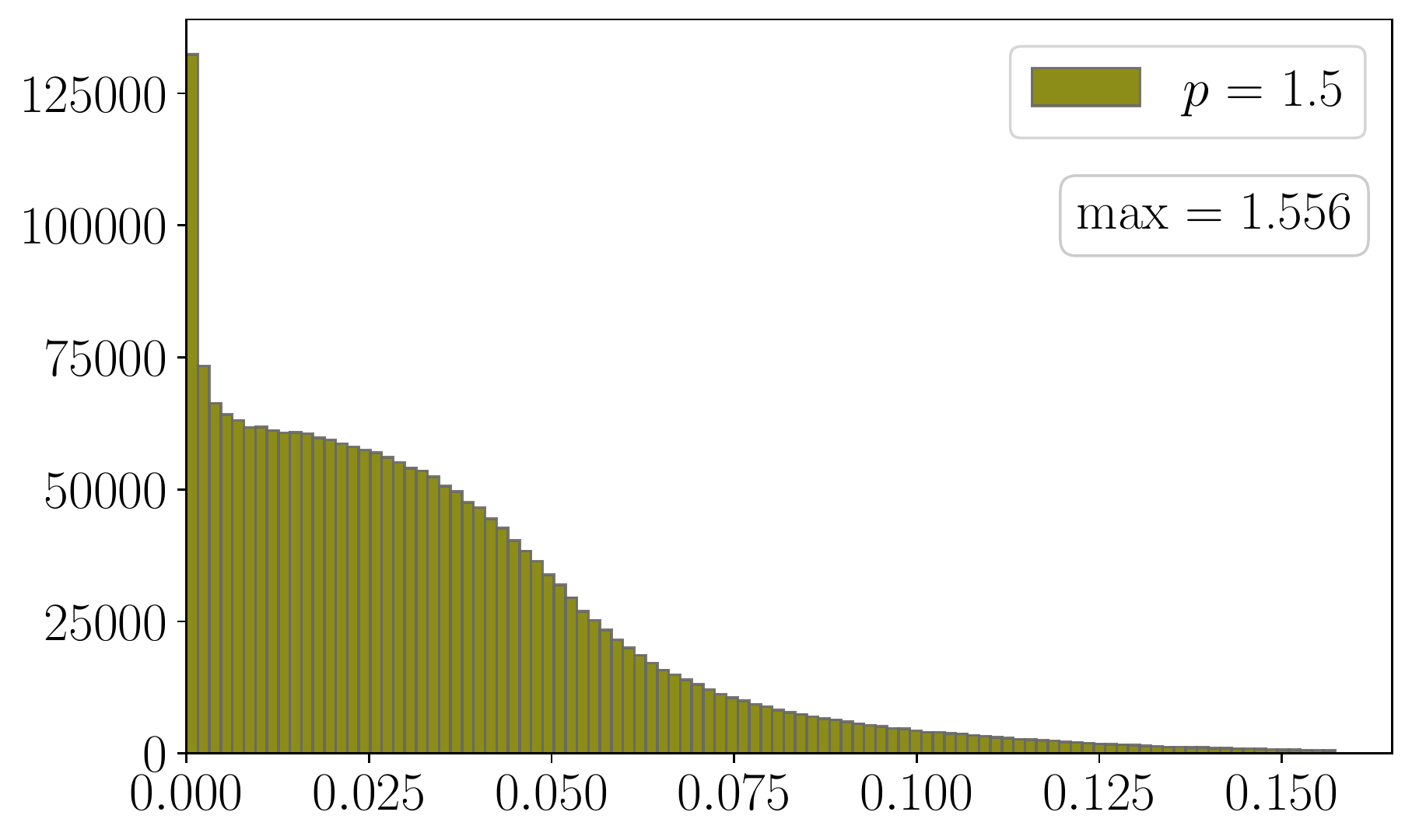}
    \end{subfigure}
    \begin{subfigure}[b]{0.45\textwidth}
        \includegraphics[width=\textwidth]{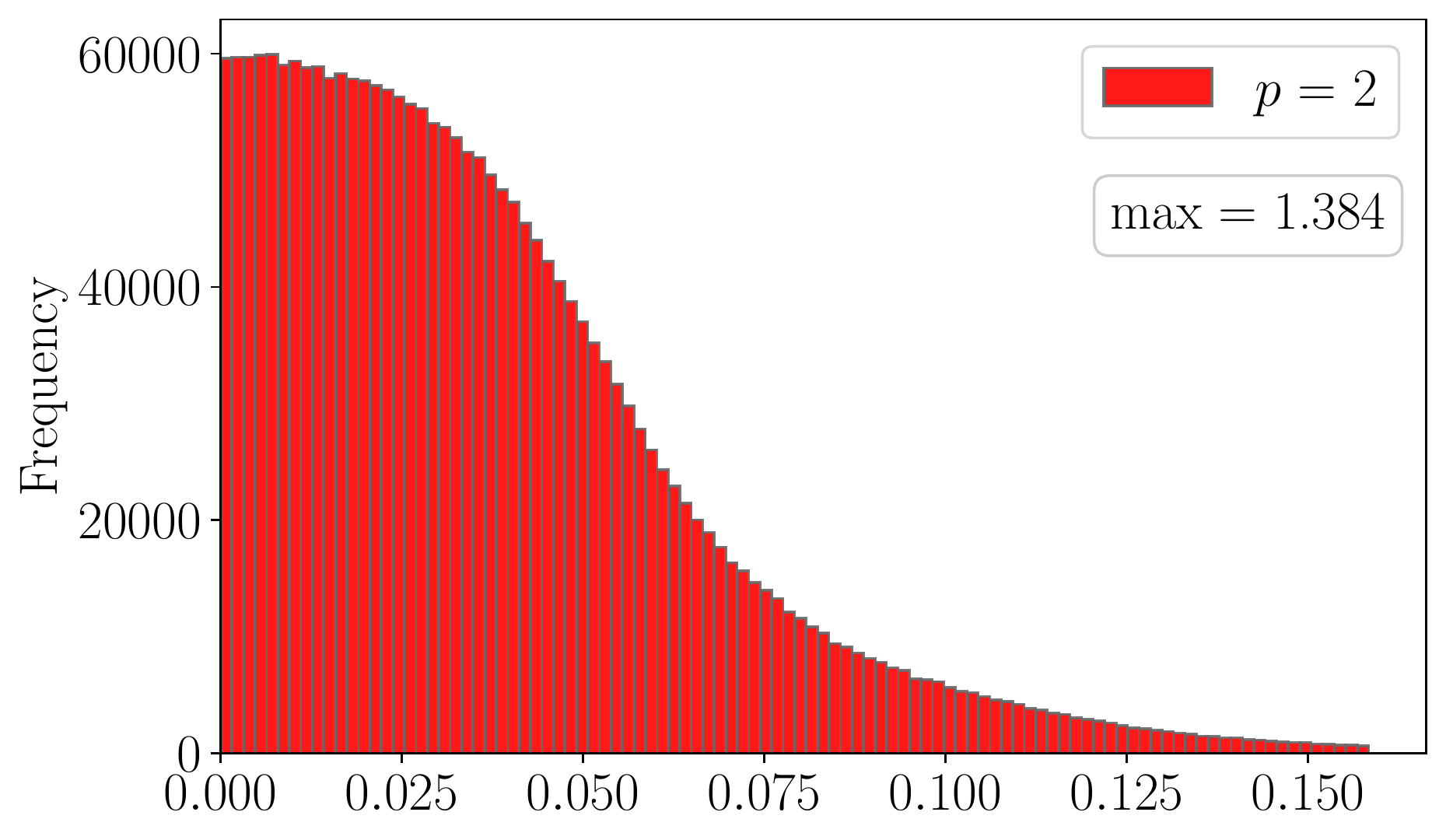}
    \end{subfigure}
    ~
    \begin{subfigure}[b]{0.45\textwidth}
        \includegraphics[width=\textwidth]{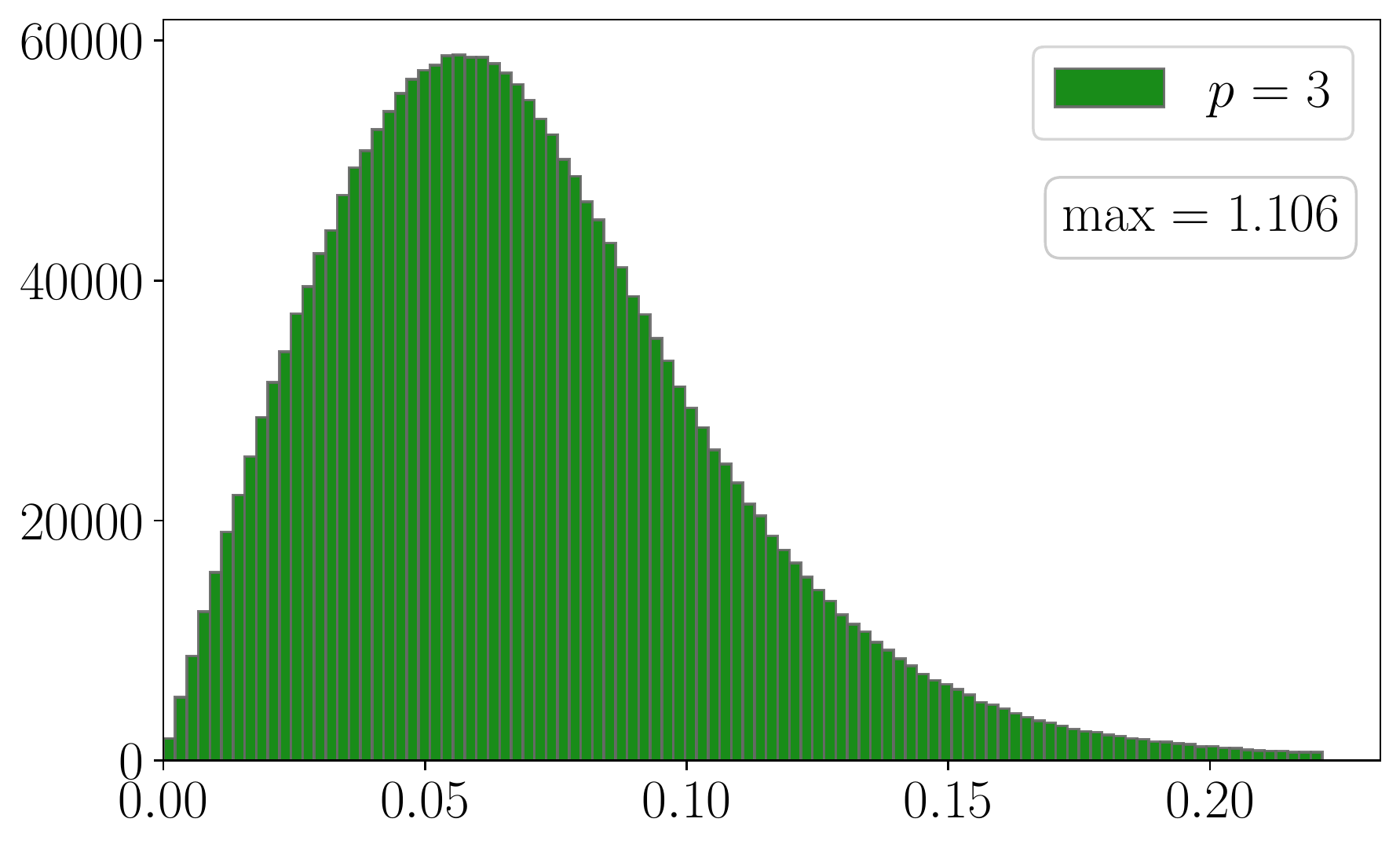}
    \end{subfigure}
    \begin{subfigure}[b]{0.45\textwidth}
        \includegraphics[width=\textwidth]{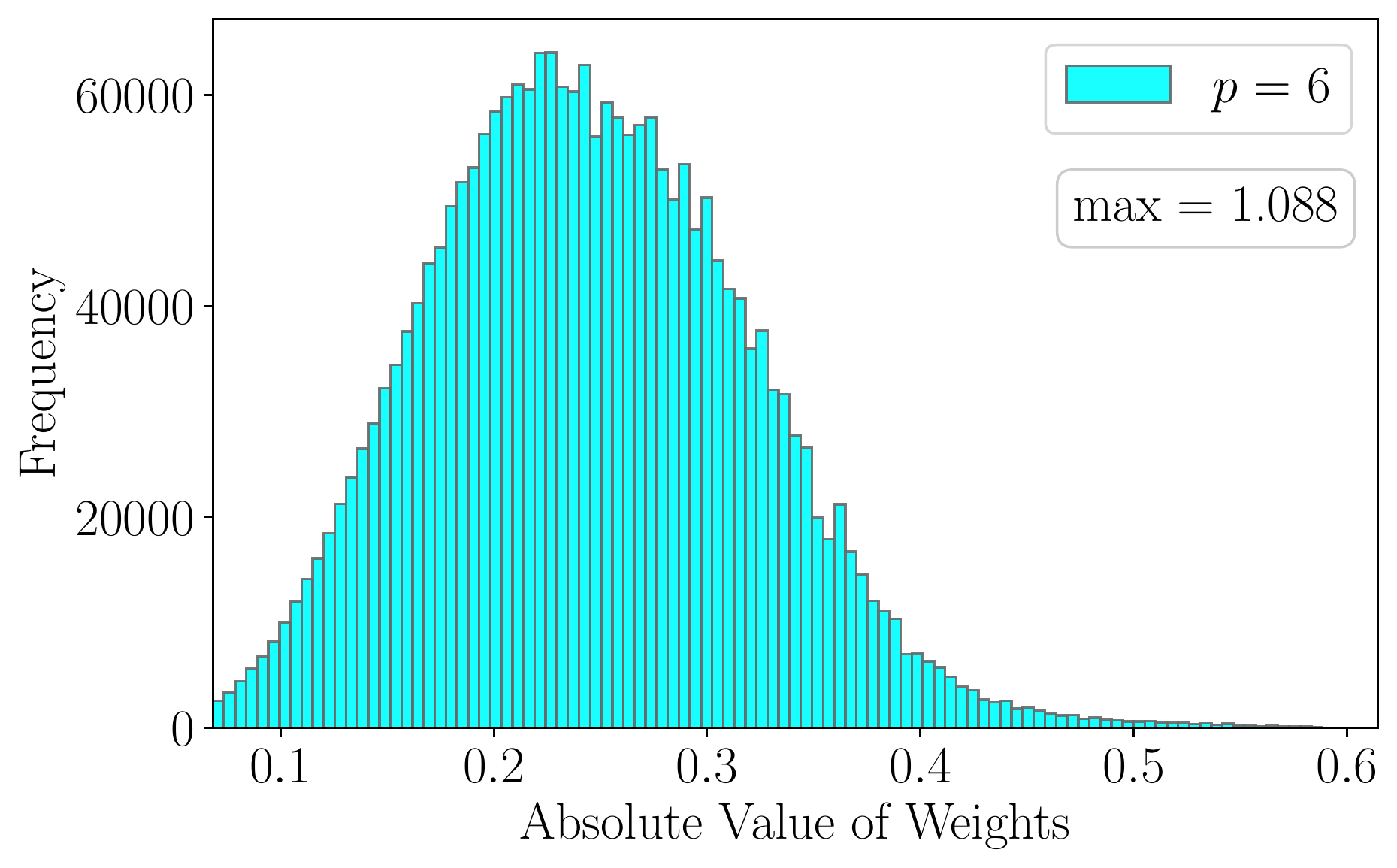}
    \end{subfigure}
    ~
    \begin{subfigure}[b]{0.45\textwidth}
        \includegraphics[width=\textwidth]{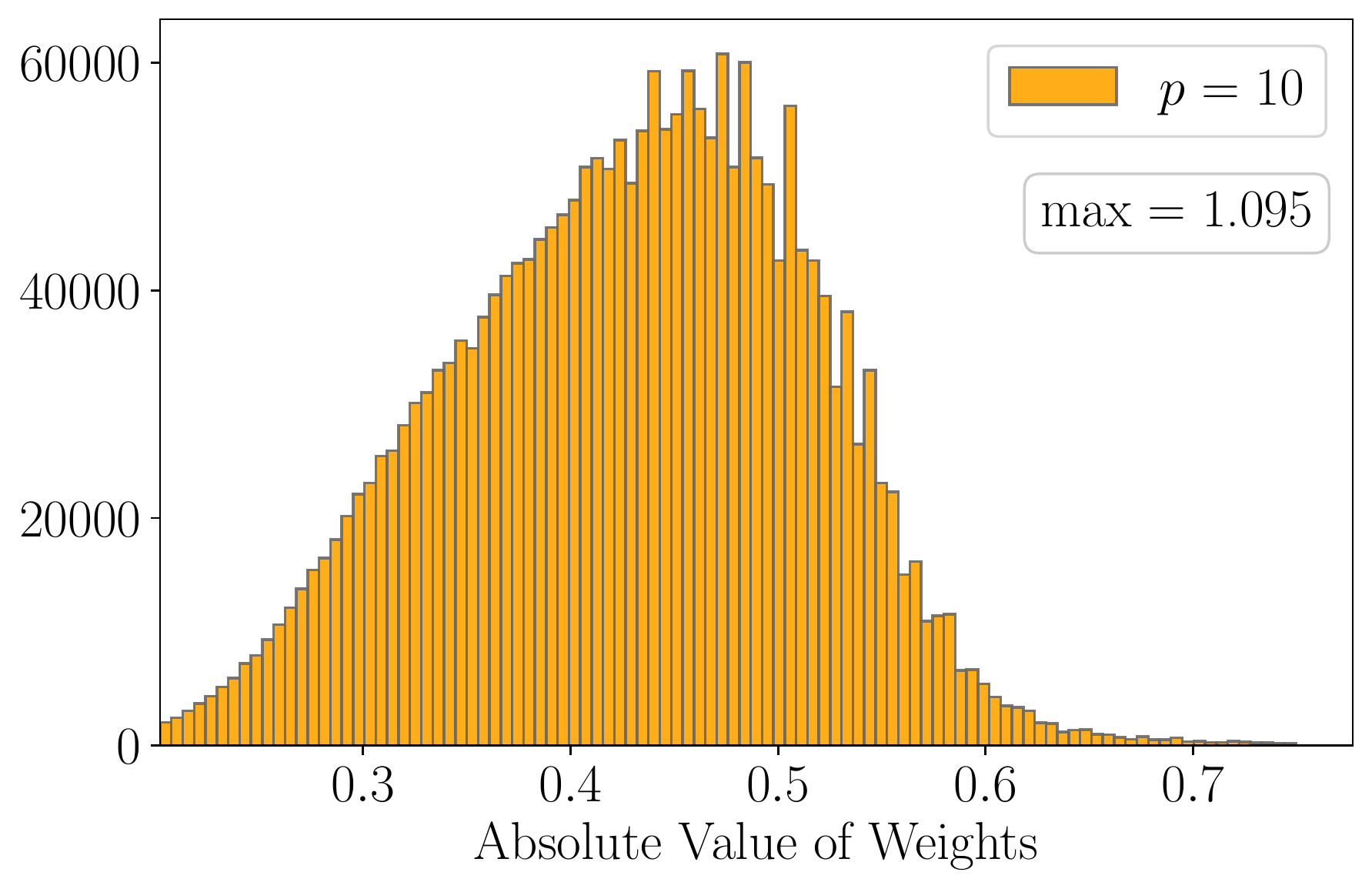}
    \end{subfigure}
    \caption{The histogram of weights in \textsc{RegNetX-200mf} models trained with \algname for the CIFAR-10 dataset. 
    For clarity, we cropped out the tails and each plot has 100 bins after cropping.
    }
    \label{fig:cifar10-hist-regnet-full}
\end{figure}

\begin{figure}
    \centering
    \begin{subfigure}[b]{0.45\textwidth}
        \centering
        \includegraphics[width=\textwidth]{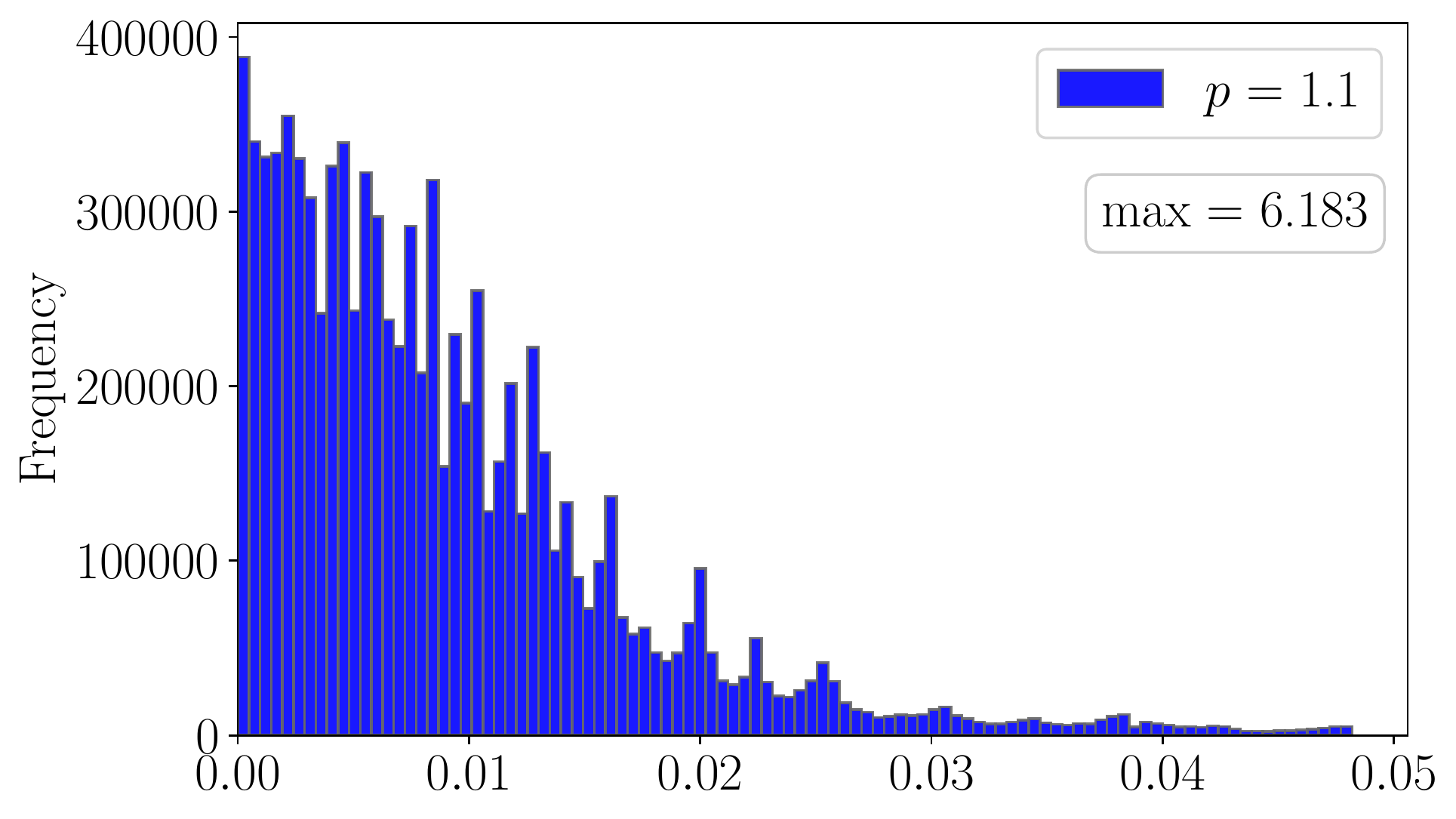}
    \end{subfigure}
    ~
    \begin{subfigure}[b]{0.45\textwidth}
        \includegraphics[width=\textwidth]{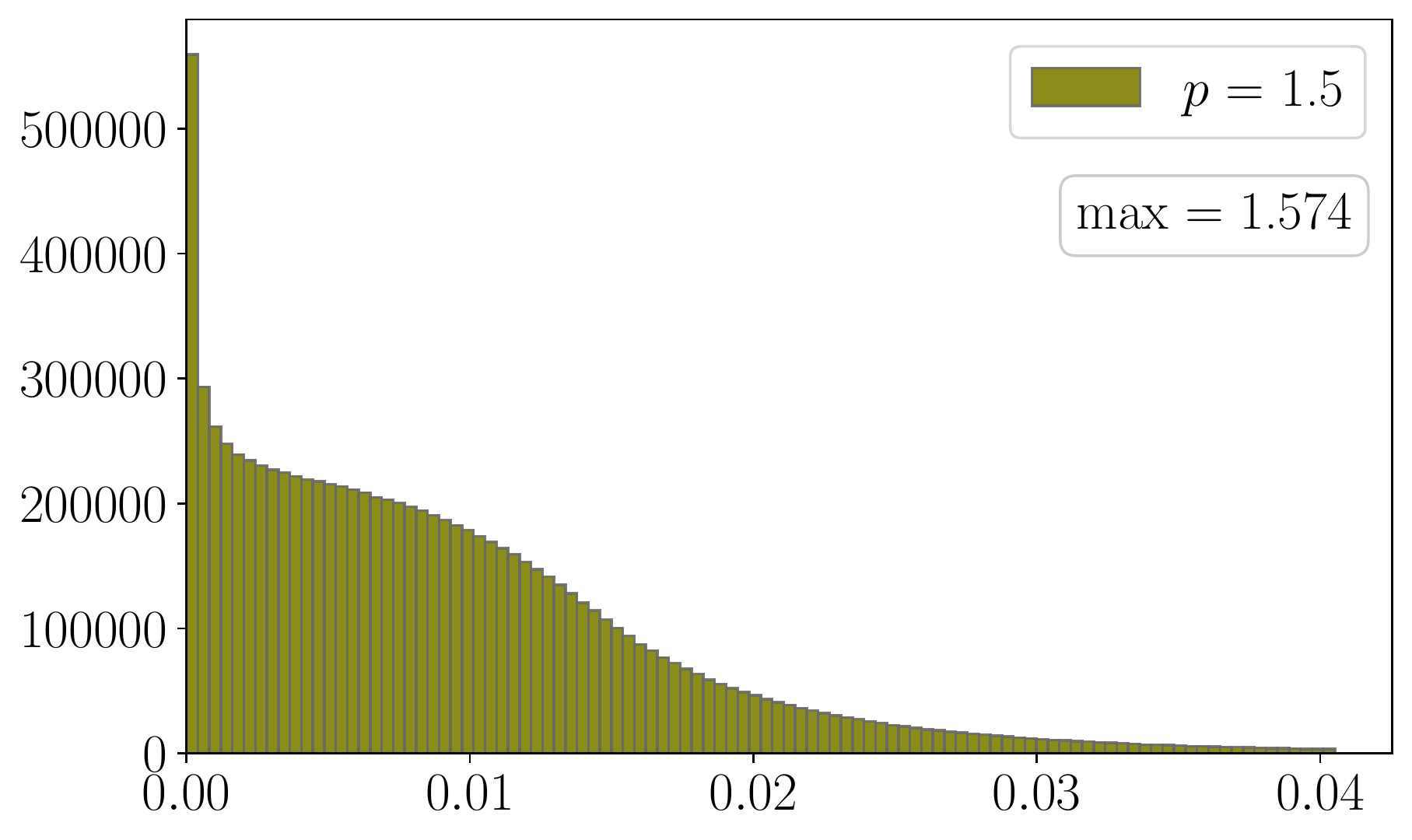}
    \end{subfigure}
    \begin{subfigure}[b]{0.45\textwidth}
        \includegraphics[width=\textwidth]{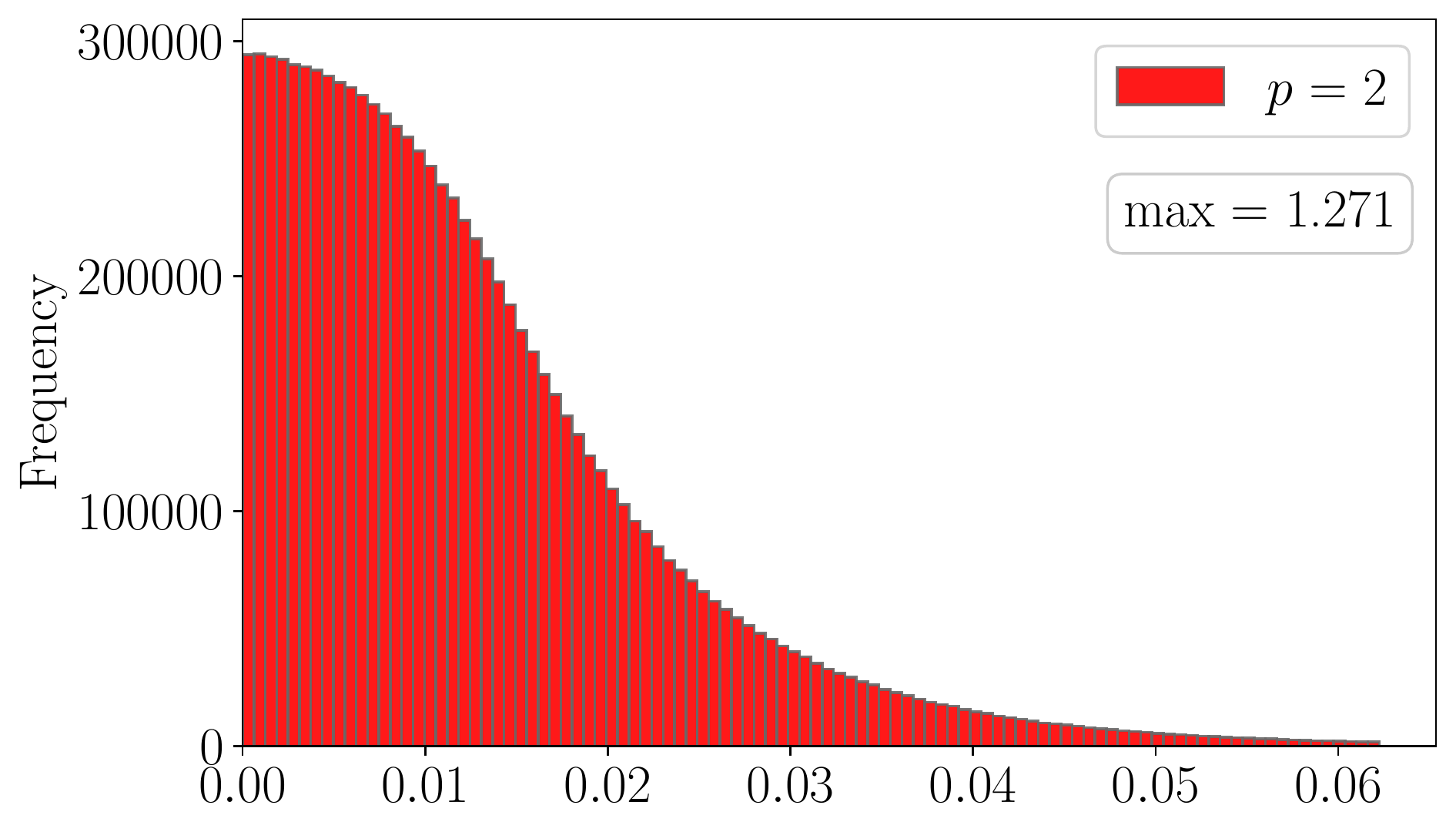}
    \end{subfigure}
    ~
    \begin{subfigure}[b]{0.45\textwidth}
        \includegraphics[width=\textwidth]{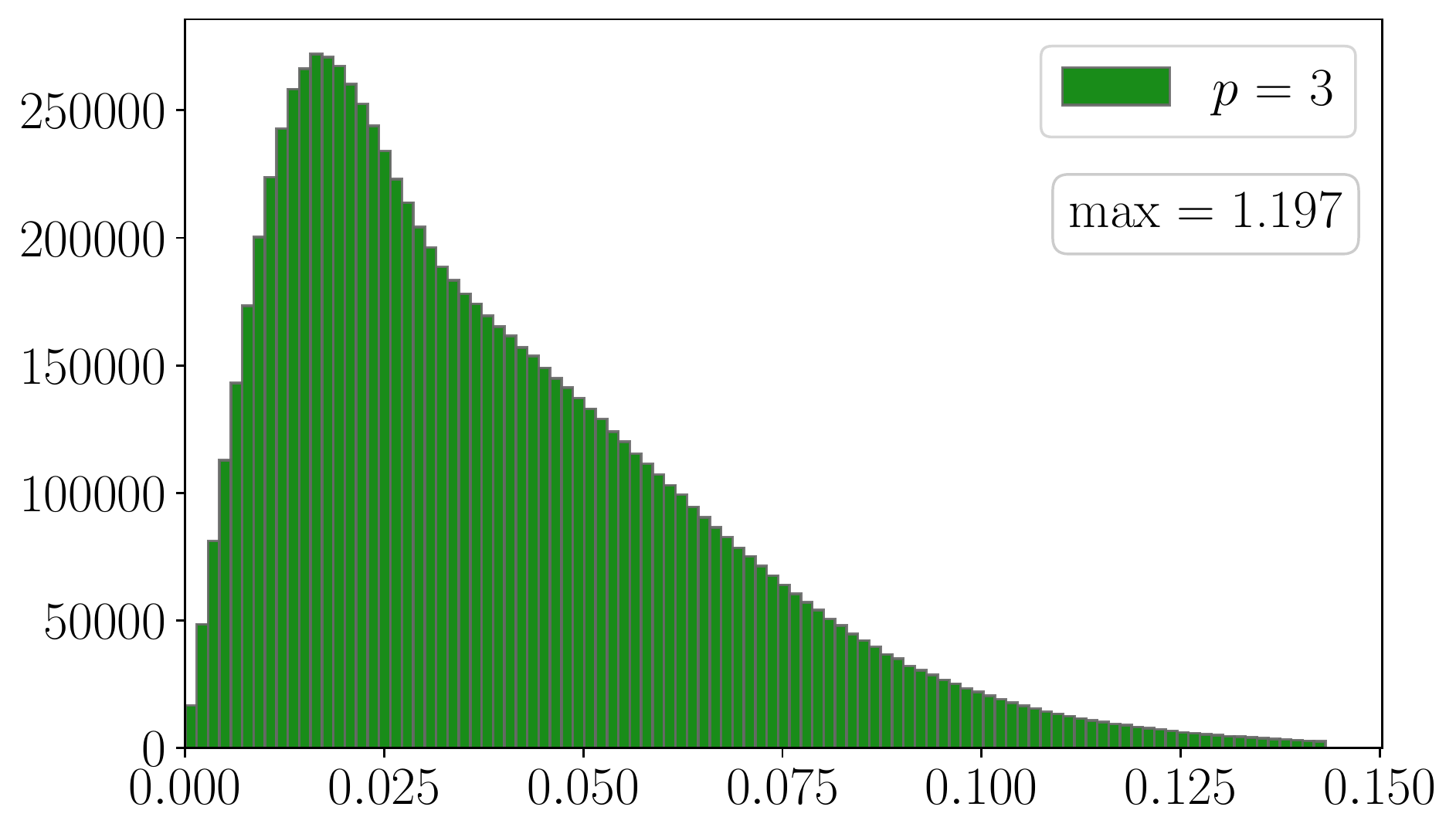}
    \end{subfigure}
    \begin{subfigure}[b]{0.45\textwidth}
        \includegraphics[width=\textwidth]{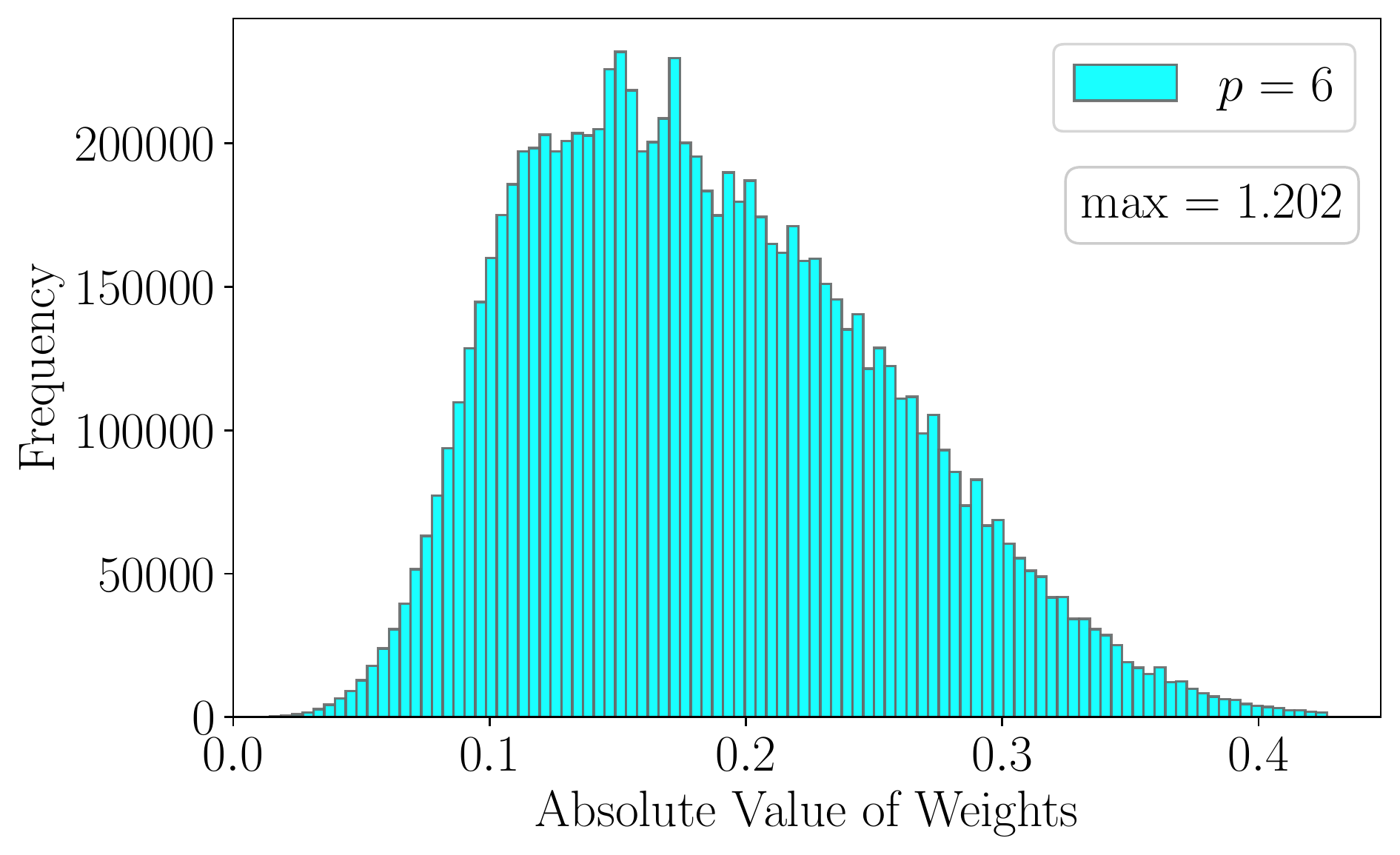}
    \end{subfigure}
    ~
    \begin{subfigure}[b]{0.45\textwidth}
        \includegraphics[width=\textwidth]{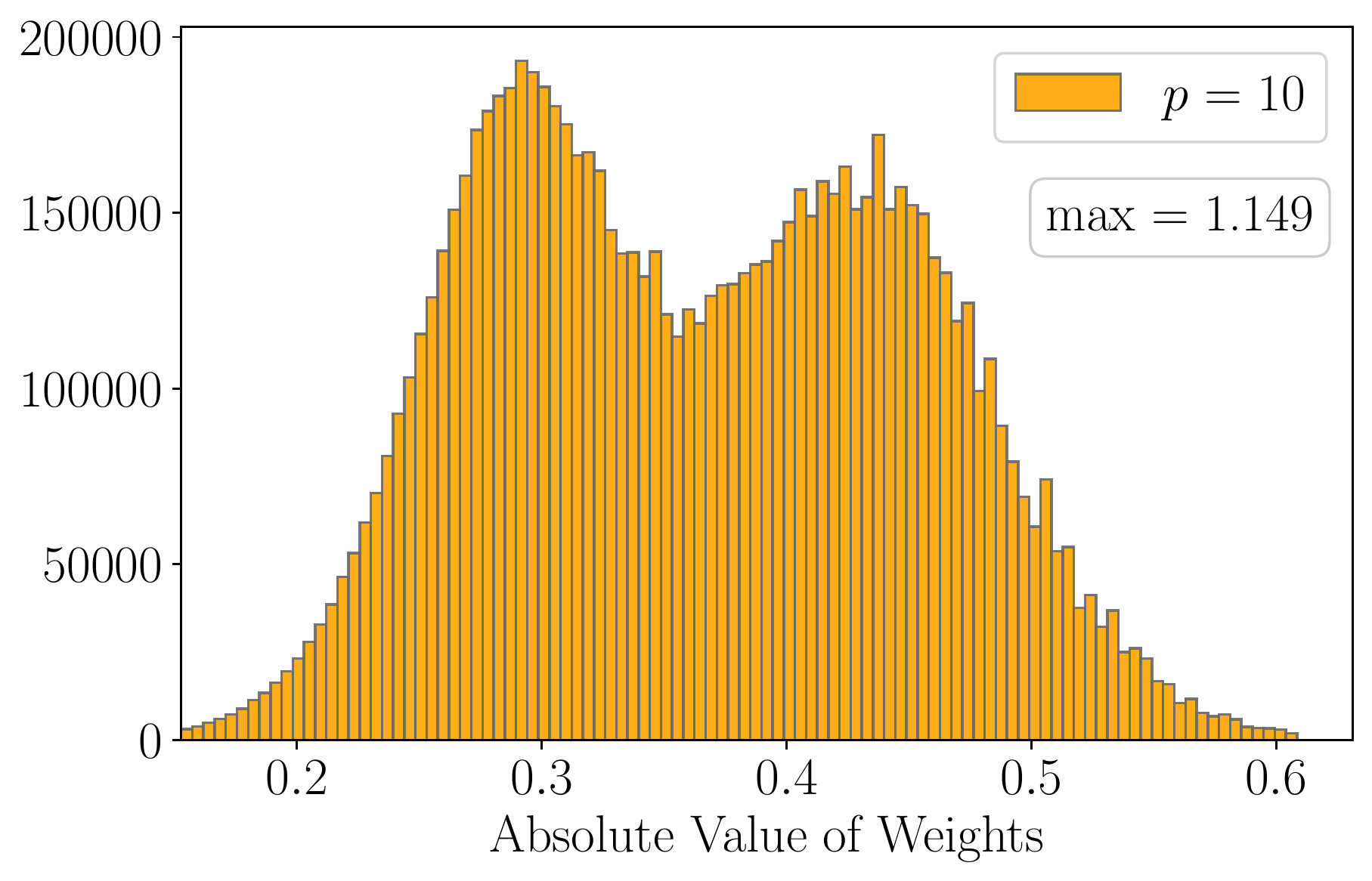}
    \end{subfigure}
    \caption{The histogram of weights in \textsc{VGG-11} models trained with \algname for the CIFAR-10 dataset. 
    For clarity, we cropped out the tails and each plot has 100 bins after cropping.
    }
    \label{fig:cifar10-hist-vgg-full}
\end{figure}

\clearpage

\subsection{CIFAR-10 experiments: generalization}
We present a more complete result for the CIFAR-10 generalization experiment in Section~\ref{sec:cifar} with additional values of $p$.

In the following table, we see that \algname with $p = 3$ continues have the highest generalization performance for all deep neural networks.

\label{sec:add-experiment-cifar-generalization}
\begin{table}[!h]
    \centering
    \setlength{\tabcolsep}{5.5pt}
    \begin{tabular}{l| c|c|c|c}
         \hline
         &  \hspace{1.25em} \textsc{VGG-11} \hspace{1.25em} & \hspace{0.75em} \textsc{ResNet-18} \hspace{0.75em} & \textsc{MobileNet-v2} & \textsc{RegNetX-200mf}  \\
         \hline \hline
         $p = 1.1$ & \pmval{88.19}{.17} & \pmval{92.63}{.12} & \pmval{91.16}{.09}& \pmval{91.21}{.18}  \\
         $p = 1.5$ & \pmval{88.45}{.29} & \pmval{92.73}{.11} & \pmval{90.81}{.19}& \pmval{90.91}{.12} \\
         $p = 2$ (SGD) & \pmval{90.15}{.16} & \bpmval{93.90}{.14} & \pmval{91.97}{.10}& \pmval{92.75}{.13} \\
         $p = 3$ & \bpmval{90.85}{.15} & \bpmval{94.01}{.13} & \bpmval{93.23}{.26}& \bpmval{94.07}{.12} \\
         $p = 6$ & \pmval{89.47}{.14} & \bpmval{93.87}{.13} & \pmval{92.84}{.15}& \pmval{93.03}{.17} \\
         $p = 10$ & \pmval{88.78}{.37} & \pmval{93.55}{.21} & \pmval{92.60}{.22}& \pmval{92.97}{.16} \\
         \hline
    \end{tabular}
    \caption{CIFAR-10 test accuracy (\%) of \algname on various deep neural networks. For each deep net and value of $p$, the average $\pm$ \textcolor{gray}{std. dev.} over 5 trials are reported. And the best performing value(s) of $p$ for each individual deep net is highlighted in \textbf{boldface}.}
    \label{tab:generalization-cifar10-full}
\end{table}

\subsection{ImageNet experiments}
\label{sec:add-experiment-imagenet}
To verify if our observations on the CIFAR-10 generalization performance hold up for other datasets, we also performed similar experiments for the much larger ImageNet dataset.
Due to computational constraints, we were only able to experiment with the \textsc{ResNet-18} and \textsc{MobileNet-v2} architectures and only for one trial.

It is worth noting that the neural networks we used cannot reach 100\% training accuracy on Imagenet.
The models we employed only achieved top-1 training accuracy in the mid-70's.
So, we are not in the so-called \textit{interpolation regime}, and there are many other factors that can significantly impact the generalization performance of the trained models.
In particular, we find that not having weight decay costs us around 3\% in validation accuracy in the $p = 2$ case and this explains why our reported numbers are lower than PyTorch's baseline for each corresponding architecture.
Despite this, we find that \algname with $p = 3$ has the best generalization performance on the ImageNet dataset, matching our observation from the CIFAR-10 dataset.

\begin{table}[!h]
    \centering
    \begin{tabular}{l| c | c}
         \hline
        & \textsc{ResNet-18} & \textsc{MobileNet-v2} \\
        \hline\hline
        $p=1.1$ & 64.08 & 63.41 \\
        $p=1.5$ & 65.14 & 65.75 \\
        $p=2$ (SGD) & 66.76 &  67.91 \\
        $p=3$ & \textbf{67.67} & \textbf{69.74} \\
        $p=6$ & 66.69 & 67.05 \\
        $p=10$ & 65.10 & 62.32 \\
         \hline
    \end{tabular}
    \caption{ImageNet top-1 validation accuracy (\%) of \algname on various deep neural networks. The best performing value(s) of $p$ for each individual deep network is highlighted in \textbf{boldface}.}
    \label{tab:imagenet}
\end{table}

\end{document}